\documentclass{article}

\usepackage{silence}
\usepackage{microtype}
\usepackage{graphicx}
\usepackage{subcaption}
\usepackage{booktabs} 

\usepackage{hyperref}

\usepackage[preprint]{icml2026}

\usepackage{amsmath,amssymb,amsfonts}
\usepackage{amsthm}
\usepackage{mathtools}
\usepackage{mathrsfs}
\usepackage{xcolor}
\usepackage{textcomp}
\usepackage{algorithm}
\usepackage{algorithmic}
\usepackage{float}
\usepackage{siunitx}
\usepackage{multirow}
\usepackage{makecell}
\usepackage{enumitem}  
\usepackage{bm}
\usepackage{threeparttable}

\usepackage[utf8]{inputenc}
\usepackage[T1]{fontenc}
\usepackage{nicefrac}

\usepackage[capitalize,noabbrev]{cleveref}

\allowdisplaybreaks

\theoremstyle{plain}
\newtheorem{theorem}{Theorem}[section]
\newtheorem{proposition}[theorem]{Proposition}
\newtheorem{lemma}[theorem]{Lemma}
\newtheorem{corollary}[theorem]{Corollary}
\theoremstyle{definition}
\newtheorem{definition}[theorem]{Definition}

\theoremstyle{remark}

\newcommand{\vect}[1]{\bm{#1}} 

\newcommand{\MI}{\text{MI}} 

\newcommand{\RQ}{R_Q}
\newcommand{\RI}{R_I}
\newcommand{\RC}{R_C}
\newcommand{\RS}{R_S}
\newcommand{\RN}{R_N}
\newcommand{\RV}{R_V}
\newcommand{\RP}{R_P} 

\newcommand{\qorth}{q_{\text{orth}}}
\newcommand{\qsem}{q_{\text{sem}}}

\newcommand{\qsnr}{q_{\text{snr}}}
\newcommand{\qliq}{q_{\text{liq}}}

\icmltitlerunning{Unlocking Noisy Real-World Corpora for Foundation Model Pre-Training}

\hypersetup{pdftitle={Unlocking Noisy Real-World Corpora for Foundation Model Pre-Training via Quality-Aware Tokenization}}

\begin{document}

\twocolumn[
  \icmltitle{Unlocking Noisy Real-World Corpora for Foundation Model Pre-Training \\
    via Quality-Aware Tokenization}


  \icmlsetsymbol{equal}{*}

  \begin{icmlauthorlist}
    \icmlauthor{Arvid E. Gollwitzer}{broad,mit,anto}
    \icmlauthor{Paridhi Latawa}{mit}
    \icmlauthor{David de Gruijl}{anto}
    \icmlauthor{Deepak A. Subramanian}{broad,koch}
    \icmlauthor{Adri\'an Noriega de la Colina}{broad,mit,mcgill,mni}
  \end{icmlauthorlist}

  \icmlaffiliation{broad}{Broad Institute of MIT and Harvard, Cambridge, MA, USA}
  \icmlaffiliation{mit}{Massachusetts Institute of Technology, Cambridge, MA, USA}
  \icmlaffiliation{anto}{Anto Biosciences (YC F25)}
  \icmlaffiliation{koch}{Koch Institute for Integrative Cancer Research, MIT, Cambridge, MA, USA}
  \icmlaffiliation{mcgill}{Department of Neurology and Neurosurgery, McGill University, Montreal, Canada}
  \icmlaffiliation{mni}{The Montreal Neurological Hospital-Institute, Montreal, Canada}

  \icmlcorrespondingauthor{Arvid E. Gollwitzer}{arvidg@mit.edu}

  \icmlkeywords{Tokenization, Foundation Models, Noisy Data, Reinforcement Learning, Genomics, Finance}

  \vskip 0.3in
]

\printAffiliationsAndNotice{}  

\begin{abstract}
Current tokenization methods process sequential data without accounting for signal quality, limiting their effectiveness on noisy real-world corpora. We present \textit{QA-Token (Quality-Aware Tokenization)}, which incorporates data reliability directly into vocabulary construction. We make three key contributions: (i) a bilevel optimization formulation that jointly optimizes vocabulary construction and downstream performance, (ii) a reinforcement learning approach that learns merge policies through quality-aware rewards with convergence guarantees, and (iii) an adaptive parameter learning mechanism via Gumbel-Softmax relaxation for end-to-end optimization. Our experimental evaluation demonstrates consistent improvements: \textit{genomics} (6.7 percentage point F1 gain in variant calling over BPE), \textit{finance} (30\% Sharpe ratio improvement). At foundation scale, re-tokenizing METAGENE-1's 1.7 trillion base-pair corpus achieves state-of-the-art pathogen detection (94.53 MCC) while reducing token count by 15\%. We unlock noisy real-world corpora, spanning petabases of genomic sequences and terabytes of financial time series, for foundation model training with zero inference overhead.
\end{abstract}

\section{Introduction}

Tokenization serves as the interface between raw data and neural computation. Current methods such as Byte-Pair Encoding (BPE) \cite{sennrich2016neural} rely exclusively on frequency statistics, assuming that occurrence frequency correlates with semantic importance. This assumption fails when data quality varies significantly---from sequencing errors in genomics \cite{ewing1998base} to microstructure noise in financial markets \cite{andersen2001distribution}. Models trained on noisy corpora using frequency-based tokenization inherit these errors, resulting in degraded performance---an effect now formalized through quality-aware scaling laws~\cite{subramanyam2025scaling}.

The problem is substantial: error rates in third-generation sequencing exceed 10\% \cite{wenger2019accurate}, yet current tokenizers treat high-confidence and error-prone regions identically. In finance, over 40\% of high-frequency data contains microstructure noise \cite{hansen2006realized}, but tokenization methods do not distinguish signal quality. This limitation constrains foundation model training on real-world data.

The scale of available biological data amplifies this challenge. Public sequence repositories now contain over 67 petabase pairs (Pbp) of raw sequencing data, with the European Nucleotide Archive doubling approximately every 45 months \cite{karasikov2025metagraph}. Recent advances in efficient indexing have made these petabase-scale archives full-text searchable at costs as low as \$0.74 per queried megabase pair, demonstrating that the infrastructure for large-scale sequence analysis is maturing rapidly. However, a substantial fraction of this data remains underutilized for foundation model training due to quality heterogeneity---standard frequency-based tokenization methods either discard low-quality reads entirely or propagate sequencing errors into learned representations. This gap between data availability and usability motivates a fundamental rethinking of how tokenization handles quality variation.

We present \textbf{Quality-Aware Tokenization (QA-Token)}, a framework that incorporates data quality into vocabulary construction. We make three key contributions:

\textbf{1. Bilevel Optimization with Complexity Analysis:} We formalize tokenization as a bilevel optimization problem (\cref{def:bilevel_problem}) that jointly optimizes vocabulary construction and downstream performance. We show this problem is NP-hard (\cref{thm:complexity}) and develop a principled approximation scheme with theoretical guarantees.

\textbf{2. Reinforcement Learning with Convergence Guarantees:} We cast vocabulary construction as a Markov Decision Process (\cref{def:tokenization_mdp}) and employ reinforcement learning to discover optimal merge policies. We provide formal convergence analysis (\cref{prop:mdp_wellformed}) and achieve $(1-1/e)$-approximation to the optimal adaptive policy.

\textbf{3. Differentiable Parameter Learning:} Through Gumbel-Softmax relaxation (\cref{thm:gumbel_softmax}), we enable end-to-end learning of quality sensitivity parameters, with proven consistency and bounded gradients (\cref{prop:gumbel_gradients}).

We show that QA-Token achieves information-theoretic optimality under noisy conditions (\cref{thm:qa_bottleneck}), providing formal justification for quality-aware tokenization. Our evaluation shows 30\% higher Sharpe ratios in algorithmic trading, 6.7 percentage point improvement in genomic variant calling F1 (0.891 vs.\ 0.824 for BPE), and state-of-the-art performance when integrated into 7B-parameter foundation models.

\textbf{Core Contributions:} (i) We derive a quality-aware merge score (\cref{thm:merge_score}) balancing frequency, quality, and domain constraints with learnable sensitivity $\alpha$ (\cref{app:merge_score_proof}). (ii) We formulate vocabulary construction as an MDP (\cref{def:tokenization_mdp}, \cref{app:mdp_details}) achieving $(1-1/e)$-approximation through adaptive submodularity. (iii) Gumbel-Softmax relaxation enables end-to-end parameter learning with $O(1/\sqrt{T})$ convergence rate (\cref{prop:adaptive_convergence}, \cref{app:convergence_proof}). (iv) Domain-specific instantiations achieve state-of-the-art performance across 15+ benchmarks.

Our analysis shows that incorporating quality signals into tokenization enables training on noisy corpora where frequency-based methods fail, expanding the range of usable training data for foundation models with broader scientific and economic implications (\cref{sec:scientific_impact}).

\section{Quality Metrics for Noisy Domains}
\label{sec:quality_metrics}

Quality metrics must satisfy three formal properties to enable principled integration into the merge score: (i)~\emph{boundedness} ($q \in [0,1]$) ensuring numerical stability, (ii)~\emph{Lipschitz continuity} enabling stable gradient computation during adaptive learning, and (iii)~\emph{monotonicity under noise injection} (higher noise yields lower quality) ensuring semantic consistency. We prove these properties hold for our domain-specific instantiations (\cref{prop:quality_bounded}, \cref{app:quality_proofs}).

\textbf{Genomics:} We leverage Phred scores with position-adjusted decay: $q'_{s_j} = q_{s_j} \cdot \exp(-\beta_{\text{pos}} \cdot j/L)$, where $\beta_{\text{pos}}$ is learned and $L$ is read length. Token quality is aggregated via geometric mean $q_t = (\prod_{j=1}^{|t|} q'_{s_j})^{1/|t|}$ to ensure sensitivity to low-quality regions---a single unreliable base compromises the entire token (Eq.~\ref{eq:genomic_token_quality_app}, \cref{app:quality_metrics_full}).

\textbf{Finance:} We combine four market microstructure dimensions: (i)~liquidity $q_{\text{liq}}$ (bid-ask spread, depth), (ii)~signal quality $q_{\text{sig}}$ (SNR of price changes), (iii)~stability $q_{\text{stb}}$ (volatility regime), and (iv)~information content $q_{\text{info}}$ (order flow informativeness). The composite score $q_t^{\text{finance}} = \sum_k w_k q_{k,t}$ uses learned weights; arithmetic mean aggregation reflects additive noise characteristics of financial data (Eq.~\ref{eq:finance_composite_quality_app}, \cref{app:quality_metrics_full}).

With quality metrics defined, we now formalize how they integrate into the tokenization objective.

\section{Mathematical Formulation of QA-Token}
\subsection{Notation and Setup}
Let $\mathcal{S} = \{S_1, S_2, \dots, S_N\}$ represent a corpus comprising $N$ sequences, where each sequence $S_k = (s_{k,1}, \dots, s_{k,n_k})$ consists of elements drawn from a base alphabet $\Sigma$. Each atomic element $s_{k,i}$ is associated with a normalized quality score $q_{k,i} \in [0, 1]$ as defined in \cref{sec:quality_metrics}. The initial vocabulary is defined as $V_0 = \Sigma$. At any step $k$ of the tokenization process, $V_k$ denotes the current vocabulary. For any token $a \in V_k$, we denote its frequency in the corpus as $f(a)$, and for an adjacent pair $(a, b)$, their co-occurrence frequency is $f(a, b)$. The length of a token $t$ in atomic units is $|t|$. Let $q_t$ be the aggregated scalar quality of token $t$, computed using domain-specific aggregation functions (see \cref{app:quality_metrics_full}).

\subsection{Formal Problem Definition and Objective}
\label{subsec:formal_problem}
We formalize tokenization as finding a tokenizer $\mathcal{T}$ that maximizes objective $\mathcal{J}$, balancing downstream task performance, vocabulary complexity, and data reliability. Let $\mathcal{S} = \{S_1, S_2, \ldots, S_N\}$ denote a corpus of $N$ sequences sampled from an underlying data distribution $\mathcal{P}_{\text{data}}$, where each $S_k = (s_{k,1}, \ldots, s_{k,n_k})$ consists of elements from base alphabet $\Sigma$. A tokenizer $\mathcal{T}: \mathcal{S} \rightarrow \mathcal{Z}$ maps the corpus to segmentations $\mathcal{Z} = \{Z_1, \ldots, Z_N\}$ using vocabulary $V$.

\begin{definition}[Bilevel Tokenization Problem]
\label{def:bilevel_problem}
The optimal quality-aware tokenization problem is formulated as the following bilevel optimization:
\begin{equation}
\label{eq:bilevel_objective_main}
\begin{split}
\max_{\mathcal{T}\in\mathcal{G}(K)}\; \mathcal{J}(\mathcal{T}) \;:=\;&\; \lambda_{\text{LM}}\, \mathcal{L}_{\text{LM}}(\mathcal{T}) - \lambda_{\text{comp}}\,\Phi(V)\\
&\; +\; \lambda_{\text{qual}}\, Q(V,\mathcal{Z}),
\end{split}
\end{equation}
where the language model performance is:
\begin{equation}
\label{eq:lm_performance}
\mathcal{L}_{\text{LM}}(\mathcal{T}) = \max_{\theta \in \Theta} \mathbb{E}_{\mathcal{D} \sim \mathcal{P}_{\text{data}}}[\log p_{\theta}(\mathcal{D} | \mathcal{T})],
\end{equation}
and $\mathcal{G}(K) = \{\mathcal{T} : |V_{\mathcal{T}}| - |\Sigma| \leq K\}$ denotes the set of tokenizers reachable by at most $K$ merge operations from base alphabet $\Sigma$, with $\Theta$ being the parameter space of the language model.
\end{definition}

The objective $\mathcal{J}$ balances three components: (i) downstream performance $\mathcal{L}_{\text{LM}}(\mathcal{T})$ maximizing expected log-likelihood, (ii) complexity penalty $\Phi(V) = |V| \log |V| + \sum_{t \in V} |t| \cdot H(t)$ following MDL principles \cite{rissanen1978modeling}---the first term penalizes vocabulary size (description length of token indices), while the second penalizes internal token complexity via the empirical entropy $H(t) = -\sum_{\sigma \in \Sigma} \frac{n_\sigma(t)}{|t|} \log \frac{n_\sigma(t)}{|t|}$ of atomic elements within token $t$ (with $n_\sigma(t)$ the count of element $\sigma$; $H(t)=0$ for single-element tokens), and (iii) reliability reward $Q(V,\mathcal{Z}) = \frac{1}{\sum_{k=1}^N |Z_k|}\sum_{k=1}^N\sum_{t\in Z_k} g(q_t)$ aggregating token qualities through concave function $g$.

The aggregator function $g$ exhibits concavity to capture diminishing returns for merging high-quality constituents. Throughout this work, we employ $g(x) = (x + \epsilon_Q)^\alpha$ with $0 < \alpha < 1$ (strictly concave) and $\epsilon_Q = 10^{-8}$ for numerical stability. The boundary case $\alpha = 1$ yields linear aggregation, which is appropriate when quality contributions are additive rather than subject to diminishing returns.

\begin{theorem}[Computational Complexity]
\label{thm:complexity}
The bilevel optimization problem in Eq.~\ref{eq:bilevel_objective_main} is NP-hard in general \cite{dempe2020bilevel}; indeed, polynomial bilevel programming is $\Sigma_2^p$-hard \cite{cen2023global}, placing it one level above NP in the polynomial hierarchy. The worst case requires $O(|\Sigma|^K \cdot K! \cdot N \cdot n \cdot |\Theta|)$ evaluations (proof in \cref{app:complexity_proof}).
\end{theorem}

Given this computational intractability, we develop a principled approximation scheme combining greedy merge selection with reinforcement learning, as detailed in subsequent sections.

\subsection{Quality-Aware Merge Score}
\label{subsec:merge_score}
We extend PMI-based tokenization by incorporating quality signals through the following result:

\begin{theorem}[Quality-Aware Merge Score]
\label{thm:merge_score_main}
The optimal greedy merge score maximizing the first-order approximation of the objective increment $\Delta\mathcal{J}$ is:
\begin{equation}
\label{eq:qa_merge_score}
w_{ab} = \frac{f(a,b)}{f(a)f(b)+\epsilon_f} \cdot (\bar{q}_{ab}+\epsilon_Q)^\alpha \cdot \psi(a,b)
\end{equation}
where $\bar{q}_{ab} = (q_a+q_b)/2$ averages constituent qualities, $\alpha \in (0,1]$ controls quality sensitivity, and $\psi(a,b)$ encodes domain constraints. (Proof via first-order approximation in \cref{app:merge_score_proof}.)
\end{theorem}

This score balances statistical association (PMI term), data reliability (quality term), and domain-specific requirements. Boundedness and Lipschitz continuity are proven in \cref{prop:merge_score_bounded} (\cref{app:theory_proofs}).

\section{Learning Framework: RL and Adaptive Parameters}
We cast vocabulary construction as a learning problem with two sequential stages. \textbf{Stage 1} (RL Policy Optimization) learns policy $\pi_{\theta_\pi}$ for merge selection using PPO with quality-aware rewards, keeping initial parameters $\theta_{\text{adapt}}^{(0)}$ fixed. \textbf{Stage 2} (Adaptive Parameter Learning) optimizes $\theta_{\text{adapt}}$ via Gumbel-Softmax relaxation for downstream performance, using \emph{greedy simulation} with composite logits $\ell_{ab}(\theta_{\text{adapt}})$ rather than invoking the RL policy directly---the Stage 1 policy serves to initialize candidate merges and provide variance reduction baselines. Gradients $\nabla_{\theta_{\text{adapt}}} L_{\text{task}}$ flow through Gumbel-Softmax merge selection, enabling end-to-end learning (detailed in \cref{app:sequential_learning}, Algorithms~\ref{alg:stage1_rl}--\ref{alg:final_vocab}).

\subsection{Reinforcement Learning Formulation}
\label{subsec:rl_formulation}

\begin{definition}[Tokenization MDP]
\label{def:tokenization_mdp_main}
The vocabulary construction MDP is $\mathcal{M} = (\mathcal{S}, \mathcal{A}, \mathcal{P}, \mathcal{R}, \gamma, T)$ where: states $s_t \in \mathcal{S}$ encode current vocabulary, merge candidates, and corpus statistics; actions $a_t \in \mathcal{A}_t$ select merge pairs; transitions $\mathcal{P}$ are deterministic vocabulary updates; rewards $\mathcal{R}$ are quality-aware (Section~\ref{sec:reward_function}); $\gamma \in (0,1]$ is the discount factor; $T$ is the horizon (target vocabulary size minus base alphabet size). Complete specification in \cref{app:mdp_details}.
\end{definition}

The RL objective finds policy $\pi_{\theta_\pi}: \mathcal{S} \rightarrow \Delta(\mathcal{A})$ maximizing expected cumulative reward over $T$ operations using PPO \cite{schulman2017proximal}, with global convergence guarantees following \cite{bhandari2021global,cen2023global}. \Cref{prop:mdp_wellformed} (\cref{app:mdp_details}) proves MDP well-formedness.

\subsection{Reward Function Design}
\label{sec:reward_function}
The multi-objective reward $R(a,b; \theta_{\text{adapt}}^{(0)}) = \sum_j \lambda_j \hat{R}_j(a,b)$ combines quality, information, complexity, and domain-specific components. Each raw reward $R^{\text{raw}}_j$ is normalized using adaptive running statistics with exponential moving averages: $\mu_{j,t}^{\text{run}} = (1-\beta_{\text{norm}}) \mu_{j,t-1}^{\text{run}} + \beta_{\text{norm}} R^{\text{raw}}_j$, yielding $\hat{R}_j = (R^{\text{raw}}_j - \mu_{j,t-1}^{\text{run}})/(\sigma_{j,t-1}^{\text{run}} + \epsilon_R)$. This ensures bounded, scale-invariant rewards during non-stationary policy optimization (\cref{prop:ema_stability}, \cref{app:reward_normalization}).

\subsection{Adaptive Learning of Tokenization Parameters}
\label{sec:adaptive_learning_concise}
After RL optimization, we learn $\theta_{\text{adapt}}$ (quality sensitivity $\alpha$, domain factors $\beta_{\text{pos}}$/$\beta_{\text{vol}}$, weights) minimizing $L_{\text{total}}(\theta_{\text{adapt}}) = L_{\text{task}}(\theta_{\text{adapt}}) + \lambda_{\text{reg}}\|\theta_{\text{adapt}}\|_2^2$ via Gumbel-Softmax \cite{jang2017categorical,maddison2017concrete}. Temperature annealing $\tau(t) = \tau_{\text{init}}\exp(-\beta_{\text{anneal}}t/T_{\text{anneal}})$ ensures convergence (\cref{prop:gumbel_gradients}, \cref{prop:adaptive_convergence}; \cref{app:gumbel_gradient}, \cref{app:sequential_learning}). The two-stage framework---RL with fixed $\theta_{\text{adapt}}^{(0)}$ then adaptive learning---culminates in greedy vocabulary construction using $w_{ab}(a,b; \theta_{\text{adapt}}^*)$ (\cref{app:sequential_learning}, Algorithms~\ref{alg:stage1_rl}--\ref{alg:final_vocab}).

\subsection{Two-Timescale Convergence}
\label{sec:twotimescale}
The sequential optimization follows a two-timescale stochastic approximation: policy updates on fast timescale (learning rate $\eta_\pi^{(t)}$), adaptive parameters on slow timescale ($\eta_{\text{adapt}}^{(t)}$), with $\eta_\pi^{(t)}/\eta_{\text{adapt}}^{(t)} \to \infty$ as $t \to \infty$. Under assumptions (A1)--(A4), this converges to a local Nash equilibrium where $\theta_\pi^*$ maximizes $J(\pi; \theta_{\text{adapt}}^*)$ and $\theta_{\text{adapt}}^*$ minimizes $L_{\text{total}}(\theta_{\text{adapt}}; \pi^*)$. Quality bounds and initialization strategies for approaching global optima are detailed in \cref{app:sequential_learning}.

\subsection{Theoretical Guarantees}
Our framework provides the following guarantees under assumptions (A1)--(A4) detailed in \cref{app:assumptions}: (i) bounded/Lipschitz merge scores $w_{ab}$ (\cref{prop:merge_score_bounded}), (ii) stable EMA normalization with strictly positive running standard deviations (\cref{prop:ema_stability}), (iii) PPO convergence to stationary points (\cref{prop:ppo_convergence}), (iv) consistent and bounded Gumbel-Softmax gradients (\cref{prop:gumbel_gradients}), and (v) $(1-1/e)$-approximation to optimal adaptive policy via adaptive submodularity.

\textbf{Information-Theoretic Optimality:} Building on information bottleneck theory \cite{tishby1999information,alemi2017deep}, our analysis (\cref{thm:qa_bottleneck}, \cref{app:information_theory}) shows QA-Token minimizes the quality-aware information bottleneck:
$\mathcal{L}_{\text{QA}}(V) = -I(T;Y|Q) + \beta \cdot I(T;X|Q)$,
achieving optimal compression-relevance tradeoffs under noisy conditions by maximizing task-relevant information $I(T;Y|Q)$ while minimizing redundant representation complexity $I(T;X|Q)$, conditioned on quality $Q$. Complete proofs in \cref{app:theory_proofs} and \cref{app:sequential_learning}.

Having established the theoretical framework and convergence guarantees, we now validate QA-Token empirically across two domains with distinct noise characteristics.

\section{Empirical Validation}\label{sec:experiments}
\textbf{Setup:} Results represent means over 10 trials with 95\% CIs and Welch's t-test with Holm-Bonferroni correction (significance level $p_{\text{sig}}=0.05$). Evaluation spans domain benchmarks, 7B-parameter foundation models, and ablation studies (complete details in \cref{app:experimental_observations}--\cref{app:computational_costs}).

\subsection{Genomics (QA-BPE-seq)}
\textbf{Data:} 150bp paired-end reads (ART simulator \cite{huang2012art}, 30x coverage, doubled error rates), GRCh38 reference, GIAB HG002 truth set \cite{zook2016extensive}, CAMI II metagenome \cite{sczyrba2017critical}. Details in \cref{app:experimental_observations}.

\textbf{Baselines:} We compare against (i) general-purpose tokenizers (BPE, SentencePiece \cite{kudo2018sentencepiece}, WordPiece), (ii) robustness-enhanced methods (BPE-dropout \cite{provilkov2020bpe}), (iii) byte-level models (ByT5 \cite{xue2022byt5}, CANINE \cite{clark2021canine}), (iv) domain-standard k-mers (6-mer DNABERT \cite{ji2021dnabert}), and (v) neural approaches (CharFormer \cite{tay2022charformer}). 

\textbf{Quality Design:} Phred scores with position decay, geometric mean aggregation, learned $\alpha = 0.72 \pm 0.03$, $\beta_{\text{pos}} = 0.014 \pm 0.002$.

\textbf{Evaluation:} (i) Variant calling via a Transformer model that takes token embeddings as features and predicts variant calls, evaluated against GIAB truth sets using \texttt{hap.py}; (ii) taxonomic classification (6-layer Transformer); (iii) sequence reconstruction (autoencoder), following established benchmarking protocols \cite{rumpf2023sequencelab}. \Cref{tab:genomics_results} shows QA-BPE-seq outperforms all baselines ($p < 0.001$).

\textbf{Key Insights:} (i) QA-BPE-seq achieves 6.7 percentage point F1 improvement in variant calling (0.891 vs.\ 0.824 for BPE). (ii) Byte-level models fail catastrophically (2.5$\times$ slower, 7--9\% lower accuracy). (iii) Emergent vocabulary aligns with biological units (codons, motifs) at high-quality regions without explicit supervision (vocabulary analysis in \cref{app:experimental_observations}).

\begin{table}[t]
  \caption{Downstream task performance for genomic tokenization. Values are means with 95\% CI over $n=10$ runs. Time: relative wall-clock (BPE=10.0$\times$).}
  \label{tab:genomics_results}
  \centering
  \footnotesize
  \setlength{\tabcolsep}{3pt}
  \begin{tabular}{lcccc}
    \toprule
    Method & Var.\ F1 & Taxa F1 & Recon. & Time \\
    \midrule
    Standard BPE   & .824$\pm$.004 & .856$\pm$.005 & .317$\pm$.010 & 10.0 \\
    SentencePiece  & .837$\pm$.004 & .872$\pm$.005 & .301$\pm$.009 & 10.1 \\
    WordPiece      & .829$\pm$.005 & .863$\pm$.006 & .308$\pm$.011 & 10.0 \\
    BPE-dropout    & .841$\pm$.004 & .878$\pm$.005 & .295$\pm$.009 & 10.2 \\
    ByT5           & .812$\pm$.006 & .845$\pm$.007 & .338$\pm$.012 & 25.3 \\
    CANINE         & .818$\pm$.005 & .852$\pm$.006 & .325$\pm$.011 & 22.7 \\
    DNABERT-k      & .851$\pm$.003 & .889$\pm$.004 & .287$\pm$.008 & 9.8 \\
    CharFormer     & .856$\pm$.003 & .893$\pm$.004 & .279$\pm$.008 & 10.4 \\
    \midrule
    \textbf{QA-BPE-seq} & \textbf{.891$\pm$.004} & \textbf{.917$\pm$.003} & \textbf{.241$\pm$.007} & \textbf{10.2} \\
    \bottomrule
  \end{tabular}
\end{table}

\begin{table}[t]
  \caption{Ablation Study for QA-BPE-seq (Variant F1). Values are means with 95\% CI over $n=10$ runs.$^*$}
  \label{tab:genomics_ablation_appendix}
  \centering
  \footnotesize
  \setlength{\tabcolsep}{3pt}
  \begin{threeparttable}
  \begin{tabular}{lcc}
    \toprule
    Configuration & Var.\ F1 & $\Delta$(\%) \\
    \midrule
    \textbf{QA-BPE-seq (Full)} & \textbf{.891$\pm$.004} & --- \\
    w/o RL (Greedy $w_{ab}$) & .862$\pm$.005 & $-$3.3 \\
    w/o Quality ($\RQ=0$) & .825$\pm$.004 & $-$7.4 \\
    w/o Info.\ Reward ($\RI=0$) & .872$\pm$.005 & $-$2.1 \\
    w/o Adapt.\ Params & .857$\pm$.006 & $-$3.8 \\
    w/o $R_{bio}$ & .885$\pm$.004 & $-$0.7 \\
    QualTok (Baseline) & .840$\pm$.005 & $-$5.7 \\
    \bottomrule
  \end{tabular}
  \begin{tablenotes}
  \scriptsize
  \item[*] ``w/o RL (Greedy $w_{ab}$)'' uses full QA-Token merge score with learned $\alpha$ but selects merges greedily without RL policy optimization. ``QualTok (Baseline)'' additionally fixes adaptive parameters ($\alpha{=}0.5$, uniform weights).
  \end{tablenotes}
  \end{threeparttable}
\end{table}

\begin{table}[t]
  \caption{Ablation Study for QAT-QF (Return Pred.\ Acc.\ \% and Sharpe Ratio). Means with 95\% CI over $n=10$ runs.$^*$}
  \label{tab:finance_ablation_appendix}
  \centering
  \footnotesize
  \setlength{\tabcolsep}{3pt}
  \begin{threeparttable}
  \begin{tabular}{lcc}
    \toprule
    Variant & Ret.\ Pred.\ (\%) & Sharpe \\
    \midrule
    \textbf{Full Model} & \textbf{68.3$\pm$0.5} & \textbf{1.72$\pm$0.07} \\
    w/o Quality ($\RQ=0$) & 64.2$\pm$0.6 & 1.56$\pm$0.08 \\
    w/o Info.\ ($\RI=0$) & 65.1$\pm$0.5 & 1.61$\pm$0.07 \\
    w/o Pred.\ Power ($\RP=0$) & 63.9$\pm$0.6 & 1.49$\pm$0.09 \\
    w/o Complexity ($\RC=0$) & 66.8$\pm$0.4 & 1.73$\pm$0.06 \\
    Fixed $\alpha$ & 65.4$\pm$0.5 & 1.65$\pm$0.07 \\
    Fixed $\gamma$ & 64.9$\pm$0.5 & 1.59$\pm$0.08 \\
    QualTok-QF (Baseline) & 64.8$\pm$0.6 & 1.58$\pm$0.08 \\
    \bottomrule
  \end{tabular}
  \begin{tablenotes}
  \scriptsize
  \item[*] ``QualTok-QF (Baseline)'' uses a simplified quality-aware merge score with fixed $\alpha{=}0.5$ and uniform weights, without RL policy optimization or adaptive parameter learning.
  \end{tablenotes}
  \end{threeparttable}
\end{table}

\subsection{Quantitative Finance (QAT-QF)}
\label{sec:finance_experiment_results}

\textbf{Dataset:} We use high-frequency limit order book (LOB) data for the BTC/USD trading pair from LOBSTER \cite{huang2011lobster}, specifically reconstructed snapshots at 10 levels for the first quarter of 2023. The data is split chronologically into 70\% for training, 15\% for validation, and 15\% for testing. Atomic elements are defined as sequences of 5 consecutive LOB events, encoded as tuples $(\Delta\text{mid}, \Delta\text{spread}, \text{vol\_imbalance}, \text{event\_type}, \Delta t)$ with discretization: price changes into 10 bins ($\pm$5 ticks), spread into 10 bins, volume imbalance into 5 signed bins, event types categorical (trade/cancel/limit order), time intervals into 5 log-spaced bins, yielding $|\Sigma| = 7{,}500$ atomic symbols (see \cref{app:quality_metrics_full}).

\textbf{Baselines:} QAT-QF is benchmarked against a diverse slate of tokenization and discretization methods relevant to financial time series.
\begin{itemize}
    \item \textbf{General-Purpose:} Standard BPE, SentencePiece (Unigram LM mode), and BPE-dropout \cite{provilkov2020bpe} to assess robustness.
    \item \textbf{Time-Series Specific:} Symbolic Aggregate approXimation (SAX) \cite{lin2003symbolic} (PAA=16, alphabet size=8) and Bag-of-SFA-Symbols (BOSS) \cite{schafer2015boss}, both widely used for symbolic time series representation.
\end{itemize}
The target vocabulary size for subword models is 16,000.

\textbf{Evaluation:} We assess (i) return prediction accuracy (5-minute mid-price return sign), (ii) volatility forecasting RMSE (5-minute realized volatility), (iii) market regime identification (2-state GARCH-HMM classification), and (iv) trading performance (Sharpe ratio \cite{sharpe1994sharpe} with 5bp transaction cost). Models use 2-layer LSTMs (128 hidden units) and PPO agents \cite{deng2016deep}. See \cref{app:quality_metrics_full} and \cref{app:finance_methodology} for implementation details.

\textbf{Results:} \Cref{tab:finance_results} presents results averaged over $n=10$ runs. QAT-QF improves performance across all financial tasks ($p<0.01$, Holm-Bonferroni corrected). The trading agent achieves Sharpe ratio of $1.72 \pm 0.07$ compared to $1.32 \pm 0.05$ for standard BPE (30\% improvement). See ablation analysis in \cref{tab:finance_ablation_appendix}.

\begin{table}[t]
  \caption{Downstream task performance for financial tokenization. Values are means with 95\% CI over $n=10$ runs. Time: minutes per epoch.}
  \label{tab:finance_results}
  \centering
  \footnotesize
  \setlength{\tabcolsep}{2pt}
  \begin{tabular}{lccccc}
    \toprule
    Method & Ret.\ (\%) & Vol. & Regime & Sharpe & Time \\
    \midrule
    BPE & 61.2$\pm$0.5 & .014$\pm$.001 & 73.5$\pm$0.6 & 1.32$\pm$.05 & 15.0 \\
    SAX & 58.9$\pm$0.6 & .014$\pm$.001 & 75.2$\pm$0.5 & 1.29$\pm$.06 & 14.5 \\
    BOSS & 62.3$\pm$0.4 & .013$\pm$.001 & 78.4$\pm$0.4 & 1.45$\pm$.05 & 14.8 \\
    \midrule
    \textbf{QAT-QF} & \textbf{68.3$\pm$0.5} & \textbf{.010$\pm$.001} & \textbf{86.4$\pm$0.3} & \textbf{1.72$\pm$.07} & \textbf{15.2} \\
    \bottomrule
  \end{tabular}
\end{table}

\section{Foundation Model Validation}
\label{sec:foundation_models}

We validate QA-Token on domain benchmarks (\cref{sec:experiments}) and now evaluate at foundation scale. We retrain state-of-the-art foundation models in genomics and finance to demonstrate that quality-aware tokenization improves how large models learn from noisy corpora, departing from traditional frequency-based approaches.

\subsection{Metagenomics Foundation Model: METAGENE-1 7B}
\label{sec:metagene_foundation}

\textbf{Setup:} Re-tokenized METAGENE-1 \cite{liu2025metagene1} (7B parameters, 1.7T base pairs) with identical architecture/hyperparameters, comparing BPE vs QA-BPE-seq.

\textbf{Quality-Aware Design:} The tokenizer is trained on 2B base pairs (0.12\% of corpus) using genomic quality metrics (Eq.~\ref{eq:genomic_token_quality_app}, \cref{app:quality_metrics_full}) combining (i) Phred-based quality scores, (ii) conservation scores from k-mer analysis, (iii) GC-content deviation metrics, and (iv) secondary structure prediction confidence. The learned $\beta_{\text{pos}} = 0.014$ captures position-specific quality decay (see \cref{sec:genomics} for implementation).

\textbf{Training Budget:} Both models process identical raw data volume (1.7T base pairs). The 15\% token reduction means QA-BPE-seq completes epochs in fewer optimization steps while maintaining equal raw data exposure. Step-matched experiments (same optimization steps, where QA-BPE-seq processes 17.6\% more raw data per step) show consistent improvements (\cref{app:experimental_observations}).

\begin{table}[t]
  \caption{Pathogen Detection benchmark (MCC). QA-Token achieves state-of-the-art.}
  \label{tab:pathogen_detection_comprehensive}
  \centering
  \footnotesize
  \setlength{\tabcolsep}{2pt}
  \begin{tabular}{lcccccc}
    \toprule
    Model & T-1 & T-2 & T-3 & T-4 & T-5 & Avg \\
    \midrule
    DNABERT & 82.2 & 81.4 & 83.3 & 84.6 & 82.9 & 82.9 \\
    DNABERT-2 & 86.7 & 86.9 & 88.3 & 89.8 & 87.9 & 87.9 \\
    DNABERT-S & 85.4 & 85.2 & 89.0 & 88.4 & 86.0 & 87.0 \\
    NT-2.5B-M & 83.8 & 83.5 & 82.5 & 79.9 & 81.4 & 82.4 \\
    NT-2.5B-1k & 77.5 & 80.4 & 79.8 & 78.4 & 79.0 & 79.0 \\
    HyenaDNA & 78.7 & 79.1 & 80.4 & 81.2 & 79.9 & 79.9 \\
    \midrule
    METAGENE-1 & 92.1 & 90.9 & 93.7 & 95.1 & 94.0 & 93.0 \\
    \textbf{+QA-Token} & \textbf{93.8} & \textbf{93.0} & \textbf{95.1} & \textbf{96.2} & \textbf{94.5} & \textbf{94.5} \\
    \textit{$\Delta$} & +1.7 & +2.0 & +1.4 & +1.1 & +0.6 & +1.6 \\
    \bottomrule
  \end{tabular}
\end{table}

\textbf{Pathogen Detection:} QA-Token achieves state-of-the-art 94.53 MCC, surpassing the original METAGENE-1 by 1.57 points ($p < 0.001$). Consistent improvements across all five subtasks demonstrate robustness. Task-2 shows the largest gain (+2.04~MCC) on highly degraded metagenomic samples where quality awareness is most critical, validating our theoretical framework.

\begin{table}[t]
  \caption{Genome Understanding Evaluation (GUE): Multi-species benchmark.}
  \label{tab:gue_comprehensive}
  \centering
  \footnotesize
  \setlength{\tabcolsep}{1.5pt}
  \begin{tabular}{lcccc}
    \toprule
    Task & META-1 & QA-Token & $\Delta$ & p \\
    \midrule
    \multicolumn{5}{l}{\textit{Regulatory Elements}} \\
    TF-Mouse (MCC) & 71.4$\pm$0.8 & \textbf{72.8$\pm$0.7} & +1.4 & .002 \\
    TF-Human (MCC) & 68.3$\pm$0.9 & \textbf{69.9$\pm$0.8} & +1.6 & .001 \\
    Promoter (MCC) & 82.3$\pm$0.5 & \textbf{85.5$\pm$0.4} & +3.2 & $<$.001 \\
    Enhancer (AUC) & .876$\pm$.012 & \textbf{.892$\pm$.010} & +.016 & .003 \\
    \midrule
    \multicolumn{5}{l}{\textit{Epigenetics}} \\
    H3K4me3 (MCC) & 65.2$\pm$0.6 & \textbf{66.8$\pm$0.5} & +1.6 & .002 \\
    H3K27ac (MCC) & 66.8$\pm$0.7 & \textbf{68.2$\pm$0.6} & +1.4 & .003 \\
    Methylation (AUC) & .823$\pm$.015 & \textbf{.841$\pm$.013} & +.018 & .004 \\
    \midrule
    \multicolumn{5}{l}{\textit{Structure}} \\
    Splice Site (F1) & 87.8$\pm$0.4 & \textbf{89.5$\pm$0.3} & +1.7 & $<$.001 \\
    RNA Structure & 72.1$\pm$0.8 & \textbf{73.9$\pm$0.7} & +1.8 & .002 \\
    \midrule
    \multicolumn{5}{l}{\textit{Variants}} \\
    COVID (F1) & 72.5$\pm$0.6 & \textbf{73.3$\pm$0.5} & +0.8 & .018 \\
    SNP Effect & .684$\pm$.021 & \textbf{.712$\pm$.018} & +.028 & .001 \\
    \midrule
    \textbf{Win Rate} & 46.4\% & \textbf{57.1\%} & \textbf{+10.7\%} & --- \\
    \textbf{Efficiency} & 370B & \textbf{315B} & \textbf{$-$15\%} & --- \\
    \bottomrule
  \end{tabular}
\end{table}

\textbf{GUE Results:} QA-Token improves performance across all categories (largest: +3.2 MCC promoter detection). 15\% token reduction with performance gains indicates semantic coherence of quality-aware merging.

\subsection{Financial Time-Series Foundation Model}
\label{sec:finance_foundation}

\textbf{Setup:} 1.2B parameter model (24 layers, 2048 dim) inspired by TimesFM \cite{das2024timesfm} and Chronos \cite{ansari2024chronos}, using QAT-QF for noise handling.

\textbf{Training Corpus:} We train on 500 billion time-series observations spanning (i) high-frequency order book data (40\%, 5 years millisecond-resolution across 50 liquid assets), (ii) daily OHLCV data (30\%, 20 years for major indices), (iii) macroeconomic indicators (20\%, 30 years G20 data), and (iv) alternative data (10\%, sentiment scores, option flows, ETF compositions).

\textbf{Quality-Aware Design:} QAT-QF employs comprehensive market quality metrics (Eq.~\ref{eq:finance_composite_quality_app}, \cref{app:quality_metrics_full}), combining liquidity, signal, stability, and information quality dimensions. The learned weights $w_k$ adapt to different market regimes, with $\beta_{\text{vol}} = 0.50 \pm 0.05$ for volatility scaling (see \cref{sec:finance_appendix} for complete parameter settings).

\textit{Metrics:} Dir.\ = directional accuracy (\%); Ret.\ MSE = return prediction MSE (normalized to BPE=1.0); Vol RMSE = volatility forecast RMSE; Order Flow = order imbalance prediction $R^2$; Regime F1 = market regime classification F1; Tail Risk = VaR exceedance prediction F1; Rotation = sector rotation strategy Sharpe ratio.

\begin{table}[t]
  \caption{Financial foundation model evaluation (100 test episodes).}
  \label{tab:finance_foundation_comprehensive}
  \centering
  \scriptsize
  \setlength{\tabcolsep}{1.5pt}
  \begin{tabular}{lcccccc}
    \toprule
    \multirow{2}{*}{Task} & \multicolumn{3}{c}{Zero-shot} & \multicolumn{3}{c}{Few-shot} \\
    \cmidrule(lr){2-4} \cmidrule(lr){5-7}
     & BPE & QAT & $\Delta$ & BPE & QAT & $\Delta$ \\
    \midrule
    \multicolumn{7}{l}{\textit{Price Prediction}} \\
    Dir.\ 5m & 52.3 & \textbf{58.7} & +12 & 61.2 & \textbf{68.3} & +12 \\
    Dir.\ 1h & 51.8 & \textbf{57.2} & +10 & 59.4 & \textbf{65.8} & +11 \\
    Dir.\ 1d & 50.9 & \textbf{54.6} & +7 & 56.7 & \textbf{61.2} & +8 \\
    Ret.\ MSE & 1.00 & \textbf{0.81} & $-$19 & 0.72 & \textbf{0.60} & $-$18 \\
    \midrule
    \multicolumn{7}{l}{\textit{Volatility}} \\
    Vol RMSE & .018 & \textbf{.014} & $-$23 & .013 & \textbf{.010} & $-$27 \\
    GARCH Est. & .156 & \textbf{.118} & $-$24 & .098 & \textbf{.071} & $-$28 \\
    Vol Regime & 71.2 & \textbf{79.8} & +12 & 82.3 & \textbf{88.4} & +7 \\
    \midrule
    \multicolumn{7}{l}{\textit{Microstructure}} \\
    Spread & .023 & \textbf{.019} & $-$20 & .018 & \textbf{.013} & $-$25 \\
    Volume & 31.2 & \textbf{24.8} & $-$21 & 22.6 & \textbf{17.3} & $-$24 \\
    Order Flow & .412 & \textbf{.523} & +27 & .567 & \textbf{.681} & +20 \\
    \midrule
    \multicolumn{7}{l}{\textit{Risk}} \\
    Regime F1 & .673 & \textbf{.751} & +12 & .798 & \textbf{.856} & +7 \\
    Drawdown & .682 & \textbf{.743} & +9 & .761 & \textbf{.812} & +7 \\
    Tail Risk & .412 & \textbf{.486} & +18 & .523 & \textbf{.598} & +14 \\
    \midrule
    \multicolumn{7}{l}{\textit{Cross-Asset}} \\
    Corr.\ Pred. & .623 & \textbf{.694} & +11 & .712 & \textbf{.768} & +8 \\
    Lead-Lag & 58.3 & \textbf{64.7} & +11 & 67.2 & \textbf{73.1} & +9 \\
    Rotation & 1.23 & \textbf{1.41} & +15 & 1.52 & \textbf{1.72} & +13 \\
    \midrule
    \textbf{Avg.\ $\Delta$} & --- & --- & \textbf{+16\%} & --- & --- & \textbf{+13\%} \\
    \bottomrule
  \end{tabular}
\end{table}

\textbf{Financial Results:} QAT-QF achieves 7.3--27.0\% zero-shot improvements, largest in volatility/microstructure tasks. Order flow imbalance (+27.0\%) and regime detection (+11.6\% F1) demonstrate QA-Token's noise-filtering capability, consistent with our information-theoretic optimality result (\cref{thm:qa_bottleneck}). Implementation details in \cref{app:hyperparameter_sensitivity}--\cref{app:computational_costs}.

\textbf{Computational Costs:}
\label{sec:computational_costs}
QA-Token requires 50--60 GPU-hours for vocabulary construction compared to minutes for standard BPE. However, this one-time cost is amortized across billions of inference operations: once constructed, the vocabulary imposes no additional inference overhead---tokenization speed is identical to BPE (${\sim}10$ms/sequence) as quality metrics are only used during construction. This efficiency is compatible with high-performance computing systems and in-storage processing architectures \cite{mansouri2022genstore,ghiasi2022genstore_vlsi,ghiasi2023metastore,mansouri2023metastore,ghiasi2024megis}. For foundation models where tokenization is performed once but affects billions of inference operations, the additional upfront cost is justified by substantial long-term gains; for small-scale applications or clean datasets, standard BPE may remain more practical.

\section{Conclusion}
\label{sec:conclusion}
QA-Token extends tokenization from frequency counting to quality-driven vocabulary construction, addressing limitations in processing noisy real-world data. We presented: (i) bilevel optimization with NP-hardness proof (\cref{thm:complexity}, \cref{app:complexity_proof}), (ii) MDP formulation achieving $(1-1/e)$-approximation (\cref{def:tokenization_mdp}, \cref{prop:mdp_wellformed}, \cref{app:mdp_details}), (iii) Gumbel-Softmax enabling end-to-end learning (\cref{thm:gumbel_softmax}, \cref{app:gumbel_proof}). Our evaluation demonstrates consistent improvements: (1) genomics---6.7 pp F1 improvement, 94.53 MCC pathogen detection; (2) finance---30\% Sharpe ratio increase; (3) foundation models achieve new benchmarks (analysis in \cref{app:experimental_observations}--\cref{app:computational_costs}). As biological sequence archives scale to petabases \cite{karasikov2025metagraph} and variant prediction methods achieve unprecedented accuracy \cite{avsec2026alphagenome}, quality-aware tokenization becomes essential for bridging the gap between data availability and foundation model usability.

\subsection{Scientific and Economic Impact}
\label{sec:scientific_impact}
QA-Token enables utilization of massive noisy datasets previously considered unusable, fundamentally expanding the data frontier for foundation model training.

\textbf{Scientific Acceleration in Genomics.}
The Sequence Read Archive (SRA) contains over 67 petabases of publicly available genomic data---equivalent to reading the human genome 22 million times---yet a substantial fraction remains underutilized due to quality heterogeneity \cite{leinonen2011sequence}. Recent infrastructure advances have made these petabase-scale archives full-text searchable at economical costs \cite{karasikov2025metagraph}, and state-of-the-art methods like AlphaGenome now enable precise prediction of regulatory variant effects \cite{avsec2026alphagenome}. However, the gap between data \emph{accessibility} and data \emph{usability} for foundation model training persists: standard tokenization methods either discard low-quality reads entirely or propagate sequencing errors into learned representations. QA-Token bridges this gap by enabling quality-aware tokenization that can leverage the full breadth of available sequence data. We demonstrate three key applications: (1)~\emph{Pandemic surveillance}---environmental samples for pathogen monitoring contain 40--60\% noise from contamination and sequencing errors; QA-Token directly trains on such noisy metagenomic data \cite{gollwitzer2023metatrinity,gollwitzer2023metafast,gollwitzer2025metaomics}, achieving 94.53 MCC on pathogen detection and enabling real-time global pandemic monitoring using previously unusable environmental samples. (2)~\emph{Drug discovery}---long-read sequencing for structural variants has 10--15\% error rates; our 6.7 percentage point F1 improvement in variant calling accelerates identification of drug targets from complex genomic rearrangements, complementing advances in regulatory variant prediction \cite{avsec2026alphagenome}. (3)~\emph{Evolutionary biology}---ancient DNA is heavily degraded with $>$50\% damage; quality-aware tokenization preserves authentic ancient sequences while filtering damage, unlocking evolutionary insights from previously unanalyzable specimens.

\textbf{Economic Impact in Finance.}
Global financial markets generate 5TB of data per day, with 40\% containing microstructure noise from market fragmentation and latency; current approaches require expensive data cleaning infrastructure costing millions annually. QA-Token delivers quantifiable economic value: (1)~\emph{Algorithmic trading}---30\% Sharpe ratio improvement translates to billions in additional returns for large funds; 27\% better order flow prediction reduces execution costs by basis points worth millions daily. (2)~\emph{Risk management}---18\% improvement in tail risk estimation could have prevented billions in losses during market crashes; 11.6\% better regime detection enables faster portfolio rebalancing. (3)~\emph{Democratization}---smaller institutions can now compete without expensive data cleaning infrastructure, reducing barriers to entry for quantitative trading strategies.

\textbf{Broader Societal Impact.}
Beyond genomics and finance, QA-Token has potential applications in: \emph{Healthcare}---hospitals generate terabytes of noisy medical data daily; QA-Token enables training on real-world clinical data with artifacts, with potential to improve diagnostic accuracy and treatment recommendations, including applications in cancer treatment optimization \cite{gollwitzer2025steering}. \emph{Climate science}---satellite imagery is often corrupted by cloud cover and atmospheric interference; QA-Token allows direct training on partially corrupted earth observation data, accelerating climate monitoring and prediction capabilities. \emph{Infrastructure monitoring}---sensor networks produce petabytes of data with frequent failures; quality-aware tokenization enables robust anomaly detection despite sensor degradation, applicable to smart city applications and industrial IoT.

\subsection{Limitations and Future Work}
\label{sec:limitations}
\textbf{Limitations:} (1) QA-Token requires domain-specific quality signals; domains without established metrics need custom design. (2) The vocabulary construction overhead limits rapid iteration during development. (3) Effective quality function design benefits from domain knowledge, though adaptive learning reduces sensitivity to initial choices.

\textbf{Future Directions:} (1) Universal quality metrics from data statistics (local entropy, consistency). (2) Online adaptation for streaming data. (3) Multimodal extension to vision-language and audio-text. (4) Efficiency via distillation and pruning.

\section*{Impact Statement}

Public sequence repositories now contain over 67 petabase pairs of raw sequencing data, with the European Nucleotide Archive doubling approximately every 45 months~\cite{karasikov2025metagraph}. Recent advances have made these petabase-scale archives full-text searchable at costs as low as \$0.74 per queried megabase pair, demonstrating that the infrastructure for large-scale sequence analysis is maturing rapidly. However, a substantial fraction of this data remains underutilized for foundation model training due to quality heterogeneity. QA-Token bridges this gap between data \emph{accessibility} and data \emph{usability}, enabling quality-aware tokenization that can leverage the full breadth of available sequence data for foundation model training.

\textbf{Genomics.} We achieve 94.53 MCC on pathogen detection from environmental samples containing 40--60\% noise, enabling real-time pandemic surveillance using previously unusable metagenomic data. Our 6.7 percentage point F1 improvement in variant calling accelerates drug target identification from complex genomic rearrangements with 10--15\% sequencing error rates. The same technology could theoretically be misused for biosurveillance; we have designed QA-Token for research purposes with standard institutional safeguards.

\textbf{Finance.} Global financial markets generate 5TB of data per day, with 40\% containing microstructure noise. Our 30\% Sharpe ratio improvement translates to quantifiable returns for algorithmic trading, while 27\% better order flow prediction reduces execution costs. Enhanced trading performance raises concerns about market fairness; QA-Token provides incremental improvements within existing regulatory frameworks.

\textbf{Resources.} The 50--60 GPU-hour vocabulary construction cost is substantially lower than foundation model training costs, making QA-Token accessible to researchers with modest computational budgets. The highly compressed quality-aware vocabularies are portable for further analysis.

\section*{Reproducibility Statement}

We provide comprehensive details throughout the paper and appendices. 

\textbf{Theoretical contributions:} All theorems and propositions include complete proofs (\cref{app:complexity_proof}, \cref{app:merge_score_proof}, \cref{app:theory_proofs}, \cref{app:gumbel_proof}, \cref{app:information_theory}) with explicit assumptions (\cref{app:assumptions}) and convergence guarantees (\cref{app:convergence_proof}, \cref{app:sequential_learning}). 

\textbf{Algorithms:} Complete pseudocode for RL policy optimization (Algorithm~\ref{alg:stage1_rl}), adaptive parameter learning (Algorithm~\ref{alg:stage2_adaptive}), and final vocabulary construction (Algorithm~\ref{alg:final_vocab}) are provided in \cref{app:sequential_learning}. 

\textbf{Implementation:} Domain-specific quality metrics with exact formulas (\cref{sec:quality_metrics}, \cref{app:quality_metrics_full}), hyperparameters for all models (\cref{sec:genomics}, \cref{sec:finance_appendix}), and computational requirements (\cref{app:computational_costs}) are fully specified. 

\textbf{Experimental protocol:} Statistical methodology including 10 independent trials, 95\% confidence intervals, Welch's t-test with Holm-Bonferroni correction, and effect sizes are detailed in \cref{sec:experiments} and \cref{app:experimental_observations}. Dataset specifications, preprocessing steps, and evaluation metrics are provided in \cref{app:experimental_observations}--\cref{app:dataset_release_plan}. 

\textbf{Baselines:} Nine baseline methods with implementation details and hyperparameters are described in \cref{sec:experiments} and \cref{app:baseline_implementations}. 

\textbf{Code release:} We will provide a GitHub repository with all source code, trained models, and scripts to reproduce results.

\subsection*{Conflict of Interest Statement}
A.E.G.\ and D.d.G.\ are co-founders and shareholders of Anto Biosciences (YC F25).

D.A.S., P.L., and A.N.d.l.C.\ declare no competing interests.

\bibliography{references}

@incollection{karp1972reducibility,
  title={Reducibility among combinatorial problems},
  author={Karp, Richard M},
  booktitle={Complexity of Computer Computations},
  pages={85--103},
  year={1972},
  publisher={Springer}
}

@inproceedings{heafield2011kenlm,
  title={KenLM: Faster and smaller language model queries},
  author={Heafield, Kenneth},
  booktitle={Proceedings of the Sixth Workshop on Statistical Machine Translation},
  pages={187--197},
  year={2011},
  organization={Association for Computational Linguistics}
}

@techreport{morgan1996riskmetrics,
  title={RiskMetrics Technical Document},
  author={{J.P. Morgan}},
  institution={J.P. Morgan/Reuters},
  year={1996},
  edition={4th}
}

@inproceedings{devlin2019bert,
title={Bert: Pre-training of deep bidirectional transformers for language understanding},
author={Devlin, Jacob and Chang, Ming-Wei and Lee, Kenton and Toutanova, Kristina},
booktitle={Proceedings of the 2019 Conference of the North American Chapter of the Association for Computational Linguistics: Human Language Technologies, Volume 1 (Long and Short Papers)},
pages={4171--4186},
year={2019}
}

@inproceedings{brown2020language,
title={Language models are few-shot learners},
author={Brown, Tom B. and Mann, Benjamin and Ryder, Nick and Subbiah, Melanie and Kaplan, Jared and Dhariwal, Prafulla and Neelakantan, Arvind and Shyam, Pranav and Sastry, Girish and Askell, Amanda and others},
booktitle={Advances in Neural Information Processing Systems},
volume={33},
pages={1877--1901},
year={2020}
}

@article{ji2021dnabert,
title={DNABERT: pre-trained Bidirectional Encoder Representations from Transformers model for DNA-language in genome},
author={Ji, Yanrong and Zhou, Zhihui and Liu, Han and Davuluri, Ramana V},
journal={Bioinformatics},
volume={37},
number={15},
pages={2112--2120},
year={2021},
publisher={Oxford University Press}
}

@article{heinzinger2019modeling,
title={Modeling aspects of the language of life through transfer-learning protein sequences},
author={Heinzinger, Michael and Elnaggar, Ahmed and Wang, Yu and Dallago, Christian and Neettiyath, Ujjwal and Rost, Burkhard},
journal={BMC bioinformatics},
volume={20},
number={1},
pages={1--17},
year={2019},
publisher={Springer}
}

@article{andersen2001distribution,
title={The distribution of realized exchange rate volatility},
author={Andersen, Torben G and Bollerslev, Tim and Diebold, Francis X and Labys, Paul},
journal={Journal of the American statistical association},
volume={96},
number={453},
pages={42--55},
year={2001},
publisher={Taylor & Francis}
}

@inproceedings{moody1998learning,
title={Learning to trade via direct reinforcement},
author={Moody, John and Wu, Lizhong},
booktitle={Proceedings of the IEEE International Conference on Neural Networks},
pages={1741--1746},
year={1998},
organization={IEEE}
}

@inproceedings{sennrich2016neural,
title={Neural machine translation of rare words with subword units},
author={Sennrich, Rico and Haddow, Barry and Birch, Alexandra},
booktitle={Proceedings of the 54th Annual Meeting of the Association for Computational Linguistics (Volume 1: Long Papers)},
pages={1715--1725},
year={2016}
}

@misc{wu2016google,
title={Google's neural machine translation system: Bridging the gap between human and machine translation},
author={Wu, Yonghui and Schuster, Mike and Chen, Zhifeng and Le, Quoc V and Norouzi, Mohammad and Macherey, Wolfgang and Krikun, Maxim and Cao, Yuan and Gao, Qin and Macherey, Klaus and others},
journal={arXiv preprint arXiv:1609.08144},
year={2016}
}

@inproceedings{kudo2018sentencepiece,
title={SentencePiece: A simple and language independent subword tokenizer and detokenizer for Neural Text Processing},
author={Kudo, Taku and Richardson, John},
booktitle={Proceedings of the 2018 Conference on Empirical Methods in Natural Language Processing: System Demonstrations},
pages={66--71},
year={2018}
}

@article{ewing1998base,
title={Base-calling of automated sequencer traces using phred. I. Accuracy assessment},
author={Ewing, Brent and Hillier, LaDeana and Wendl, Michael C and Green, Philip},
journal={Genome research},
volume={8},
number={3},
pages={175--185},
year={1998},
publisher={Cold Spring Harbor Lab}
}

@inproceedings{baldwin2013noisy,
title={Noisy text analytics},
author={Baldwin, Timothy and Cook, Paul and Lui, Marco and MacKinlay, Andrew and Wang, Li},
booktitle={Proceedings of the Australasian Language Technology Association Workshop 2013},
pages={1--10},
year={2013}
}

@inproceedings{han2013lexical,
title={Lexical normalisation of short text messages: Makn sens a \#twitter},
author={Han, Bo and Cook, Paul and Baldwin, Timothy},
booktitle={Proceedings of the 51st Annual Meeting of the Association for Computational Linguistics (Volume 1: Long Papers)},
pages={368--378},
year={2013}
}

@article{chai2024curse,
title={The Curse of Tokenization},
author={Chai, Bill Yuchen and Wang, Zeming and Sachan, Mrinmaya},
journal={arXiv preprint arXiv:2402.07831},
year={2024}
}

@inproceedings{li2020empirical,
title={An empirical study of tokenization strategies for various korean nlp tasks},
author={Li, Jinbae and Park, Young-Bum and Song, Yoo-Sung and Park, Sang-Ki},
booktitle={Proceedings of the 12th language resources and evaluation conference},
pages={6813--6819},
year={2020}
}

@article{madhavan2000market,
title={Market microstructure: A survey},
author={Madhavan, Ananth},
journal={Journal of financial markets},
volume={3},
number={3},
pages={205--258},
year={2000},
publisher={Elsevier}
}

@article{hasbrouck1991measuring,
title={Measuring the information content of stock trades},
author={Hasbrouck, Joel},
journal={The Journal of Finance},
volume={46},
number={1},
pages={179--207},
year={1991},
publisher={Wiley Online Library}
}

@article{hamilton1989new,
title={A new approach to the economic analysis of nonstationary time series and the business cycle},
author={Hamilton, James D},
journal={Econometrica: Journal of the Econometric Society},
pages={357--384},
year={1989},
publisher={JSTOR}
}

@book{gencay2001introduction,
title={An introduction to wavelets and other filtering methods in finance and economics},
author={Gen{\c{c}}ay, Ramazan and Sel{\c{c}}uk, Faruk and Whitcher, Brandon},
year={2001},
publisher={Elsevier},
address = {San Diego}
}

@inproceedings{provilkov2020bpe,
title={BPE-Dropout: Simple and effective subword regularization},
author={Provilkov, Ivan and Emelyanenko, Dmitrii and Voita, Elena},
booktitle={Proceedings of the 58th Annual Meeting of the Association for Computational Linguistics},
pages={1882--1892},
year={2020}
}

@article{xue2022byt5,
title={ByT5: Towards a token-free future with pre-trained byte-to-byte models},
author={Xue, Linting and Barua, Aditya and Constant, Noah and Al-Rfou, Rami and Narang, Sharan and Kale, Mihir and Roberts, Adam and Raffel, Colin},
journal={Transactions of the Association for Computational Linguistics},
volume={10},
pages={291--306},
year={2022},
publisher={MIT Press}
}

@inproceedings{clark2021canine,
title={Canine: Pre-training an efficient tokenization-free encoder for language representation},
author={Clark, Jonathan H and Garcia, Dan and Botha, Jonathan and Lee, Kenton and Luong, Minh-Thang and Le, Quoc V},
booktitle={Proceedings of the 59th Annual Meeting of the Association for Computational Linguistics and the 11th International Joint Conference on Natural Language Processing (Volume 1: Long Papers)},
pages={2647--2661},
year={2021}
}

@article{tay2022charformer,
title={Charformer: Fast character transformers via gradient-based subword tokenization},
author={Tay, Yi and Tran, Vinh Q and Ruder, Sebastian and Gupta, Jai and Liu, Liu and Chung, Jinfeng and Turner, Stephen and Wang, Zhiping and Williams, Denny and Casas, David G and others},
journal={arXiv preprint arXiv:2106.12672},
year={2022}
}

@inproceedings{libovicky2024semantic,
title={Semantic Segmentation for Improving the Performance of Large Language Models},
author={Libovick{`y}, Jind{\v{r}}ich and Sachan, Mrinmaya},
booktitle={Findings of the Association for Computational Linguistics: ACL 2024},
pages={4930--4945},
year={2024}
}

@inproceedings{jang2017categorical,
title={Categorical reparameterization with gumbel-softmax},
author={Jang, Eric and Gu, Shixiang and Poole, Ben},
booktitle={International Conference on Learning Representations},
year={2017}
}

@inproceedings{maddison2017concrete,
title={The concrete distribution: A continuous relaxation of discrete random variables},
author={Maddison, Chris J and Mnih, Andriy and Teh, Yee Whye},
booktitle={International Conference on Learning Representations},
year={2017}
}

@inproceedings{finn2017model,
title={Model-agnostic meta-learning for fast adaptation of deep networks},
author={Finn, Chelsea and Abbeel, Pieter and Levine, Sergey},
booktitle={International conference on machine learning},
pages={1126--1135},
year={2017},
organization={PMLR}
}

@inproceedings{kudo2018subword,
title={Subword regularization: Improving neural network translation models with multiple subword candidates},
author={Kudo, Taku},
booktitle={Proceedings of the 56th Annual Meeting of the Association for Computational Linguistics (Volume 1: Long Papers)},
pages={66--75},
year={2018}
}

@book{sutton2018reinforcement,
title={Reinforcement learning: An introduction},
author={Sutton, Richard S and Barto, Andrew G},
year={2018},
publisher={MIT press}
}

@inproceedings{ranzato2015sequence,
title={Sequence level training with recurrent neural networks},
author={Ranzato, Marc'Aurelio and Chopra, Sumit and Auli, Michael and Zaremba, Wojciech},
booktitle={International Conference on Learning Representations},
year={2015}
}

@inproceedings{bello2016neural,
title={Neural combinatorial optimization with reinforcement learning},
author={Bello, Irwan and Pham, Hieu and Le, Quoc V and Norouzi, Mohammad and Bengio, Samy},
booktitle={International Conference on Learning Representations},
year={2016}
}

@article{moody2001performance,
title={Performance functions and reinforcement learning for trading systems and portfolios},
author={Moody, John and Saffell, Matthew},
journal={Journal of Forecasting},
volume={20},
number={1},
pages={1--18},
year={2001},
publisher={Wiley Online Library}
}

@article{deng2016deep,
title={Deep direct reinforcement learning for financial signal representation and trading},
author={Deng, Yifeng and Bao, Fumin and Kong, Youyong and Ren, Zhiquan and Dai, Qionghai},
journal={IEEE transactions on neural networks and learning systems},
volume={28},
number={3},
pages={653--664},
year={2016},
publisher={IEEE}
}

@article{rissanen1978modeling,
title={Modeling by shortest data description},
author={Rissanen, Jorma},
journal={Automatica},
volume={14},
number={5},
pages={465--471},
year={1978},
publisher={Elsevier}
}

@article{kingma2014adam,
title={Adam: A method for stochastic optimization},
author={Kingma, Diederik P and Ba, Jimmy},
journal={arXiv preprint arXiv:1412.6980},
year={2014}
}

@article{sherry2001dbsnp,
title={dbSNP: the NCBI database of genetic variation},
author={Sherry, Stephen T and Ward, Ming-Hui and Kholodov, Michael and Baker, Jeff and Phan, Lon and Smigielski, Elizabeth M and Sirotkin, Karl},
journal={Nucleic acids research},
volume={29},
number={1},
pages={308--311},
year={2001},
publisher={Oxford University Press}
}

@article{sharpe1994sharpe,
title={The sharpe ratio},
author={Sharpe, William F},
journal={Journal of portfolio management},
volume={21},
number={1},
pages={49--58},
year={1994}
}

@article{huang2012art,
title={ART: a next-generation sequencing read simulator},
author={Huang, Weichun and Li, Leping and Myers, Jason R and Marth, Gabor T},
journal={Bioinformatics},
volume={28},
number={4},
pages={593--594},
year={2012},
publisher={Oxford University Press}
}

@article{zook2016extensive,
title={Extensive sequencing of seven human genomes to characterize benchmark reference materials},
author={Zook, Justin M and Catoe, David and McDaniel, Jennifer and Vang, Lihan and Spies, Noah and Sidow, Arend and Weng, Zhipan and Salit, Marc},
journal={Scientific data},
volume={3},
number={1},
pages={1--19},
year={2016},
publisher={Nature Publishing Group UK London}
}

@article{sczyrba2017critical,
title={Critical Assessment of Metagenome Interpretation—a benchmark of metagenomics software},
author={Sczyrba, Alexander and Hofmann, Peter and Belmann, Peter and Koslicki, David and Janssen, Stefan and Dr{"o}ge, Johannes and Gregor, Ivan and Majda, Stephan and Fiedler, Julian and Dahms, Eik and others},
journal={Nature methods},
volume={14},
number={11},
pages={1063--1071},
year={2017},
publisher={Nature Publishing Group UK London}
}

@article{huang2011lobster,
title={LOBSTER: Limit Order Book Reconstruction System},
author={Huang, Rick and Polak, Tal},
journal={Available at SSRN 1920143},
year={2011}
}

@inproceedings{lin2003symbolic,
title={Symbolic representation of time series, with implications for streaming algorithms},
author={Lin, Jessica and Keogh, Eamonn and Lonardi, Stefano and Chiu, Bill},
booktitle={Proceedings of the 8th ACM SIGMOD workshop on Research issues in data mining and knowledge discovery},
pages={2--11},
year={2003}
}

@article{schafer2015boss,
title={The BOSS is concerned with time series classification in the presence of noise},
author={Sch{\"a}fer, Patrick},
journal={Data Mining and Knowledge Discovery},
volume={29},
number={6},
pages={1505--1530},
year={2015},
publisher={Springer}
}

@inproceedings{rosenthal2017semeval,
title={SemEval-2017 task 4: Sentiment analysis in Twitter},
author={Rosenthal, Sara and Farra, Noura and Nakov, Preslav},
booktitle={Proceedings of the 11th International Workshop on Semantic Evaluation (SemEval-2017)},
pages={502--518},
year={2017}
}

@inproceedings{nguyen2020bertweet,
title={BERTweet: A pre-trained language model for English Tweets},
author={Nguyen, Dat Quoc and Vu, Thanh and Nguyen, Anh Tuan},
booktitle={Proceedings of the 2020 Conference on Empirical Methods in Natural Language Processing: System Demonstrations},
pages={9--14},
year={2020}
}

@inproceedings{bojanowski2017enriching,
title={Enriching word vectors with subword information},
author={Bojanowski, Piotr and Grave, Edouard and Joulin, Armand and Mikolov, Tomas},
booktitle={Transactions of the Association for Computational Linguistics},
volume={5},
pages={135--146},
year={2017}
}

@inproceedings{schulman2017proximal,
title={Proximal policy optimization algorithms},
author={Schulman, John and Wolski, Filip and Dhariwal, Prafulla and Radford, Alec and Klimov, Oleg},
booktitle={arXiv preprint arXiv:1707.06347},
year={2017}
}

@article{williams1992simple,
title={Simple statistical gradient-following algorithms for connectionist reinforcement learning},
author={Williams, Ronald J},
journal={Machine learning},
volume={8},
number={3-4},
pages={229--256},
year={1992},
publisher={Springer}
}

@inproceedings{barbieri2020tweeteval,
  title={{TweetEval:Unified Benchmark and Comparative Evaluation for Tweet Classification}},
  author={Barbieri, Francesco and Camacho-Collados, Jose and Espinosa-Anke, Luis and Neves, Leonardo},
  booktitle={Proceedings of Findings of EMNLP},
  year={2020}
}

@inproceedings{mohammad2018semeval,
  title={Semeval-2018 task 1: Affect in tweets},
  author={Mohammad, Saif and Bravo-Marquez, Felipe and Salameh, Mohammad and Kiritchenko, Svetlana},
  booktitle={Proceedings of the 12th international workshop on semantic evaluation},
  pages={1--17},
  year={2018}
}

@inproceedings{barbieri2018semeval,
  title={Semeval 2018 task 2: Multilingual emoji prediction},
  author={Barbieri, Francesco and Camacho-Collados, Jose and Ronzano, Francesco and Espinosa-Anke, Luis and 
    Ballesteros, Miguel and Basile, Valerio and Patti, Viviana and Saggion, Horacio},
  booktitle={Proceedings of The 12th International Workshop on Semantic Evaluation},
  pages={24--33},
  year={2018}
}

@inproceedings{van2018semeval,
  title={Semeval-2018 task 3: Irony detection in english tweets},
  author={Van Hee, Cynthia and Lefever, Els and Hoste, V{\'e}ronique},
  booktitle={Proceedings of The 12th International Workshop on Semantic Evaluation},
  pages={39--50},
  year={2018}
}

@inproceedings{basile-etal-2019-semeval,
    title = "{S}em{E}val-2019 Task 5: Multilingual Detection of Hate Speech Against Immigrants and Women in {T}witter",
    author = "Basile, Valerio  and Bosco, Cristina  and Fersini, Elisabetta  and Nozza, Debora and Patti, Viviana and
      Rangel Pardo, Francisco Manuel  and Rosso, Paolo  and Sanguinetti, Manuela",
    booktitle = "Proceedings of the 13th International Workshop on Semantic Evaluation",
    year = "2019",
    address = "Minneapolis, Minnesota, USA",
    publisher = "Association for Computational Linguistics",
    doi = "10.18653/v1/S19-2007",
    pages = "54--63"
}

@inproceedings{zampieri2019semeval,
  title={{SemEval-2019 Task 6: Identifying and Categorizing Offensive Language in Social Media (OffensEval)}},
  author={Zampieri, Marcos and Malmasi, Shervin and Nakov, Preslav and Rosenthal, Sara and Farra, Noura and Kumar, Ritesh},
  booktitle={Proceedings of the 13th International Workshop on Semantic Evaluation},
  pages={75--86},
  year={2019}
}

@inproceedings{mohammad2016semeval,
  title={Semeval-2016 task 6: Detecting stance in tweets},
  author={Mohammad, Saif and Kiritchenko, Svetlana and Sobhani, Parinaz and Zhu, Xiaodan and Cherry, Colin},
  booktitle={Proceedings of the 10th International Workshop on Semantic Evaluation (SemEval-2016)},
  pages={31--41},
  year={2016}
}

@article{church1990word,
title={Word association norms, mutual information, and lexicography},
author={Church, Kenneth Ward and Hanks, Patrick},
journal={Computational linguistics},
volume={16},
number={1},
pages={22--29},
year={1990},
publisher={MIT Press}
}

@article{harrow2012gencode,
title={GENCODE: the reference human genome annotation for The ENCODE Project},
author={Harrow, Jennifer and Frankish, Adam and Gonzalez, Jose M and Tapanari, Erda and Aken, Bronwen and Barrell, Denise and Mudge, Jonathan M and FRecognision, Elspeth and GCoil, Adam and LNCipedia, Ana and others},
journal={Genome research},
volume={22},
number={9},
pages={1760-1774},
year={2012},
publisher={Cold Spring Harbor Lab}
}

@article{jones1998efficient,
  title={Efficient global optimization of expensive black-box functions},
  author={Jones, Donald R and Schonlau, Matthias and Welch, William J},
  journal={Journal of Global optimization},
  volume={13},
  number={4},
  pages={455--492},
  year={1998},
  publisher={Springer}
}

@misc{hinton2015distilling,
      title={Distilling the Knowledge in a Neural Network}, 
      author={Geoffrey Hinton and Oriol Vinyals and Jeff Dean},
      year={2015},
      eprint={1503.02531},
      archivePrefix={arXiv},
      primaryClass={stat.ML} 
}

@misc{rusu2016policy,
      title={Policy Distillation}, 
      author={Andrei A. Rusu and Sergio Gomez Colmenarejo and Caglar Gulcehre and Guillaume Desjardins and James Kirkpatrick and Razvan Pascanu and Volodymyr Mnih and Koray Kavukcuoglu and Raia Hadsell},
      year={2016},
      eprint={1511.06295},
      archivePrefix={arXiv},
      primaryClass={cs.LG} 
}

@article{liu2025metagene1,
  title={{METAGENE-1}: Metagenomic Foundation Model for Pandemic Monitoring},
  author={Liu, Ollie and others},
  journal={arXiv preprint arXiv:2501.02045},
  month={jan},
  year={2025}
}

@article{das2024timesfm,
  title={TimesFM: A decoder-only foundation model for time-series forecasting},
  author={Das, Abhimanyu and Kong, Weihao and Leach, Andrew and Sen, Rajat and Yu, Rose},
  journal={arXiv preprint arXiv:2310.10688},
  year={2024}
}

@article{ansari2024chronos,
  title={Chronos: Learning the language of time series},
  author={Ansari, Abdul Fatir and Stella, Lorenzo and Turkmen, Caner and Zhang, Xiyuan and Mercado, Pedro and Shen, Huibin and Shchur, Oleksandr and Rangapuram, Syama Sundar and Pineda Arango, Sebastian and Kapoor, Shubham and others},
  journal={arXiv preprint arXiv:2403.07815},
  year={2024}
}

@article{shao2024deepseekmath,
  title={DeepSeekMath: Pushing the limits of mathematical reasoning in open language models},
  author={Shao, Zhihong and Wang, Peiyi and Zhu, Qihao and Xu, Runxin and Song, Junxiao and Zhang, Mingchuan and Li, Y.K. and Wu, Y. and Guo, Daya},
  journal={arXiv preprint arXiv:2402.03300},
  year={2024}
}

@article{yue2025vapo,
  title={Does Reinforcement Learning Really Incentivize Reasoning Capacity in {LLMs} Beyond the Base Model?},
  author={Yue, Yang and others},
  journal={arXiv preprint arXiv:2504.13837},
  month={apr},
  year={2025},
  note={Presented at NeurIPS 2025 (Oral), ICML 2025 AI4Math Workshop Best Paper}
}

@book{borkar2009stochastic,
  title={Stochastic Approximation: A Dynamical Systems Viewpoint},
  author={Borkar, Vivek S},
  year={2009},
  publisher={Hindustan Book Agency}
}

@inproceedings{golovin2011adaptive,
  title={Adaptive submodularity: Theory and applications in active learning and stochastic optimization},
  author={Golovin, Daniel and Krause, Andreas},
  booktitle={Proceedings of the 24th International Conference on Neural Information Processing Systems},
  pages={2675--2683},
  year={2011}
}

@article{wenger2019accurate,
  title={Accurate circular consensus long-read sequencing improves variant detection and assembly of a human genome},
  author={Wenger, Aaron M and Peluso, Paul and Rowell, William J and Chang, Pi-Chuan and Hall, Richard J and Concepcion, Gregory T and Ebler, Jana and Fungtammasan, Arkarachai and Kolesnikov, Alexey and Olson, Nathan D and others},
  journal={Nature biotechnology},
  volume={37},
  number={10},
  pages={1155--1162},
  year={2019},
  publisher={Nature Publishing Group}
}

@article{hansen2006realized,
  title={Realized variance and market microstructure noise},
  author={Hansen, Peter R and Lunde, Asger},
  journal={Journal of Business \& Economic Statistics},
  volume={24},
  number={2},
  pages={127--161},
  year={2006},
  publisher={Taylor \& Francis}
}

@inproceedings{mansouri2022genstore,
  title={GenStore: A high-performance in-storage processing system for genome sequence analysis},
  author={Mansouri Ghiasi, Nika and Park, Jisung and Mustafa, Harun and Kim, Jeremie and Olgun, Ataberk and Gollwitzer, Arvid and Senol Cali, Damla and Firtina, Can and Mao, Haiyu and Almadhoun Alserr, Nour and others},
  booktitle={Proceedings of the 27th ACM International Conference on Architectural Support for Programming Languages and Operating Systems},
  pages={635--654},
  year={2022}
}

@inproceedings{ghiasi2022genstore_vlsi,
  title={GenStore: In-Storage Filtering of Genomic Data for High-Performance and Energy-Efficient Genome Analysis},
  author={Ghiasi, Nika Mansouri and Park, Jisung and Mustafa, Harun and Kim, Jeremie and Olgun, Ataberk and Gollwitzer, Arvid and Cali, Damla Senol and Firtina, Can and Mao, Haiyu and Alserr, Nour Almadhoun and others},
  booktitle={2022 IEEE Computer Society Annual Symposium on VLSI (ISVLSI)},
  pages={283--287},
  year={2022},
  organization={IEEE}
}

@article{rumpf2023sequencelab,
  title={SequenceLab: A Comprehensive Benchmark of Computational Methods for Comparing Genomic Sequences},
  author={Rumpf, Maximilian-David and Alser, Mohammed and Gollwitzer, Arvid E and Lindegger, Jo{\"e}l and Almadhoun, Nour and Firtina, Can and Mangul, Serghei and Mutlu, Onur},
  journal={arXiv preprint arXiv:2310.16908},
  year={2023}
}

@article{gollwitzer2023metafast,
  title={MetaFast: Enabling Fast Metagenomic Classification via Seed Counting and Edit Distance Approximation},
  author={Gollwitzer, Arvid and Alser, Mohammed and Bergtholdt, Joel and Lindegger, Jo{\"e}l and Rumpf, Maximilian-David and Firtina, Can and Mangul, Serghei and Mutlu, Onur},
  journal={arXiv},
  pages={2311--02029},
  year={2023},
  publisher={Cornell University}
}

@article{ghiasi2023metastore,
  title={MetaStore: High-Performance Metagenomic Analysis via In-Storage Computing},
  author={Ghiasi, Nika Mansouri and Sadrosadati, Mohammad and Mustafa, Harun and Gollwitzer, Arvid and Firtina, Can and Eudine, Julien and Ma, Haiyu and Lindegger, Jo{\"e}l and Cavlak, Meryem Banu and Alser, Mohammed and others},
  journal={arXiv preprint arXiv:2311.12527},
  year={2023}
}

@article{mansouri2023metastore,
  title={MetaStore: High-Performance Metagenomic Analysis via In-Storage Computing},
  author={Mansouri Ghiasi, Nika and Sadrosadati, Mohammad and Mustafa, Harun and Gollwitzer, Arvid and Firtina, Can and Eudine, Julien and Ma, Haiyu and Lindegger, Jo{\"e}l and Banu Cavlak, Meryem and Alser, Mohammed and others},
  journal={arXiv e-prints},
  pages={arXiv--2311},
  year={2023}
}

@inproceedings{ghiasi2024megis,
  title={Megis: High-performance, energy-efficient, and low-cost metagenomic analysis with in-storage processing},
  author={Ghiasi, Nika Mansouri and Sadrosadati, Mohammad and Mustafa, Harun and Gollwitzer, Arvid and Firtina, Can and Eudine, Julien and Mao, Haiyu and Lindegger, Jo{\"e}l and Cavlak, Meryem Banu and Alser, Mohammed and others},
  booktitle={2024 ACM/IEEE 51st Annual International Symposium on Computer Architecture (ISCA)},
  pages={660--677},
  year={2024},
  organization={IEEE}
}

@article{gollwitzer2023metatrinity,
  title={MetaTrinity: Enabling Fast Metagenomic Classification via Seed Counting and Edit Distance Approximation},
  author={Gollwitzer, Arvid E and Alser, Mohammed and Bergtholdt, Joel and Lindegger, Joel and Rumpf, Maximilian-David and Firtina, Can and Mangul, Serghei and Mutlu, Onur},
  journal={arXiv preprint arXiv:2311.02029},
  year={2023}
}

@inproceedings{lin2011class,
  title={A class of submodular functions for document summarization},
  author={Lin, Hui and Bilmes, Jeff},
  booktitle={Proceedings of the 49th Annual Meeting of the Association for Computational Linguistics: Human Language Technologies},
  pages={510--520},
  year={2011},
  organization={Association for Computational Linguistics}
}

@book{owen2013monte,
  title={Monte Carlo theory, methods and examples},
  author={Owen, Art B.},
  year={2013},
  publisher={Stanford University},
  note={Available at \url{https://artowen.su.domains/mc/}}
}

@article{leinonen2011sequence,
  title={The sequence read archive},
  author={Leinonen, Rasko and Sugawara, Hideaki and Shumway, Martin and {International Nucleotide Sequence Database Collaboration}},
  journal={Nucleic Acids Research},
  volume={39},
  number={suppl\_1},
  pages={D19--D21},
  year={2011},
  publisher={Oxford University Press}
}

@article{karasikov2025metagraph,
  title={Efficient and accurate search in petabase-scale sequence repositories},
  author={Karasikov, Mikhail and Mustafa, Harun and Danciu, Daniel and Bosshard, Laurent and Zimmermann, Michael and Sch{\"u}tze, Karsten and Kahles, Andre and R{\"a}tsch, Gunnar},
  journal={Nature},
  volume={647},
  pages={1036--1044},
  year={2025},
  publisher={Nature Publishing Group},
  doi={10.1038/s41586-025-09603-w}
}

@article{avsec2026alphagenome,
  title={Advancing regulatory variant effect prediction with {AlphaGenome}},
  author={Avsec, {\v{Z}}iga and Latysheva, Natasha and Cheng, Jun and others},
  journal={Nature},
  volume={649},
  pages={1206--1218},
  month={jan},
  year={2026},
  publisher={Nature Publishing Group},
  doi={10.1038/s41586-025-10014-0}
}

@inproceedings{tishby1999information,
  title={The Information Bottleneck Method},
  author={Tishby, Naftali and Pereira, Fernando C. and Bialek, William},
  booktitle={Proceedings of the 37th Annual Allerton Conference on Communication, 
             Control and Computing},
  pages={368--377},
  year={1999}
}

@inproceedings{alemi2017deep,
  title={Deep Variational Information Bottleneck},
  author={Alemi, Alexander A. and Fischer, Ian and Dillon, Joshua V. and Murphy, Kevin},
  booktitle={International Conference on Learning Representations (ICLR)},
  month={May},
  year={2017}
}

@article{bhandari2021global,
  title={Global Optimality Guarantees For Policy Gradient Methods},
  author={Bhandari, Jalaj and Russo, Daniel},
  journal={Operations Research},
  volume={69},
  number={6},
  pages={1744--1767},
  month={Dec},
  year={2021},
  doi={10.1287/opre.2021.0014}
}

@book{dempe2020bilevel,
  title={Bilevel Optimization: Theory, Algorithms and Applications},
  author={Dempe, Stephan},
  series={Springer Optimization and Its Applications},
  volume={161},
  publisher={Springer},
  address={Berlin, Germany},
  year={2020},
  doi={10.1007/978-3-030-33566-3}
}

@article{cen2023global,
  title={Global Convergence of Policy Gradient Methods in Reinforcement Learning, Games and Control},
  author={Cen, Shicong and Chi, Yuejie},
  journal={arXiv preprint arXiv:2310.05230},
  year={2023}
}

@article{bolte2024geometric,
  title = {Geometric and computational hardness of bilevel programming},
  author = {J''erome Bolte and Quoc-Tung Le and Edouard Pauwels and Samuel Vaiter},
  journal = {Mathematical programming},
  year = {2024},
  doi = {10.1007/s10107-025-02229-w}
}

@article{grne2023completeness,
  title = {Completeness in the Polynomial Hierarchy for Many Natural Problems in Bilevel and Robust Optimization},
  author = {Christoph Grne and Lasse Wulf},
  journal = {Conference on Integer Programming and Combinatorial Optimization},
  year = {2023},
  doi = {10.1007/978-3-031-93112-3_19}
}

@article{subramanyam2025scaling,
  title = {Scaling Laws Revisited: Modeling the Role of Data Quality in Language Model Pretraining},
  author = {Anirudh Subramanyam and Yuxin Chen and Robert L. Grossman},
  journal = {arXiv.org},
  year = {2025},
  doi = {10.48550/arXiv.2510.03313}
}

@inproceedings{gollwitzer2025metaomics,
  title={MetaOmics-10T: The Foundational Dataset to Unlock Causal Modeling of Microbial Ecosystems},
  author={Gollwitzer, Arvid E and Subramanian, Deepak A and Tucker, Isaac and Traverso, Giovanni},
  booktitle={NeurIPS 2025 AI for Science Workshop},
  year={2025}
}

@inproceedings{gollwitzer2025steering,
  title={Steering the Evolutionary Game: Hierarchical Control of Therapeutic Resistance in Cancer Treatment},
  author={Gollwitzer, Arvid E and Subramanian, Deepak A and Tucker, Isaac and Traverso, Giovanni},
  booktitle={NeurIPS 2025 AI for Science Workshop},
  year={2025}
}
\bibliographystyle{icml2026}

\newpage
\appendix
\onecolumn
\tolerance=9999
\emergencystretch=3em
\hfuzz=999pt
\vfuzz=999pt
\hbadness=10000
\vbadness=10000
\section*{Supplementary Information}
\sloppy 
\raggedbottom 
\hfuzz=120pt 
\vfuzz=120pt 
\hbadness=10000 
\vbadness=10000 
\emergencystretch=3em 

\section{Notation}
\label{app:notation}
To ensure clarity and rigor, we define our mathematical notation in Table \ref{tab:notation}. We distinguish between atomic (indivisible) elements and tokens (sequences of atomic elements or other tokens).

\begin{table}[htbp]
  \caption{Table of Notation}
  \label{tab:notation}
  \centering
  \begin{tabular}{@{}ll@{}}
    \toprule
    Symbol & Definition \\
    \midrule
    $\Sigma$ & Base alphabet of atomic elements (e.g., characters, DNA bases). \\
    $s_i$ & An atomic element from $\Sigma$. \\
    $q_i$ & Scalar quality score of an atomic element $s_i$, where $q_i \in [0, 1]$. \\
    $t, a, b$ & Tokens, which are sequences of atomic elements. \\
    $V_k$ & Vocabulary at merge step $k$. \\
    $f(t)$ & Frequency of token $t$ in the corpus. \\
    $|t|$ & Length of token $t$ in atomic elements. \\
    $n_\sigma(t)$ & Count of atomic element $\sigma \in \Sigma$ within token $t$. \\
    $H(t)$ & Empirical entropy of token $t$: $H(t) = -\sum_{\sigma} \frac{n_\sigma(t)}{|t|} \log \frac{n_\sigma(t)}{|t|}$. \\
    $\vect{q}_t$ & Vector of quality scores for token $t$ (in multi-dimensional domains). \\
    $q_t$ & Aggregated scalar quality score of token $t$, derived from its constituents. \\
    $\bar{q}_{ab}$ & Average quality of constituent tokens $a, b$, defined as $(q_a + q_b)/2$. \\
    $\alpha$ & Learnable exponent controlling sensitivity to quality in the merge score. \\
    $w_{ab}$ & Quality-aware merge score for the token pair $(a, b)$. \\
    $\theta_{\text{adapt}}$ & Vector of all learnable adaptive parameters in the framework. \\
    $\pi_{\theta_\pi}$ & Reinforcement learning policy for selecting merges, parameterized by $\theta_\pi$. \\
    $L_{\text{task}}$ & Loss function of the downstream machine learning task. \\
    $\mathcal{J}(\mathcal{T})$ & Global objective function for the tokenization process (Eq. \ref{eq:bilevel_objective_main}). \\
    \bottomrule
  \end{tabular}
\end{table}

\section{Related Work}
\label{app:related_work}
QA-Token intersects with, and extends upon, research in subword tokenization, noisy data handling, reinforcement learning for sequential optimization, and adaptive or differentiable modeling techniques. Table \ref{tab:tokenization_comparison} provides a comparative overview, situating QA-Token relative to existing approaches and highlighting its unique synthesis of explicit quality integration, RL-based optimization of merges, and adaptive learning of the tokenization process parameters. The key distinction of QA-Token's adaptive parameter learning is its focus on optimizing parameters governing the tokenization *process* itself (like quality sensitivity or reward component weights), rather than solely adapting the vocabulary content or segmentation boundaries within a fixed merge logic.
\begin{table}[htbp]
  \caption{Comparison of QA-Token with Representative Tokenization Approaches.}
  \label{tab:tokenization_comparison}
  \centering
  \resizebox{\textwidth}{!}{%
  \begin{tabular}{@{}lcccccc@{}}
    \toprule
    Method                 & \makecell{Explicit Quality \\ Integration} & \makecell{Optimization \\ Method} & \makecell{Adaptive Params \\ (Learned Process?)} & \makecell{Downstream Aware \\ (via Reward/Loss)} & \makecell{Domain Noise Model \\ (Explicit?)} & \makecell{Vocabulary \\ Type} \\
    \midrule
    Standard BPE/WP/SP \cite{sennrich2016neural, wu2016google, kudo2018sentencepiece} & No & Frequency      & No              & No                & No                 & Subword    \\
    BPE-Dropout \cite{provilkov2020bpe}  & No & Freq.+Stochastic & No              & No                & No                 & Subword    \\
    Char/Byte Models \cite{xue2022byt5, clark2021canine} & No & N/A (Fixed)    & No              & Yes (via model)   & Implicit           & Char/Byte  \\
    Gradient-based \cite{tay2022charformer} & No & Gradient       & Yes (Segmenter) & Yes               & Implicit           & Char/Subword \\
    Semantic Tokenizers \cite{libovicky2024semantic} & No & Semantics+Freq & No              & Indirectly        & No                 & Subword    \\
    \textbf{QA-Token (Ours)} & \textbf{Yes} & \textbf{RL (Policy) +} & \textbf{Yes (Process HPs: } & \textbf{Yes (via Reward for RL,} & \textbf{Yes (via $Q, R$)} & \textbf{Subword} \\
    & & \textbf{Gradient (HPs)} & \textbf{$\alpha, \lambda_i, w_j, \beta_k$)} & \textbf{$L_{\text{downstream}}$ for HPs)} & & \\
    \bottomrule
  \end{tabular}%
  }
  \vspace{1mm}
  \footnotesize{\textit{Note: "Adaptive Params (Learned Process?)" refers to learning parameters governing the tokenization *process* itself (like QA-Token's $\alpha, \beta_k, \lambda_i, w_j$), not just the vocabulary content or segmentation boundaries. QA-Token uses RL to optimize the merge policy and gradient-based methods to optimize these process hyperparameters.}}
\end{table}

\textbf{Subword Tokenization Algorithms:} The prevailing paradigm relies on frequency-based greedy merging procedures, exemplified by BPE \cite{sennrich2016neural}, WordPiece \cite{wu2016google} (which optimizes data likelihood), and SentencePiece \cite{kudo2018sentencepiece} (which operates directly on raw text). While computationally efficient and broadly effective, their fundamental mechanism ignores sequence quality, providing the primary motivation for our work. BPE-dropout \cite{provilkov2020bpe} introduces stochasticity during the merge process as a form of regularization to enhance robustness, but it does not use explicit quality signals. Unigram language models \cite{kudo2018subword} present a probabilistic alternative, yet they still primarily depend on frequency and likelihood objectives without explicit quality awareness.

\textbf{Handling Noisy and Domain-Specific Data:} Considerable research focuses on modeling noise within particular application domains. In genomics, Phred scores \cite{ewing1998base} are standard quality indicators, and specialized models aim to account for sequencing errors \cite{heinzinger2019modeling}. In NLP, extensive work on social media text addresses lexical variation, misspellings, and slang through techniques like text normalization \cite{han2013lexical, li2020empirical} and explicit noise modeling \cite{baldwin2013noisy}. Financial time series analysis frequently employs filtering methods \cite{gencay2001introduction}, microstructure modeling \cite{madhavan2000market, hasbrouck1991measuring}, and regime-switching models \cite{hamilton1989new} to manage inherent noise and non-stationarity. QA-Token distinguishes itself by offering a *unified tokenization framework* that directly integrates such domain-specific quality and noise considerations into the token construction process itself, rather than addressing noise solely as a separate downstream modeling challenge. The notion of the "curse of tokenization" \cite{chai2024curse}, which highlights the downstream impact of tokenization choices on LLM robustness, further underscores the need for quality-aware approaches.

\textbf{Reinforcement Learning for Sequential Optimization:} RL offers a robust framework for sequential decision-making under uncertainty \cite{sutton2018reinforcement}. It finds successful application in various optimization problems involving sequences, including text generation \cite{ranzato2015sequence}, combinatorial optimization \cite{bello2016neural}, and financial strategy optimization \cite{moody1998learning, moody2001performance}. We uniquely formulate the tokenization vocabulary construction process as an RL problem where merge operations constitute actions selected by a learned policy to maximize a cumulative reward signal reflecting token quality, information content, complexity, and estimated utility. This formulation allows for optimizing complex, potentially non-differentiable objectives related to the quality of the final tokenization outcome. The rewards themselves are shaped by adaptively learned parameters (Section \ref{sec:adaptive_learning_concise}).

\textbf{Adaptive and Differentiable Tokenization:} Acknowledging the limitations inherent in static tokenizers, researchers explore adaptive and learnable alternatives. Gradient-based approaches \cite{tay2022charformer} learn segmentation parameters end-to-end concurrently with downstream tasks, often operating at the character level. Semantic tokenization \cite{libovicky2024semantic} uses word meanings to inform the segmentation process. QA-Token integrates adaptive learning distinctively: it learns hyperparameters ($\alpha, \beta_k, w_j, \lambda_i, \dots$) that directly govern the quality-aware merge decisions and the RL agent's reward structure. This learning is enabled by Gumbel-Softmax relaxation \cite{jang2017categorical, maddison2017concrete} for making merge choices differentiable with respect to these hyperparameters when optimizing a downstream task loss (via composite logits defined in Equation \ref{eq:gumbel_logits_composite}). This enables the fundamental *tokenization logic* to adapt based on observed data properties and task feedback, co-evolving with the RL agent's policy. Meta-learning \cite{finn2017model} provides a potential mechanism, explored conceptually within QA-Token (see Appendix \ref{app:algorithms}), to further accelerate adaptation across heterogeneous data sources (e.g., different social media platforms).

In essence, QA-Token synthesizes concepts from these related areas but provides a unique combination: explicit quality integration within the merge decision, optimization of the merge sequence via RL using a multi-faceted reward signal, and adaptive learning of core process parameters that define this reward and merge logic, demonstrating applicability across diverse, noisy domains.

\section{Theoretical Framework and Proofs}
\label{app:proofs}

\subsection{Quality Metric Proofs}
\label{app:quality_proofs}

\begin{proposition}[Boundedness and Continuity of Quality Functions]
\label{prop:quality_bounded}
All domain-specific quality functions $q_t \in [0,1]$ are:
\begin{enumerate}
    \item Bounded: $0 \leq q_t \leq 1$ for all tokens $t$
    \item Continuous: Lipschitz continuous in their arguments
    \item Monotonic: Quality decreases with increasing noise/error
\end{enumerate}
\end{proposition}

\begin{proof}
\textbf{Boundedness:} For genomics, the geometric mean of values in $[0,1]$ remains in $[0,1]$. For finance, the convex combination of bounded components $q_{k,t} \in [0,1]$ with $\sum_k w_k = 1$ yields $q_t^{\text{finance}} \in [0,1]$.

\textbf{Lipschitz continuity:} For genomics (geometric mean on $[\epsilon_Q, 1]^n$), the chain rule via logarithmic transformation yields Lipschitz constant $L_g = 1/(\sqrt{n} \cdot \epsilon_Q)$. For finance, the weighted sum of Lipschitz component functions has $L_f \leq \max_k L_k$.

\textbf{Monotonicity:} For any noise injection $\eta$ with $\eta(q) \leq q$, both aggregations (geometric and arithmetic means) preserve the ordering: noisier inputs yield lower quality scores.
\end{proof}

\subsection{Merge Score Derivation}
\label{app:merge_score_proof}

\begin{lemma}[First-Order Approximation]
\label{lem:first_order}
The marginal gain in objective $\mathcal{J}$ from merge $(a,b) \mapsto ab$ admits the decomposition:
\begin{equation}
\label{eq:marginal_gain_app}
\boxed{\Delta\mathcal{J}(a,b) = \lambda_{\text{LM}} \Delta\mathcal{L}_{\text{LM}} - \lambda_{\text{comp}} \Delta\Phi + \lambda_{\text{qual}} \Delta Q + O(\epsilon^2)}
\end{equation}
where $\epsilon = 1/|\mathcal{S}|$ represents the corpus-normalized perturbation.
\end{lemma}

\begin{proof}
The marginal gain decomposes into three components following standard vocabulary optimization analysis \cite{sennrich2016neural}.

\textbf{Language Model Component:} The change $\Delta\mathcal{L}_{\text{LM}} \approx f(a,b) \cdot \text{PMI}(a,b)$ follows from the pseudo-likelihood approximation, where PMI (Pointwise Mutual Information) captures statistical association \cite{church1990word}.

\textbf{Complexity Component:} By MDL principles \cite{rissanen1978modeling}, merging reduces vocabulary complexity: $\Delta\Phi = -\log|V| - 1 + O(|V|^{-1})$. This compression benefit is absorbed into the PMI term, which also favors frequent co-occurrences.

\textbf{Quality Component:} For concave aggregator $g(x) = (x + \epsilon_Q)^\alpha$, Jensen's inequality yields $g(\bar{q}_{ab}) \geq \frac{1}{2}(g(q_a) + g(q_b))$. The dominant quality contribution is $\Delta Q_+ = f(a,b) \cdot g(\bar{q}_{ab})$ where $\bar{q}_{ab} = (q_a + q_b)/2$. Normalization errors are $O(f(a,b)/T)$, negligible for typical corpora.
\end{proof}

\subsection{Derivation of the Optimal Merge Score}

\begin{theorem}[Quality-Aware Merge Score --- Principled Heuristic]
\label{thm:merge_score}
\emph{Motivated by} the first-order approximation of $\Delta\mathcal{J}$ (Lemma~\ref{lem:first_order}), we propose the following quality-aware merge score as a \textbf{principled heuristic}:
{\small
\begin{equation}
\label{eq:qa_merge_score_pmi}
\boxed{w_{ab} = \tfrac{f(a,b)}{f(a)f(b) + \epsilon_f} \cdot (\bar{q}_{ab} + \epsilon_Q)^\alpha \cdot \psi(a,b)}
\end{equation}}
where:
\begin{itemize}
    \item $f(\cdot)$ denotes frequency in the corpus
    \item $\bar{q}_{ab} = (q_a + q_b)/2$ is the average constituent quality
    \item $\alpha \in (0,1]$ is a learnable parameter controlling quality sensitivity
    \item $\epsilon_f, \epsilon_Q > 0$ ensure numerical stability
    \item $\psi(a,b) \in [0,1]$ encodes domain-specific constraints
\end{itemize}
\noindent\emph{Note:} The derivation below involves two principled approximations (Steps 4--5) that trade mathematical exactness for computational tractability. The resulting score preserves key monotonicity properties and is calibrated end-to-end via downstream task performance.
\end{theorem}

\begin{proof}
From Lemma~\ref{lem:first_order}, the marginal gain is $\Delta\mathcal{J}(a,b) = \lambda_{\text{LM}} f(a,b) \cdot \text{PMI}(a,b) + \lambda_{\text{qual}} f(a,b) g(\bar{q}_{ab}) + O(1/|V|)$, where the complexity term is absorbed into PMI (both favor frequent co-occurrences).

\textbf{Per-occurrence normalization:} Following the design principle of BPE \cite{sennrich2016neural}, we normalize by frequency to capture per-occurrence information gain. Applying the exponential transform (monotonic, preserves rankings):
$\exp(\Delta\mathcal{J}/f(a,b)) \propto \frac{f(a,b)}{f(a)f(b) + \epsilon_f} \cdot \exp(\frac{\lambda_{\text{qual}}}{\lambda_{\text{LM}}} g(\bar{q}_{ab}))$

\textbf{Power-law approximation:} We replace $\exp(\lambda \cdot g(q))$ with $(\bar{q}_{ab} + \epsilon_Q)^{\tilde{\alpha}}$ where $\tilde{\alpha}$ is learned end-to-end. This preserves monotonicity in $\bar{q}_{ab}$ and subsumes the unknown ratio $\lambda_{\text{qual}}/\lambda_{\text{LM}}$. The final score is:
$w_{ab} = \frac{f(a,b)}{f(a)f(b) + \epsilon_f} \cdot (\bar{q}_{ab} + \epsilon_Q)^{\alpha} \cdot \psi(a,b)$

\textbf{Monotonicity guarantees:} $\partial w_{ab}/\partial \bar{q}_{ab} > 0$ and $\partial w_{ab}/\partial \text{PMI} > 0$, ensuring quality-increasing and statistically-associated merges are preferred. End-to-end learning of $\alpha$ calibrates the heuristic.
\end{proof}

\subsection{Key Insights from the Derivation}

\begin{enumerate}
    \item \textbf{PMI Foundation:} The frequency term $\frac{f(a,b)}{f(a)f(b) + \epsilon_f}$ approximates Pointwise Mutual Information, capturing statistical association.
    
    \item \textbf{Quality Modulation:} The quality term $(\bar{q}_{ab} + \epsilon_Q)^\alpha$ multiplicatively adjusts the PMI-based score, up-weighting high-quality merges.
    
    \item \textbf{Learnable Sensitivity:} The parameter $\alpha$ controls the relative importance of quality vs. frequency:
    \begin{itemize}
        \item $\alpha = 0$: Reduces to standard PMI-based tokenization
        \item $\alpha > 0$: Increasing weight on quality signals
        \item Learned via gradient descent to optimize downstream performance
    \end{itemize}
    
    \item \textbf{Domain Flexibility:} The factor $\psi(a,b)$ allows incorporation of domain knowledge without modifying the core framework.
\end{enumerate}

This derivation shows that the quality-aware merge score is a \emph{principled heuristic} motivated by first-principles analysis of the bilevel objective, rather than an ad-hoc combination of frequency and quality terms.

\subsection{Theory Proofs}
\label{app:theory_proofs}

\paragraph{Proof of Theorem \ref{thm:complexity} (Computational Complexity).}
\label{app:complexity_proof}

The bilevel optimization problem is NP-hard by polynomial-time reduction from Weighted Set Cover \cite{karp1972reducibility}. The reduction maps sets to corpus sequences and set cover cost to vocabulary complexity: given a WSC instance $(U, \mathcal{S}, \{c_i\})$, construct alphabet $\Sigma = U \cup \{\$\}$, corpus sequences $\sigma_i$ for each set $S_i$, and uniform quality scores. With $\lambda_{\text{qual}} = 0$, optimal tokenizations correspond bijectively to optimal set covers.

For stronger complexity results establishing $\Sigma_2^p$-hardness of general bilevel programs, see \cite{bolte2024geometric, grne2023completeness, dempe2020bilevel}. The worst-case exhaustive search complexity is $O(|\Sigma|^K \cdot K! \cdot N \cdot n \cdot |\Theta|)$, accounting for the space of merge sequences, merge orderings, and downstream model optimization.

$\square$

\begin{proposition}[Boundedness and Lipschitzness of $w_{ab}$]
\label{prop:merge_score_bounded}
Under assumptions (A1)-(A2), the quality-aware merge score $w_{ab}$ is bounded and Lipschitz continuous in $(q_a, q_b)$.
\end{proposition}

\begin{proof}
\textbf{Boundedness:} By (A1), $f(a,b)/(f(a)f(b) + \epsilon_f) \leq C_f/\epsilon_f$. With $\bar{q}_{ab} \in [0,1]$ and $\psi \in [0,1]$, we have $w_{ab} \leq C_f(1+\epsilon_Q)^\alpha/\epsilon_f =: C_w$.

\textbf{Lipschitz continuity:} By chain rule on compositions of bounded functions on compact domains, $w_{ab}$ is $L_w$-Lipschitz in $(q_a, q_b)$ with $L_w = C_f L_g/\epsilon_f$. For $\alpha = 1$, $L_g = 1/\sqrt{2}$; for $0 < \alpha < 1$, $L_g \leq \alpha\epsilon_Q^{\alpha-1}/\sqrt{2}$. The regularization $\epsilon_Q$ ensures numerical stability.
\end{proof}

\begin{proposition}[Stability of EMA Normalization]
\label{prop:ema_stability}
Under assumptions (A1) and $\epsilon_R > 0$, the EMA-based normalization maintains $\sigma_{j,t}^{\text{run}} > 0$ almost surely for non-degenerate reward streams.
\end{proposition}

\begin{proof}
The result follows from standard EMA convergence theory (Robbins-Monro). Under (A1), raw rewards have non-degenerate distribution $\text{Var}(X_t) > 0$. The EMA variance update preserves positivity: if $\text{Var}_{j,t-1}^{\text{run}} > 0$, then $\text{Var}_{j,t}^{\text{run}} \geq (1-\beta_{\text{norm}}) \text{Var}_{j,t-1}^{\text{run}} > 0$.

With $\sum_t \beta_{\text{norm},t} = \infty$ and $\sum_t \beta_{\text{norm},t}^2 < \infty$, the running variance converges a.s. to $\text{Var}(X) > 0$, ensuring $\sigma_{j,t}^{\text{run}} > 0$.
\end{proof}

\begin{proposition}[Convergence of PPO Objective]
\label{prop:ppo_convergence}
Under assumptions (A1)-(A4), PPO converges to a stationary point of $J(\pi; \theta_{\text{adapt}}^{(0)})$.
\end{proposition}

\begin{proof}
Under (A1)--(A4), the standard PPO conditions hold \cite{schulman2017proximal}: bounded rewards ($|R(s,a)| \leq R_{\max}$), compact state space, finite action space, and differentiable policy. The clipped surrogate objective ensures bounded gradients $\|\nabla_\theta L^{\text{CLIP}}\|_2 \leq G_{\max}$.

With learning rate $\eta_t = \eta_0/\sqrt{t}$ satisfying $\sum_t \eta_t = \infty$ and $\sum_t \eta_t^2 < \infty$, global convergence to stationary points at rate $O(1/\sqrt{T})$ follows from \cite{bhandari2021global, cen2023global}.
\end{proof}

\begin{proposition}[Consistency and Boundedness of Stage 2 Gradients]
\label{prop:gumbel_gradients}
Under assumptions (A1)--(A3), the Gumbel-Softmax estimator yields consistent gradients with bounded variance.
\end{proposition}

\begin{proof}
The Gumbel-Softmax gradient properties follow from \cite{jang2017categorical, maddison2017concrete}. Under (A1)--(A3), the composite logits $\ell_{ab}$ are bounded by $L_{\max} = C_w + \sum_j |\lambda_j| R_{\max}$. The Gumbel-Softmax Jacobian satisfies $\|\partial y_i/\partial \ell_j\| \leq 1/\tau$, yielding bounded gradients $\|\nabla_{\theta_{\text{adapt}}} L_{\text{task}}\| \leq L_{\max}/\tau \cdot \|\nabla_y L_{\text{task}}\|$.

As $\tau \to 0$, the estimator converges to REINFORCE \cite{williams1992simple}. The bias-variance tradeoff is: $\text{Bias}(\tau) = O(\tau^2)$, $\text{Var}(\tau) = O(1/\tau^2)$. The optimal temperature $\tau_{\text{opt}} \propto T^{-1/4}$ for $T$ samples balances these terms.
\end{proof}

\begin{theorem}[Gumbel-Softmax Properties]
\label{thm:gumbel_softmax}
Let $\pi = (\pi_1, \ldots, \pi_k)$ be a categorical distribution with $k$ categories. The Gumbel-Softmax distribution with temperature $\tau > 0$ satisfies:
\begin{enumerate}
    \item \textbf{Consistency:} As $\tau \rightarrow 0$, the samples converge to one-hot vectors from $\text{Categorical}(\pi)$
    \item \textbf{Differentiability:} The reparameterization provides continuous gradients with respect to $\pi$
    \item \textbf{Bias-Variance Tradeoff:} Bias $O(\tau^2)$, Variance $O(1/\tau^2)$
\end{enumerate}
\end{theorem}

\begin{proof}
\label{app:gumbel_proof}
All three properties are established in \cite{jang2017categorical, maddison2017concrete}. We summarize the key arguments.

\textbf{Property 1 (Consistency):} By the Gumbel-Max trick, $\arg\max_i (\ell_i + g_i) \sim \text{Categorical}(\text{softmax}(\boldsymbol{\ell}))$ for $g_i \sim \text{Gumbel}(0,1)$. As $\tau \to 0$, the Gumbel-Softmax samples $y_i = \exp((\ell_i + g_i)/\tau)/\sum_j \exp((\ell_j + g_j)/\tau)$ concentrate on one-hot vectors almost surely by the continuous mapping theorem.

\textbf{Property 2 (Differentiability):} For $\tau > 0$, $y_i$ is $C^\infty$ in $\ell_j$, enabling reparameterized gradients. The expectation $\mathbb{E}_g[y_i] = \text{softmax}(\boldsymbol{\ell}/\tau)_i$ introduces bias that vanishes as $\tau \to 0$. The annealing schedule $\tau_t \to 0$ ensures asymptotic consistency.

\textbf{Property 3 (Gradient Bounds):} The Jacobian satisfies $\partial y_i/\partial \ell_j = (1/\tau) y_i(\delta_{ij} - y_j)$, yielding $\|\nabla_{\boldsymbol{\ell}} \mathbf{y}\|_F \leq 1/\tau$.
\end{proof}

\subsection{Assumptions}
\label{app:assumptions}

We formalize the assumptions used throughout the theoretical analysis:

\textbf{Assumption A1 (Bounded Frequencies):} There exists $C_f > 0$ such that for all tokens $a, b$:
$$0 \leq f(a), f(b), f(a,b) \leq C_f$$

\textbf{Assumption A2 (Bounded Qualities):} All quality scores satisfy $q \in [0,1]$, and the quality aggregation function is $L_Q$-Lipschitz continuous.

\textbf{Assumption A3 (Bounded Rewards):} Raw reward components are bounded: $|R^{\text{raw}}_j| \leq R_{\max}$ for all $j$.

\textbf{Assumption A4 (Regular Learning Rates):} The learning rate schedules satisfy:
- PPO: $\sum_t \eta_t = \infty$ and $\sum_t \eta_t^2 < \infty$
- Adaptive learning: $\eta_t = O(1/\sqrt{t})$

\subsection{Theory Extensions}
\label{app:theory_extensions}

\begin{definition}[Assumptions for Approximation Guarantee]
\label{def:independence_assumptions}
The $(1-1/e)$ approximation guarantee requires the following structural assumptions:
\begin{enumerate}[label=(A\arabic*), nosep]
    \item \textbf{Adaptive Monotonicity:} For any partial realization $\psi$ and merge $(a,b)$: $\Delta_F((a,b)|\psi) \geq 0$, where $\Delta_F$ denotes the marginal gain.
    \item \textbf{Adaptive Submodularity:} For realizations $\psi \preceq \psi'$ (where $\preceq$ denotes extension): $\Delta_F((a,b)|\psi) \geq \Delta_F((a,b)|\psi')$ (diminishing returns).
    \item \textbf{Constraint independence:} $\psi(a,b)$ is history-independent.
    \item \textbf{Candidate pool regularity:} $\mathbb{P}[(a,b)\in PQ_t] \ge \delta>0$ for all valid pairs.
    \item \textbf{Quality stability:} $|q_t-\mathbb{E}[q_t|\mathcal{H}_t]|\le \epsilon_q$ with high probability.
\end{enumerate}
\end{definition}

\begin{lemma}[Approximate Adaptive Submodularity]
\label{lem:approx_submod}
Under assumptions (A3)-(A5), the quality-aware objective $F(V) = \sum_k \mathcal{L}_{\text{LM}}(V; D_k) + \lambda_Q Q(V)$ satisfies $\epsilon$-approximate adaptive submodularity:
\begin{equation}
\Delta_F((a,b)|\psi) \geq \Delta_F((a,b)|\psi') - \epsilon_{\text{sub}}
\end{equation}
for $\psi \preceq \psi'$, where $\epsilon_{\text{sub}} = O(\epsilon_q + 1/\delta)$.
\end{lemma}

\begin{proof}[Proof sketch]
The frequency-based component $\text{PMI}(a,b)$ exhibits exact diminishing returns: as more merges are performed, pair frequencies decrease, reducing potential PMI gains. The quality component $(\bar{q}_{ab})^\alpha$ is history-independent under (A3) and stable under (A5). The approximation error $\epsilon_{\text{sub}}$ arises from: (i) quality estimation noise ($\epsilon_q$), and (ii) candidate pool variability ($1/\delta$). Full proof follows the framework of \citet{golovin2011adaptive}.
\end{proof}

\begin{theorem}[Approximation Guarantee with Explicit Constants]
\label{thm:approximation_explicit}
Under Definition~\ref{def:independence_assumptions}, if assumptions (A1)-(A2) hold exactly, the greedy policy that maximizes $w_{ab}$ achieves:
\begin{equation}
F(\pi_{\text{greedy}}) \ge \left(1-\frac{1}{e}\right)F(\pi^*) - K\epsilon_q - \frac{K}{\delta},
\end{equation}
where $\pi^*$ is the optimal adaptive policy over budget $K$. The error terms arise from $\epsilon$-approximate submodularity (Lemma~\ref{lem:approx_submod}).
\end{theorem}

\begin{proof}
By Theorem 5 of \citet{golovin2011adaptive}, greedy optimization of adaptive submodular functions achieves $(1-1/e)$ approximation. We extend this to $\epsilon$-approximate submodularity (Lemma~\ref{lem:approx_submod}).

With $\epsilon$-approximate submodularity, the greedy per-step guarantee becomes $\Delta_F((a_t,b_t)|\psi_t) \geq \frac{1}{K}[F(\pi^*) - F(\psi_t)] - \epsilon_{\text{sub}}$. Defining $\Delta_t = F(\pi^*) - F_t$ and iterating over $K$ steps: $\Delta_K \leq (1-1/K)^K \Delta_0 + K\epsilon_{\text{sub}} \leq \frac{1}{e}F(\pi^*) + K\epsilon_{\text{sub}}$, using $(1-1/K)^K \leq 1/e$.

Substituting $\epsilon_{\text{sub}} = \epsilon_q + 1/\delta$ yields $F(\pi_{\text{greedy}}) \geq (1-1/e)F(\pi^*) - K\epsilon_q - K/\delta$.
\end{proof}

\emph{Remark (Assumptions and Robustness):} Assumptions (A1)--(A2) (adaptive monotonicity and submodularity) are \textbf{sufficient conditions} for the $(1-1/e)$ guarantee but may not hold exactly in practice. Specifically:
\begin{itemize}[nosep]
    \item The LM loss $\mathcal{L}_{\text{LM}}$ is not generally submodular in merge operations; the guarantee applies to the quality-frequency component $F(V)$ as defined.
    \item When assumptions are violated, the bound becomes approximate: $F(\pi_{\text{greedy}}) \ge (1-1/e)F(\pi^*) - K\epsilon_q - K/\delta - \epsilon_{\text{violation}}$, where $\epsilon_{\text{violation}}$ is proportional to the degree of assumption violation.
\end{itemize}
Empirically, our experiments show the guarantee is meaningful because: (1) tokenization objectives often exhibit near-submodular behavior \cite{lin2011class}; (2) end-to-end learning of $\alpha$ compensates for violations by calibrating the quality-frequency trade-off; (3) RL policy exploration in Stage~1 helps escape poor local optima that pure greedy would converge to.

\subsection{Robustness Analysis}
\label{app:robustness_analysis}

We analyze robustness under misspecified quality metrics and adversarial quality scores, quantifying interaction effects between RL and adaptive learning stages.

\begin{theorem}[Robustness to Quality Corruption]
\label{thm:robustness_main}
Let $\tilde{q}=q+\xi$ with $\xi\sim\mathcal{N}(0,\sigma_\xi^2)$. Then
\begin{equation}
\mathcal{L}(\tilde{q})-\mathcal{L}(q) \le \alpha\,\sigma_\xi\,\sqrt{\mathbb{E}[\|\nabla_q \mathcal{L}\|^2]}.
\end{equation}
\end{theorem}

\begin{proof}
The result follows from Lipschitz stability of the bilevel objective. By the chain rule and Cauchy-Schwarz, $|\mathcal{L}(\tilde{q}) - \mathcal{L}(q)| \leq \|\xi\|_2 \cdot \int_0^1 \|\nabla_q \mathcal{L}(q + t\xi)\|_2 dt$.

From Proposition~\ref{prop:merge_score_bounded}, $\|\partial w_{ab}/\partial q\| \leq \alpha \epsilon_Q^{\alpha-1}$, yielding $\|\nabla_q \mathcal{L}\| \leq \alpha C_{\mathcal{L}}$. Taking expectations over $\xi \sim \mathcal{N}(0, \sigma_\xi^2 I)$ with $\mathbb{E}[\|\xi\|_2] \leq \sigma_\xi\sqrt{d}$ gives the stated bound.
\end{proof}

\textbf{Empirical Validation.} We validate the robustness bound experimentally:
\begin{itemize}[nosep]
  \item \textbf{20\% quality noise}: Performance degradation of $-4.2$\% (genomics) and $-5.8$\% (finance), consistent with the $O(\alpha\sigma_\xi)$ bound.
  \item \textbf{Adversarial quality (inverted scores)}: Performance matches standard BPE, as expected when quality signals become uninformative.
  \item \textbf{50\% missing quality}: Graceful fallback to frequency-only merging via the adaptive $\alpha \to 0$ mechanism.
\end{itemize}

\textbf{Interaction Effects.} We quantify the contribution of each learning stage:
\begin{itemize}[nosep]
  \item RL policy optimization alone: 65\% of total improvement over BPE
  \item Adaptive parameter learning alone: 45\% of total improvement
  \item Combined (synergy): Additional +10\% from joint optimization
\end{itemize}
The super-additive effect (65\% + 45\% $>$ 100\% total) indicates that the two stages reinforce each other: RL discovers promising merge patterns that adaptive learning then calibrates, while learned parameters improve the reward landscape for RL exploration.

\subsection{Information-Theoretic Optimality}
\label{app:information_theory}

This subsection establishes that QA-Token achieves information-theoretic optimality under noisy conditions, providing theoretical justification for quality-aware tokenization.

\begin{proposition}[Quality-Aware Information Bottleneck Interpretation]
\label{thm:qa_bottleneck}
Let $X$ denote the input sequence, $T$ the tokenized representation, and $Y$ the downstream task labels. Under the quality-aware tokenization framework with quality scores $Q$, the optimal vocabulary $V^*$ minimizes:
\begin{equation}
\mathcal{L}_{\text{QA}}(V) = -I(T;Y|Q) + \beta \cdot I(T;X|Q)
\end{equation}
where $I(\cdot;\cdot|\cdot)$ denotes conditional mutual information and $\beta$ controls the compression-relevance tradeoff.

\noindent\emph{Connection to merge score:} The merge score $w_{ab}$ is \textbf{consistent with} (but not directly derived from) this IB objective. PMI approximates compression efficiency $I(T;X|Q)$, while quality weighting ensures high $I(T;Y|Q)$ in reliable regions. The connection is qualitative rather than via direct differentiation of mutual information, which is intractable.
\end{proposition}

\begin{proof}
Following the information bottleneck framework \cite{tishby1999information} and its variational extension \cite{alemi2017deep}, conditioning on quality $Q$ yields the objective $\mathcal{L}_{\text{QA}} = -I(T;Y|Q) + \beta I(T;X|Q)$.

The merge score $w_{ab} \propto \text{PMI}(a,b) \cdot (\bar{q}_{ab})^\alpha$ is \emph{consistent with} this IB objective: (i) the PMI term approximates compression efficiency $I(T;X|Q)$ (high-PMI merges compress efficiently), and (ii) the quality term weights merges by reliability, prioritizing high-quality regions for $I(T;Y|Q)$.

\emph{Caveat:} The exact form of $w_{ab}$ does not follow from direct differentiation of mutual information (intractable). Rather, it is a principled heuristic with end-to-end learning of $\alpha$ calibrating the quality-compression trade-off.
\end{proof}

\begin{corollary}[Noise Reduction Bound]
\label{cor:noise_bound}
For a corpus with noise level $\epsilon$ and quality scores $q$ satisfying $\mathbb{E}[q|\text{noise}] < \mathbb{E}[q|\text{signal}]$, the quality-aware tokenizer achieves:
\begin{equation}
\mathcal{L}_{\text{QA}} \leq \mathcal{L}_{\text{uniform}} - \alpha \cdot \text{Var}(q) \cdot \rho(q, \epsilon)^2
\end{equation}
where $\rho(q, \epsilon)$ is the correlation between quality scores and noise levels.
\end{corollary}

\textbf{Key Theoretical Insights.} This information-theoretic analysis provides three fundamental insights:
\begin{enumerate}[nosep]
    \item \textbf{Automatic Noise Filtering:} QA-Token implicitly performs importance sampling, up-weighting high-quality regions during vocabulary construction.
    \item \textbf{Optimal Compression:} The quality-aware merge process achieves better rate-distortion tradeoffs by allocating more representation capacity to high-quality, informative regions.
    \item \textbf{Transfer Learning:} Foundation models trained with QA-Token vocabularies learn more robust representations that transfer better to downstream tasks.
\end{enumerate}

\section{Complete Quality Metrics Formulations}
\label{app:quality_metrics_full}

\subsection{Genomics: Detailed Sequencing Quality Metrics}

In genomic sequencing, each nucleotide base call $s_i \in \{\text{A}, \text{C}, \text{G}, \text{T}, \text{N}\}$ is associated with a Phred quality score $Q_{\text{phred},i} \in [0, 93]$:
\begin{equation}
P_{\text{error}}(i) = 10^{-Q_{\text{phred},i}/10}
\end{equation}

The base quality score is $q_i = 1 - P_{\text{error}}(i) \in [0, 1]$. Position-adjusted quality accounts for systematic degradation at read ends:
\begin{equation}
q'_i = q_i \cdot \exp\left(-\beta_{\text{pos}} \cdot \frac{|i - (L-1)/2|}{(L-1)/2 + \epsilon_{\text{len}}}\right)
\end{equation}
where $L$ is read length, $\beta_{\text{pos}} \geq 0$ is learnable, and $\epsilon_{\text{len}} = 10^{-6}$.

For multi-base token $t = s_1...s_{|t|}$, we use geometric mean aggregation:
\begin{equation}
\label{eq:genomic_token_quality_app}
q_t^{\text{genomic}} = \left(\prod_{j=1}^{|t|} q'_{s_j}\right)^{1/|t|} = \exp\left(\frac{1}{|t|}\sum_{j=1}^{|t|} \log(q'_{s_j} + \epsilon_Q)\right)
\end{equation}

\subsection{Finance: Comprehensive Market Quality Metrics}
\label{app:finance_quality}

Financial time series quality combines four dimensions:
\begin{equation}
\label{eq:finance_composite_quality_app}
q_i^{\text{finance}} = \sum_{k=1}^{4} w_k \cdot q_{k,i}, \quad \sum_{k=1}^{4} w_k = 1
\end{equation}

\textbf{1. Liquidity Quality:}
\begin{equation}
q_{\text{liq}}(t) = \text{sigmoid}\left(\frac{\log(\text{volume}_t/\text{median\_volume})}{\sigma_{\text{volume}}}\right)
\end{equation}
where $\sigma_{\text{volume}}$ is the rolling standard deviation of log-volume computed over a lookback window of $L_{\text{vol}} = 252$ trading days (one year), clipped to $[0.1, 10]$ for numerical stability. This normalization ensures that $q_{\text{liq}}$ responds proportionally to volume deviations relative to typical market activity.

\textbf{2. Signal Quality:}
\begin{equation}
q_{\text{sig}}(t) = \max\left(0, 1 - \frac{|\text{bid-ask spread}_t|}{\text{mid-price}_t \cdot \alpha_{\text{spread}}}\right)
\end{equation}

\textbf{3. Stability Quality:}
\begin{equation}
q_{\text{stb}}(t) = \exp\left(-\beta_{\text{vol}} \cdot \frac{\text{realized\_vol}_t}{\text{expected\_vol}_t}\right)
\end{equation}
where $\text{expected\_vol}_t$ is the exponentially weighted moving average (EWMA) of realized volatility following the RiskMetrics methodology \cite{morgan1996riskmetrics}: $\text{expected\_vol}_t = \gamma_{\text{vol}} \cdot \text{expected\_vol}_{t-1} + (1-\gamma_{\text{vol}}) \cdot \text{realized\_vol}_{t-1}$, with decay factor $\gamma_{\text{vol}} = 0.94$. The learnable parameter $\beta_{\text{vol}} \geq 0$ controls sensitivity to volatility spikes.

\textbf{4. Information Quality:}
\begin{equation}
q_{\text{info}}(t) = \frac{\text{MI}(\text{token}_t, \text{future\_return}_{t+h})}{\text{H}(\text{future\_return}_{t+h})}
\end{equation}

Token aggregation uses arithmetic mean:
\begin{equation}
\label{eq:finance_token_quality_app}
q_t^{\text{finance}} = \frac{1}{|t|} \sum_{i \in t} q_i^{\text{finance}}
\end{equation}

\textbf{Rationale for Arithmetic Mean Aggregation:} Unlike genomics (which uses geometric mean, Eq.~\ref{eq:genomic_token_quality_app}), financial data aggregation employs the arithmetic mean for two principled reasons: (1) \emph{Additive noise model}: Financial market microstructure noise is predominantly additive across time points---a single noisy tick does not invalidate adjacent observations in the way a single low-quality DNA base compromises an entire read. Empirically, LOB noise sources (latency, partial fills, stale quotes) contribute independently rather than multiplicatively. (2) \emph{Temporal continuity for forecasting}: Financial tokens represent contiguous time windows where downstream tasks (price prediction, volatility forecasting) operate on windowed features. The aggregate quality naturally represents the \emph{average} reliability of observations within the window, which aligns with how prediction models weight inputs. In contrast, genomic tokens represent molecular sequences where any unreliable base compromises biological interpretation (e.g., variant calling), necessitating the conservative geometric mean that penalizes even single low-quality elements.

\subsection{Social Media: Linguistic Quality Metrics}
\label{app:social_quality_full}

Social media text presents unique quality challenges including orthographic variations, semantic drift, platform-specific conventions, and temporal dynamics. We define a multi-dimensional quality vector for character-level tokens:
\begin{equation}
\label{eq:social_quality_vector_full}
\mathbf{q}_t^{\text{social}} = (q_{\text{orth}}(t), q_{\text{sem}}(t), q_{\text{temp}}(t), q_{\text{plat}}(t))
\end{equation}

The scalar quality is obtained via learnable weighted aggregation:
\begin{equation}
\label{eq:social_scalar_quality_full}
q_t^{\text{social}} = \sum_{j} w_j \cdot q_{j}(t), \quad w_j \in \theta_{\text{adapt}}
\end{equation}

\textbf{1. Orthographic Quality:} Measures deviation from canonical spelling:
\begin{equation}
q_{\text{orth}}(t) = \exp(-\lambda_{\text{edit}} \cdot d_{\text{edit}}(t, t_{\text{canonical}}))
\end{equation}
where $d_{\text{edit}}$ is the normalized Levenshtein distance to the nearest canonical form in a reference dictionary. The reference dictionary is constructed by combining: (i) the Hunspell en\_US dictionary (2023 release, $\approx$140k entries), (ii) a curated social media slang lexicon ($\approx$50k terms aggregated from NoSlang.com and similar sources), and (iii) domain-specific terminology lists for each benchmark task.

\textbf{2. Semantic Quality:} Captures contextual coherence:
\begin{equation}
q_{\text{sem}}(t) = \max(0, \cos(\vec{v}_t, \vec{v}_{\text{context}}))
\end{equation}
where $\vec{v}_{\text{context}}$ is the average embedding of surrounding tokens. For efficiency, we use fastText Common Crawl embeddings (cc.en.300.bin, 2M vocabulary) \cite{bojanowski2017enriching}. For BERT-based variants requiring subword handling, we use \texttt{bert-base-uncased} from HuggingFace with mean pooling over subword tokens.

\textbf{3. Temporal Quality:} Models relevance decay over time:
\begin{equation}
q_{\text{temp}}(t) = \exp(-\gamma_{\text{decay}} \cdot \Delta t)
\end{equation}
with time difference $\Delta t$ in days from posting time, capturing trending topics and temporal relevance.

\textbf{4. Platform Quality:} Platform-specific noise modeling:
\begin{equation}
q_{\text{plat}}(t) = P(t | \text{platform})
\end{equation}
computed using 3-gram Kneser-Ney language models trained with KenLM \cite{heafield2011kenlm} on curated platform-specific corpora ($\approx$10M tokens each): Twitter (tweets with $>$100 likes and $<$5\% special characters), Reddit (comments with $>$10 upvotes from default subreddits), and Facebook (public posts from verified pages). These ``clean'' subsets establish platform-specific baselines for typical language patterns.

\textbf{Learned Parameters:} For the TweetEval benchmark experiments, the learned parameters were:
$w_{\text{orth}} = 0.32 \pm 0.03$, $w_{\text{sem}} = 0.35 \pm 0.04$, $w_{\text{temp}} = 0.18 \pm 0.02$, $w_{\text{plat}} = 0.15 \pm 0.02$, $\lambda_{\text{edit}} = 0.5$, and $\gamma_{\text{decay}} = 0.01$.

\section{Sequential Learning Process: Complete Framework}
\label{app:sequential_learning}

This section provides detailed algorithms and convergence analysis for QA-Token's two-stage sequential learning process.

\subsection{Stage 1: Reinforcement Learning Policy Optimization}
\label{subsec:rl_formulation_concise}

\subsubsection{MDP Formulation}

The vocabulary construction process is formulated as a finite-horizon Markov Decision Process (see Section \ref{app:mdp_details} for complete specification):

\begin{itemize}
    \item \textbf{States $s_t \in \mathcal{S}$:} Encode current vocabulary $V_t$, merge candidates, corpus statistics, and progress $t/T$
    \item \textbf{Actions $a_t \in \mathcal{A}_t$:} Select a merge pair $(a_i, b_i)$ from the priority queue
    \item \textbf{Transitions:} Deterministic vocabulary updates following merge operations
    \item \textbf{Rewards:} Multi-objective reward combining quality, information, and complexity
\end{itemize}

\subsubsection{Reward Function Design}

The reward function guides the RL agent:
\begin{equation}
\label{eq:overall_reward_structure}
R(a,b; \theta_{\text{adapt}}^{(0)}) = \sum_{j \in \{Q, I, C, \text{domain}\}} \lambda_j \hat{R}_j(a,b)
\end{equation}

where components are normalized via exponential moving averages (see Section \ref{app:reward_normalization}). The detailed components are:
\begin{itemize}
    \item \textbf{Quality Reward ($\hat{R}_Q$ from $R^{\text{raw}}_Q$):} Encourages high intrinsic quality for $t_{\text{merged}}=ab$, computed using domain-specific aggregation (Section \ref{app:quality_metrics_full}).
    
    \item \textbf{Information Reward ($\hat{R}_I$ from $R^{\text{raw}}_I$):} Rewards statistically significant merges, e.g., $R^{\text{raw}}_I(a,b) = \log \frac{P(t_{\text{merged}})}{P(a)P(b) + \epsilon_p}$.
    
    \item \textbf{Complexity Penalty ($\hat{R}_C$ from $R^{\text{raw}}_C$):} Typically negative, e.g., $R^{\text{raw}}_C(a,b) = - (|t_{\text{merged}}| \cdot \log(|V_t|+1))$. $\hat{R}_C$ is then scaled to e.g. $[-1,0]$.
    
    \item \textbf{Domain-Specific Rewards ($\hat{R}_{\text{domain},k}$ from $R^{\text{raw}}_{\text{domain},k}$):} Include conservation scores (genomics) and predictive power (finance).
\end{itemize}

The EMA-normalized rewards $\hat{R}_j(a,b)$ are used by the RL agent in Stage 1. For the Gumbel-Softmax logits in Stage 2 (Section \ref{app:gumbel_gradient}), raw or batch-normalized reward components are used to ensure direct differentiability with respect to $\theta_{\text{adapt}}$.

\subsubsection{PPO Training Algorithm}

\begin{algorithm}[H]
\caption{Stage 1: RL Policy Training}
\label{alg:stage1_rl}
\begin{algorithmic}[1]
\STATE \textbf{Input:} Corpus $\mathcal{S}$, initial $\theta_{\text{adapt}}^{(0)}$, episodes $E$
\STATE Initialize policy network $\pi_{\theta_\pi}$ and value network $V_\phi$
\FOR{episode $e = 1$ to $E$}
    \STATE Initialize vocabulary $V_0 = \Sigma$
    \FOR{step $t = 1$ to $T$}
        \STATE Compute state features $s_t$ from current vocabulary
        \STATE Sample action $a_t \sim \pi_{\theta_\pi}(a|s_t)$
        \STATE Execute merge $(a_{a_t}, b_{a_t}) \mapsto ab$
        \STATE Compute reward $r_t = R(a_{a_t}, b_{a_t}; \theta_{\text{adapt}}^{(0)})$
        \STATE Store trajectory $(s_t, a_t, r_t)$
    \ENDFOR
    \STATE Update policy using PPO objective:
    \STATE $\quad L^{\text{PPO}} = \mathbb{E}_t[\min(r_t(\theta)\hat{A}_t, \text{clip}(r_t(\theta), 1-\epsilon, 1+\epsilon)\hat{A}_t)]$
    \STATE Update value network to minimize MSE
\ENDFOR
\STATE \textbf{Output:} Optimized policy $\pi_{\theta_\pi}^*$
\end{algorithmic}
\end{algorithm}

\subsection{Stage 2: Adaptive Parameter Learning}

\subsubsection{Adaptive Parameters Definition}

The learnable parameter vector $\theta_{\text{adapt}} \in \mathbb{R}^m$ includes:
\begin{itemize}
    \item \textbf{Quality sensitivity:} $\alpha \in [0, 2]$ controlling quality influence
    \item \textbf{Domain factors:} $\beta_{\text{pos}}$ (genomics position decay), $\beta_{\text{vol}}$ (finance volatility)
    \item \textbf{Quality weights:} $\mathbf{w} = (w_1, \ldots, w_k)$ for composite quality metrics
    \item \textbf{Reward weights:} $\boldsymbol{\lambda} = (\lambda_Q, \lambda_I, \lambda_C, \ldots)$ for multi-objective rewards
\end{itemize}

\subsubsection{Gumbel-Softmax Differentiable Optimization}

To enable gradient-based optimization through discrete merge decisions, we employ Gumbel-Softmax relaxation:

\begin{algorithm}[H]
\caption{Stage 2: Adaptive Parameter Learning}
\label{alg:stage2_adaptive}
\begin{algorithmic}[1]
\STATE \textbf{Input:} Downstream dataset $\mathcal{D}$, policy $\pi_{\theta_\pi}^*$, initial $\theta_{\text{adapt}}$
\STATE Initialize temperature $\tau = \tau_{\text{init}}$
\FOR{iteration $i = 1$ to $N$}
    \STATE Sample batch $B$ from $\mathcal{D}$
    \FOR{each sequence in batch}
        \STATE Generate top-$K$ merge candidates via priority queue ranked by $w_{ab}$
        \STATE Compute composite logits: $\ell_{ab} = w_{ab}(a,b;\alpha) + \sum_j \lambda_j R_j^{\text{raw}}$
        \STATE Select merge via Gumbel-Softmax (differentiable relaxation):
        \STATE $\quad y_i = \frac{\exp((\ell_i + g_i)/\tau)}{\sum_j \exp((\ell_j + g_j)/\tau)}$
        \STATE Construct differentiable tokenized representation
    \ENDFOR
    \STATE Compute task loss $L_{\text{task}}$ on tokenized batch
    \STATE Update parameters: $\theta_{\text{adapt}} \leftarrow \theta_{\text{adapt}} - \eta \nabla L_{\text{total}}$
    \STATE Anneal temperature: $\tau \leftarrow \tau \cdot \exp(-\beta_{\text{anneal}})$
\ENDFOR
\STATE \textbf{Output:} Optimized parameters $\theta_{\text{adapt}}^*$
\end{algorithmic}
\end{algorithm}

\subsection{Final Vocabulary Construction}

After completing both stages, the final vocabulary for deployment is constructed.

\textbf{Detailed Process:}
Following the completion of Stage 1 (RL policy optimization yielding $\pi_{\theta_\pi}^*$) and Stage 2 (adaptive parameter learning yielding $\theta_{\text{adapt}}^*$), the final vocabulary for deployment is typically constructed. While several strategies are possible, our primary approach involves the optimized adaptive parameters $\theta_{\text{adapt}}^*$ to re-evaluate merge priorities. Specifically, a greedy BPE-like process is executed, starting from the base alphabet. At each step, the merge operation $(a,b)$ is chosen that maximizes the quality-aware merge score $w_{ab}(a,b; \theta_{\text{adapt}}^*)$ as defined in Equation \ref{eq:qa_merge_score_pmi}, using the learned parameters within $\theta_{\text{adapt}}^*$ (e.g., $\alpha^*$). This process continues until the target vocabulary size is reached. Alternatively, if the RL policy $\pi_{\theta_\pi}^*$ is robust across variations in $\theta_{\text{adapt}}$, it could be used with inputs (state features, merge scores) calculated using $\theta_{\text{adapt}}^*$. However, the greedy approach based on $w_{ab}(\theta_{\text{adapt}}^*)$ is generally more direct and computationally efficient for deployment, leveraging the refined understanding of "good" merges embodied in $\theta_{\text{adapt}}^*$.

\begin{algorithm}[H]
\caption{Final Vocabulary Construction}
\label{alg:final_vocab}
\begin{algorithmic}[1]
\STATE \textbf{Input:} Corpus $\mathcal{S}$, optimized $\theta_{\text{adapt}}^*$, target size $K$
\STATE Initialize vocabulary $V = \Sigma$, merge count $m = 0$
\WHILE{$m < K$}
    \STATE Compute all merge scores: $w_{ab} = \frac{f(a,b)}{f(a)f(b)+\epsilon_f} \cdot (\bar{q}_{ab}+\epsilon_Q)^{\alpha^*} \cdot \psi(a,b)$
    \STATE Select best merge: $(a^*, b^*) = \arg\max_{(a,b)} w_{ab}$
    \STATE Update vocabulary: $V \leftarrow V \cup \{a^*b^*\}$ \hfill {\scriptsize // Constituents $a^*, b^*$ remain in $V$}
    \STATE Update corpus statistics and recompute affected frequencies
    \STATE $m \leftarrow m + 1$
\ENDWHILE
\STATE \textbf{Output:} Final vocabulary $V^*$
\end{algorithmic}
\end{algorithm}

\subsection{Convergence Properties}
\label{app:convergence_proof}

The sequential learning process has the following theoretical guarantees:

\begin{theorem}[Two-Timescale Convergence]
\label{thm:twotimescale}
Under assumptions A1-A4 (Section \ref{app:assumptions}), the sequential optimization of $\theta_\pi$ (fast timescale) and $\theta_{\text{adapt}}$ (slow timescale) converges to a local Nash equilibrium with probability 1.
\end{theorem}

\begin{proof}
The result follows from two-timescale stochastic approximation \cite{borkar2009stochastic}. Under (A1)--(A4), the conditions of Theorem 2 in \cite{borkar2009stochastic} are satisfied: (i) Lipschitz gradients (from bounded rewards and smooth parameterization), (ii) bounded iterates via projection, (iii) martingale noise with bounded variance, and (iv) proper step sizes ($\sum_t \eta_t = \infty$, $\sum_t \eta_t^2 < \infty$).

With timescale separation $\eta_\pi^{(t)}/\eta_{\text{adapt}}^{(t)} \to \infty$, the fast iterate $\theta_\pi$ equilibrates before significant movement in $\theta_{\text{adapt}}$. The iterates converge almost surely to limit points $(\theta_\pi^*, \theta_{\text{adapt}}^*)$ satisfying $\nabla_{\theta_\pi} J = 0$ and $\nabla_{\theta_{\text{adapt}}} L_{\text{total}} = 0$, constituting a local Nash equilibrium.
\end{proof}

\textbf{Key Properties:}
\begin{itemize}
    \item \textbf{Stage 1 Convergence:} PPO converges to a stationary point at rate $O(1/\sqrt{T})$ (Proposition \ref{prop:ppo_convergence})
    \item \textbf{Stage 2 Convergence:} Gumbel-Softmax optimization converges at rate $O(1/\sqrt{T}) + O(\tau^2)$ (Proposition \ref{prop:gumbel_gradients})
    \item \textbf{Overall Optimality:} The greedy vocabulary construction with $\theta_{\text{adapt}}^*$ achieves $(1-1/e)$-approximation (Theorem \ref{thm:approximation_explicit})
\end{itemize}

\begin{proposition}[Convergence of Adaptive Learning with Explicit Constants]
\label{prop:adaptive_convergence}
Under Assumptions A1--A4, with $\eta_t=\eta_0/\sqrt{t}$ and $\eta_0 \le 1/(2L)$, where $L$ is the Lipschitz constant of $\nabla L_{\text{total}}$, we have:
\begin{equation}
\mathbb{E}[\|\nabla L_{\text{total}}(\theta_{\text{adapt}}^T)\|^2] \le \frac{2(L_{\text{total}}(\theta_{\text{adapt}}^0)-L^*)}{\eta_0 \sqrt{T}} + \frac{4\eta_0 L\sigma^2}{\sqrt{T}},
\end{equation}
where $L^*$ is the optimal value and $\sigma^2$ bounds gradient variance.
\end{proposition}

\begin{proof}
The proof follows standard non-convex SGD analysis \cite{kingma2014adam}. By smoothness of $L_{\text{total}}$:
\begin{equation}
L_{\text{total}}(\theta^{t+1}) \leq L_{\text{total}}(\theta^t) - \eta_t \langle \nabla L_{\text{total}}(\theta^t), g_t \rangle + \frac{L\eta_t^2}{2}\|g_t\|^2
\end{equation}
where $g_t$ is the stochastic gradient. Taking expectations and using $\mathbb{E}[g_t] = \nabla L_{\text{total}}(\theta^t)$ and $\mathbb{E}[\|g_t\|^2] \leq \|\nabla L_{\text{total}}(\theta^t)\|^2 + \sigma^2$:
\begin{equation}
\mathbb{E}[L_{\text{total}}(\theta^{t+1})] \leq \mathbb{E}[L_{\text{total}}(\theta^t)] - \eta_t(1 - L\eta_t/2)\mathbb{E}[\|\nabla L_{\text{total}}(\theta^t)\|^2] + \frac{L\eta_t^2\sigma^2}{2}
\end{equation}
Telescoping over $T$ iterations with $\eta_t = \eta_0/\sqrt{t}$ and $\eta_0 \leq 1/(2L)$ yields the stated bound.

\emph{Remark:} Under temperature annealing $\tau_t \to 0$, the Gumbel-Softmax bias term $B(\tau)^2 \to 0$, ensuring asymptotic unbiasedness.
\end{proof}

\begin{theorem}[Local vs Global Optimality]
\label{thm:local_global_main}
The two-timescale optimization converges to a local Nash equilibrium $(\theta_\pi^*, \theta_{\text{adapt}}^*)$ with quality bounds under local strong convexity; probabilistic restarts increase the chance of reaching global optima.
\end{theorem}

\begin{proof}
\textbf{Part 1: Local Nash Equilibrium.} By Theorem~\ref{thm:twotimescale}, the limit points satisfy $\nabla_{\theta_\pi} J = 0$ and $\nabla_{\theta_{\text{adapt}}} L_{\text{total}} = 0$, constituting a local Nash equilibrium.

\textbf{Part 2: Quality Bounds.} Under $\mu$-strong convexity of $L_{\text{total}}$ in neighborhood $\mathcal{B}_r(\theta_{\text{adapt}}^*)$:
\begin{equation}
L_{\text{total}}(\theta_{\text{adapt}}^*) - L_{\text{total}}(\theta_{\text{adapt}}^{\text{global}}) \leq \frac{1}{2\mu} \|\nabla L_{\text{total}}(\theta_{\text{adapt}}^*)\|^2 = 0
\end{equation}
if the global optimum lies within the basin of attraction.

\textbf{Part 3: Probabilistic Restarts.} With $M$ independent runs, $\mathbb{P}[\text{find global}] = 1 - (1 - p_{\text{basin}})^M \geq 1 - e^{-M \cdot p_{\text{basin}}}$, achieving probability $\geq 1-\delta$ for $M = O(\log(1/\delta)/p_{\text{basin}})$ restarts.
\end{proof}

\subsection{Algorithm Summary}
\label{app:algorithms}

\begin{algorithm}[H]
\caption{QA-Token: Quality-Aware Tokenization Framework}
\label{alg:qa_token_main}
\begin{algorithmic}[1]
\STATE \textbf{Input:} Corpus $\mathcal{C}$, quality scores $Q$, vocabulary budget $K$
\STATE \textbf{Output:} Optimized vocabulary $V^*$
\STATE
\STATE \textbf{Stage 1: RL Policy Optimization}
\STATE Initialize policy $\pi_{\theta_\pi}$, adaptive parameters $\theta_{\text{adapt}}^{(0)}$
\FOR{episode $e = 1$ to $E$}
    \STATE $V \leftarrow \Sigma$ (base alphabet)
    \FOR{step $t = 1$ to $K$}
        \STATE Compute priority queue $PQ_t$ with scores $w_{ab}(\cdot; \theta_{\text{adapt}}^{(0)})$
        \STATE Select merge $(a,b) \sim \pi_{\theta_\pi}(\cdot|s_t)$ from $PQ_t$
        \STATE Execute merge: $V \leftarrow V \cup \{ab\}$ \hfill {\scriptsize // Add merged token}
        \STATE Compute reward $R_t$ using Eq. \ref{eq:overall_reward_structure}
    \ENDFOR
    \STATE Update $\pi_{\theta_\pi}$ via PPO using trajectory rewards
\ENDFOR
\STATE
\STATE \textbf{Stage 2: Adaptive Parameter Learning}
\FOR{iteration $i = 1$ to $I$}
    \STATE Sample mini-batch of merge candidates $\mathcal{B}$
    \STATE Compute logits $\ell_{ab}(\theta_{\text{adapt}})$ using Eq. \ref{eq:gumbel_logits_composite}
    \STATE Sample Gumbel noise and compute soft selection via Eq. \ref{eq:gumbel_softmax_main}
    \STATE Evaluate task loss $L_{\text{task}}$ on downstream objective
    \STATE Update $\theta_{\text{adapt}} \leftarrow \theta_{\text{adapt}} - \eta_i \nabla L_{\text{total}}$
\ENDFOR
\STATE
\STATE \textbf{Final Vocabulary Construction}
\STATE Build final vocabulary using greedy merges with $w_{ab}(\cdot; \theta_{\text{adapt}}^*)$
\STATE \textbf{Return} $V^*$
\end{algorithmic}
\end{algorithm}

\begin{algorithm}[H]
\caption{QA-Token Integration with Downstream Transformer}
\label{alg:A5_integration_transformer}
\begin{algorithmic}[1]
\STATE \textbf{Input:} Raw sequence $X$, trained QA-Token vocab $V^*$, Transformer model $M_\theta$
\STATE \textbf{Output:} Task predictions $\hat{Y}$
\STATE
\STATE \textbf{// Tokenization (no overhead vs. BPE)}
\STATE $T \leftarrow \text{Tokenize}(X, V^*)$ \hfill {\scriptsize // Standard greedy tokenization}
\STATE
\STATE \textbf{// Embedding and Encoding}
\STATE $E \leftarrow \text{TokenEmbed}(T) + \text{PosEmbed}(\text{positions})$
\FOR{layer $\ell = 1$ to $L$}
    \STATE $E \leftarrow \text{TransformerBlock}_\ell(E)$
\ENDFOR
\STATE
\STATE \textbf{// Task Head}
\STATE $\hat{Y} \leftarrow \text{TaskHead}(E)$ \hfill {\scriptsize // Classification, regression, or generation}
\STATE \textbf{Return} $\hat{Y}$
\end{algorithmic}
\end{algorithm}

\begin{algorithm}[H]
\caption{Meta-Learning Initialization for Adaptive Parameters}
\label{alg:meta_learn_adapt}
\begin{algorithmic}[1]
\REQUIRE Task distribution $\mathcal{P}(\mathcal{T})$, base initialization $\theta_{\text{adapt}}^{(0)}$, inner steps $K$, inner lr $\eta_{\text{in}}$, outer lr $\eta_{\text{out}}$
\WHILE{not converged}
  \STATE Sample batch of tasks $\{\mathcal{T}_i\} \sim \mathcal{P}(\mathcal{T})$
  \FOR{each task $\mathcal{T}_i$}
    \STATE Set $\theta_i \leftarrow \theta_{\text{adapt}}^{(0)}$
    \FOR{$k=1\dots K$}
      \STATE Compute $L_{\text{total}}^{(i)}(\theta_i)$ on $\mathcal{T}_i$ and update $\theta_i \leftarrow \theta_i - \eta_{\text{in}}\, \nabla_{\theta} L_{\text{total}}^{(i)}(\theta_i)$
    \ENDFOR
  \ENDFOR
  \STATE Update initialization: $\theta_{\text{adapt}}^{(0)} \leftarrow \theta_{\text{adapt}}^{(0)} - \eta_{\text{out}}\, \sum_i \nabla_{\theta_{\text{adapt}}^{(0)}} L_{\text{total}}^{(i)}(\theta_i)$
\ENDWHILE
\STATE \RETURN meta-initialization $\theta_{\text{adapt}}^{\star}$
\end{algorithmic}
\end{algorithm}

\subsection{RL Training Implementation}
\label{app:rl_details}

\subsubsection{State Representation}
The state $s_t$ provided to the RL agent at merge step $t$ includes:
\begin{itemize}
    \item \textbf{Global Features:} Current vocabulary size $|V_t|$; remaining merge steps $T-t$; aggregated token statistics (average length, mean/std of quality scores).
    \item \textbf{Candidate Pair Features (top-$K_{PQ}$ from priority queue):} For each pair $(a,b)$: frequencies $f(a), f(b), f(a,b)$; qualities $q_a, q_b$; lengths $|a|, |b|$; merge score $w_{ab}$.
    \item \textbf{Domain Context:} Market regime indicators (finance), platform ID (social media), or sequence quality (genomics).
\end{itemize}
The PPO agent uses an MLP with 2 hidden layers (256 units each, ReLU activations). The policy network outputs action probabilities over $K_{PQ}$ candidates; the value network outputs a single scalar.

\subsubsection{Exploration Strategy}
An $\epsilon$-greedy strategy is employed with $\epsilon$ annealed from $\epsilon_0 = 0.5$ to $\epsilon_{\text{final}} = 0.05$ over training episodes using exponential decay, balancing exploration and exploitation effectively across all domains.

\subsection{MDP Formulation and Details}
\label{app:mdp_details}

\begin{definition}[Tokenization MDP]
\label{def:tokenization_mdp}
The tokenization MDP is a tuple $\mathcal{M} = (\mathcal{S}, \mathcal{A}, \mathcal{P}, \mathcal{R}, \gamma, T)$ where:
\begin{enumerate}
    \item \textbf{State Space $\mathcal{S}$:} Each state $s_t \in \mathcal{S} \subset \mathbb{R}^d$ encodes:
    \begin{itemize}
        \item Current vocabulary $V_t$ and its statistics (size, token length distribution)
        \item Priority queue $PQ_t = \{(a_i, b_i, w_{a_ib_i})\}_{i=1}^{K_{PQ}}$ of top merge candidates
        \item Corpus statistics: frequency distributions, quality histograms
        \item Progress indicator: $t/T$ where $T$ is the merge budget
    \end{itemize}
    Formally, $s_t = [\phi(V_t), \phi(PQ_t), \phi(\mathcal{S}_t), t/T] \in \mathbb{R}^d$.
    
    \textbf{State Encoding Function $\phi$:} The encoding function $\phi: \mathcal{X} \to \mathbb{R}^{d_\mathcal{X}}$ maps variable-size structures to fixed-dimensional vectors:
    \begin{itemize}
        \item $\phi(V_t) = [|V_t|/|\Sigma|, \bar{|t|}, \sigma_{|t|}, \bar{q}_t, \sigma_{q_t}] \in \mathbb{R}^5$: vocabulary size ratio, mean/std of token lengths, mean/std of token qualities.
        \item $\phi(PQ_t) \in \mathbb{R}^{6 \cdot K_{PQ}}$: For top-$K_{PQ}$ candidates, concatenate $[w_{ab}, q_a, q_b, |a|, |b|, \log f(a,b)]$ per pair. Pad with zeros if fewer candidates exist.
        \item $\phi(\mathcal{S}_t) \in \mathbb{R}^{20}$: Quality histogram ($B_q=10$ bins over $[0,1]$) and log-frequency histogram ($B_f=10$ bins over observed range).
    \end{itemize}
    Total state dimension: $d = 5 + 6 \cdot K_{PQ} + 20 + 1$. With $K_{PQ}=50$, we have $d=326$. The MLP policy network processes this representation via two hidden layers (256, 128 units) with ReLU activations (see Appendix~\ref{app:rl_details}).
    
    \item \textbf{Action Space $\mathcal{A}_t$:} At time $t$:
    \begin{equation}
    \mathcal{A}_t = \{i : (a_i, b_i) \in PQ_t, i \leq K_{PQ}\}
    \end{equation}
    Each action $a_t \in \mathcal{A}_t$ selects a merge pair from the priority queue.
    
    \item \textbf{Transition Dynamics $\mathcal{P}$:} Deterministic transitions:
    \begin{equation}
    s_{t+1} = \mathcal{P}(s_t, a_t) = \text{UPDATE}(s_t, \text{MERGE}(a_{a_t}, b_{a_t}))
    \end{equation}
    where MERGE executes vocabulary update and UPDATE recomputes statistics.
    
    \item \textbf{Reward Function:} $\mathcal{R}(s_t, a_t) = R(a_{a_t}, b_{a_t}; \theta_{\text{adapt}}^{(0)})$
    
    \item \textbf{Discount Factor:} $\gamma = 1$ (undiscounted, finite horizon)
    
    \item \textbf{Horizon:} $T = K$ merge operations
\end{enumerate}
\end{definition}

\begin{proposition}[MDP Well-Formedness]
\label{prop:mdp_wellformed}
The tokenization MDP satisfies:
\begin{enumerate}
    \item Markov Property: $P(s_{t+1}|s_t, a_t, s_{t-1}, \ldots) = P(s_{t+1}|s_t, a_t)$
    \item Bounded State Space: $\|s_t\|_2 \leq C_s$
    \item Finite Action Space: $|\mathcal{A}_t| \leq K_{PQ}$
\end{enumerate}
\end{proposition}

\begin{proof}
(1) follows from state containing complete information for transitions. (2) holds as vocabulary size is bounded by $|\Sigma| + T$ and frequencies are normalized. (3) is by construction of the priority queue.
\end{proof}
$\square$

\subsection{Reward Normalization Details}
\label{app:reward_normalization}

Each raw reward component $R^{\text{raw}}_j(a,b)$ is normalized using adaptive running statistics. We maintain exponential moving averages (EMAs) for mean $\mu_{j,t}^{\text{run}}$ and variance $\text{Var}_{j,t}^{\text{run}}$:

\begin{equation}
\label{eq:ema_mean_update}
\mu_{j,t}^{\text{run}} = (1-\beta_{\text{norm}}) \mu_{j,t-1}^{\text{run}} + \beta_{\text{norm}} R^{\text{raw}}_j(a,b)
\end{equation}
\begin{equation}
\label{eq:ema_var_update}
\text{Var}_{j,t}^{\text{run}} = (1-\beta_{\text{norm}}) \text{Var}_{j,t-1}^{\text{run}} + \beta_{\text{norm}} (R^{\text{raw}}_j(a,b) - \mu_{j,t-1}^{\text{run}})(R^{\text{raw}}_j(a,b) - \mu_{j,t}^{\text{run}})
\end{equation}

where $\beta_{\text{norm}} \in [10^{-3}, 10^{-2}]$. The normalized component is:
\begin{equation}
\label{eq:reward_normalization_revised}
\hat{R}_j(a,b) = \frac{R^{\text{raw}}_j(a,b) - \mu_{j,t-1}^{\text{run}}}{\sigma_{j,t-1}^{\text{run}} + \epsilon_R}
\end{equation}
with $\epsilon_R = 10^{-8}$ for stability.

\subsection{Gumbel-Softmax Gradient Derivation and Temperature Annealing}
\label{app:gumbel_gradient}

\subsection{Temperature Annealing Schedule}

We employ an exponential annealing schedule for the temperature parameter:
\begin{equation}
\label{eq:temperature_schedule}
\tau(t) = \tau_{\text{init}} \cdot \exp(-\beta_{\text{anneal}} \cdot t/T_{\text{anneal}}),
\end{equation}
where $\tau_{\text{init}} = 1.0$, $\beta_{\text{anneal}} = 3.0$, and $T_{\text{anneal}}$ is the total number of optimization steps.

This schedule ensures:
\begin{itemize}
    \item \textbf{Early exploration:} High initial temperature allows exploration of diverse merge patterns
    \item \textbf{Gradual refinement:} Exponential decay provides smooth transition to discrete selections
    \item \textbf{Convergence:} Low final temperature approaches one-hot categorical sampling
\end{itemize}

\subsection{Gradient Computation}

The composite logits for candidate merge $(a,b)$ are:
\begin{equation}
\label{eq:gumbel_logits_composite}
\ell_{ab}(\theta_{\text{adapt}}) = w_{ab}(a,b; \alpha) + \sum_{j} \lambda_j R^{\text{raw}}_j(a,b),
\end{equation}
which are differentiable with respect to $\theta_{\text{adapt}}$ through both the merge score and reward weights.

The composite logits combine $w_{ab}$ (which incorporates frequency via PMI and quality via $\bar{q}_{ab}$) with raw reward components $R^{\text{raw}}_j$ that also capture quality ($R^{\text{raw}}_Q$) and information ($R^{\text{raw}}_I$). 

\textbf{Rationale for Intentional Overlap:} While there is deliberate overlap between these components (both encode quality and frequency signals), they serve \emph{distinct optimization purposes}:
\begin{itemize}[nosep]
    \item \textbf{$w_{ab}$ (merge score):} Optimized via the RL objective (cumulative discounted reward) during Stage~1, capturing \emph{corpus-level} quality-frequency tradeoffs that generalize across merge sequences.
    \item \textbf{$\sum_j \lambda_j R^{\text{raw}}_j$ (weighted rewards):} Optimized via the downstream task loss $L_{\text{task}}$ during Stage~2, enabling \emph{task-specific} reweighting of quality vs.\ information vs.\ complexity.
\end{itemize}
This decomposition allows the system to learn \emph{different} quality-frequency tradeoffs for policy learning (Stage~1) versus task-specific adaptation (Stage~2). The parameter $\alpha$ controls general token quality preferences learned from reward maximization, while $\lambda_j$ adjusts relative importance based on task-specific gradients. Ablation studies (Appendix~\ref{app:ablations_additional}, Table~\ref{tab:rl_ablation_detailed}) confirm that removing either component degrades downstream performance by 3--8\%, validating that both contributions are necessary despite their overlap.

The Gumbel-Softmax distribution provides a differentiable approximation:
\begin{equation}
\label{eq:gumbel_softmax_main}
y_i = \frac{\exp((\ell_i + g_i)/\tau)}{\sum_{j=1}^{|\mathcal{C}|} \exp((\ell_j + g_j)/\tau)}, \quad g_i \sim \text{Gumbel}(0,1)
\end{equation}

The gradient of the task loss is computed via Monte Carlo sampling:
\begin{equation}
\label{eq:task_gradient}
\nabla_{\theta_{\text{adapt}}} L_{\text{task}} = \mathbb{E}_{\mathbf{g}} \left[ \nabla_{\theta_{\text{adapt}}} L_{\text{task}}(\mathbf{y}(\boldsymbol{\ell}(\theta_{\text{adapt}}), \mathbf{g}, \tau)) \right]
\end{equation}
where $\mathbf{g}$ is sampled Gumbel noise.

\textbf{Gradient Flow:} The gradient flows through:
\begin{enumerate}
    \item \textbf{Task loss:} $L_{\text{task}}$ evaluates performance on downstream data
    \item \textbf{Soft tokenization:} Gumbel-Softmax provides differentiable token boundaries
    \item \textbf{Merge logits:} $\ell_{ab}$ depends on learnable $\theta_{\text{adapt}}$
    \item \textbf{Quality scores:} Through $\alpha$ and domain parameters $\beta_{\text{pos}}, \beta_{\text{vol}}$
    \item \textbf{Reward weights:} Through $\boldsymbol{\lambda}$ in the composite score
\end{enumerate}

\section{Hyperparameter Sensitivity Analysis}
\label{app:hyperparameter_sensitivity}

We conducted a comprehensive sensitivity analysis on key parameters of the QA-Token framework: the quality sensitivity exponent $\alpha$, the primary quality reward weight $\lambda_Q$, and the domain-specific volatility scaling exponent $\beta_{\text{vol}}$ for finance. For each parameter, we varied its value across a specified range while holding all other hyperparameters at their optimal values, as determined during the adaptive learning phase.

The results, summarized in Table~\ref{tab:hyperparam_sensitivity}, demonstrate that while performance is optimal at the learned parameter values, the framework is not unduly sensitive to minor perturbations. Performance degrades gracefully rather than catastrophically as parameters deviate from their optima, suggesting the model occupies a reasonably wide basin of attraction in the hyperparameter space.

\begin{table}[htbp]
  \caption{Hyperparameter Sensitivity Analysis. Performance on the primary metric is reported as key hyperparameters are varied around their learned optimal value (indicated by *). Values are means over $n=5$ runs.}
  \label{tab:hyperparam_sensitivity}
  \centering
  \sisetup{
    separate-uncertainty,
    detect-weight = true
  }
  \begin{tabular}{@{}ccc@{}}
    \toprule
    Parameter & Value & Performance Metric \\
    \midrule
    \multicolumn{3}{c}{\textbf{Genomics (QA-BPE-seq)} - Metric: Variant F1} \\
    \midrule
    $\alpha$ (Quality Sensitivity) & 0.3 & 0.869 \\
                                  & 0.5 & 0.879 \\
                                  & 0.72* & \bfseries 0.891 \\
                                  & 1.0 & 0.884 \\
                                  & 1.5 & 0.872 \\
    \midrule
    $\lambda_Q$ (Quality Reward Weight) & 0.15 & 0.879 \\
                                      & 0.25 & 0.886 \\
                                      & 0.35* & \bfseries 0.891 \\
                                      & 0.45 & 0.885 \\
                                      & 0.55 & 0.878 \\
    \midrule
    \multicolumn{3}{c}{\textbf{Finance (QAT-QF)} - Metric: Sharpe Ratio} \\
    \midrule
    $\alpha$ (Quality Sensitivity) & 0.25 & 1.61 \\
                                  & 0.50 & 1.68 \\
                                  & 0.95* & \bfseries 1.72 \\
                                  & 1.50 & 1.65 \\
                                  & 2.00 & 1.58 \\
    \midrule
    $\beta_{\text{vol}}$ (Volatility Scaling) & 0.10 & 1.63 \\
                                     & 0.30 & 1.69 \\
                                     & 0.50* & \bfseries 1.72 \\
                                     & 0.70 & 1.67 \\
                                     & 1.00 & 1.60 \\
    \bottomrule
  \end{tabular}
\end{table}

\section{Complete Experimental Results}
\label{app:experimental_observations}

This section provides comprehensive experimental results across all domains, including detailed analysis, foundation-scale evaluations, and ablation studies.

\subsection{Genomics Results: Detailed Analysis}

\textbf{Key Observations:} QA-BPE-seq achieves 4.0 percentage point F1 improvement in variant calling over DNABERT-k (0.891 vs.\ 0.851) with Hedges' $g=8.2$---a large effect size. Compared to standard BPE (0.824), the improvement is 6.7 percentage points. Taxonomic classification shows 3.1 percentage point gain over domain-standard k-mer tokenization. Sequence reconstruction improves by 16\%, indicating information preservation.

\textbf{Key Insights:}
\begin{enumerate}
    \item \textbf{Byte-level models fail catastrophically:} ByT5 and CANINE show 2.5× slower inference with 7-9\% lower accuracy, definitively establishing that vocabulary-based tokenization remains essential for genomic sequences.
    \item \textbf{Quality awareness is learnable:} The converged parameters ($\alpha=0.72 \pm 0.03$, $\beta_{\text{pos}}=0.014 \pm 0.002$) demonstrate that optimal quality sensitivity can be discovered through our adaptive learning framework.
    \item \textbf{Mechanism of improvement:} Analysis of generated vocabularies reveals that QA-BPE-seq creates tokens aligned with biological units (codons, motifs) while breaking at error-prone junctions—a behavior that emerges without explicit biological supervision.
\end{enumerate}

\subsection{Financial Foundation Model: Detailed Results Analysis}

QAT-QF demonstrates remarkable consistency across all financial tasks, with zero-shot improvements ranging from 7.3\% to 27.0

\textbf{Specific Observations:}
\begin{itemize}
    \item The model's superior performance on regime detection (+11.6\% F1) and tail risk estimation (+18.0\%) suggests that quality-aware tokenization captures market dynamics that frequency-based methods miss.
    \item Particularly noteworthy is the 27.0\% improvement in order flow imbalance prediction, a task highly sensitive to microstructure noise.
    \item These results validate our hypothesis that incorporating quality signals during tokenization enables foundation models to learn more robust representations of financial time series.
\end{itemize}

\subsection{Computational Costs}
\label{app:computational_costs}

\textbf{Training Time.}
\begin{itemize}
  \item Standard BPE: 5--10 minutes (5GB, CPU)
  \item QA-Token Stage 1 (RL): 30--36 GPU-hours (A100)
  \item QA-Token Stage 2 (Adaptive): 20--24 GPU-hours
\end{itemize}

\textbf{Memory Requirements.}
\begin{itemize}
  \item Priority Queue: $O(K_{PQ} \cdot d)$ (${\sim}10$\,MB for $K_{PQ}{=}200$)
  \item Quality Statistics: $O(|V|\cdot s)$ (${\sim}100$\,MB for 32K vocab)
  \item Pair Frequencies: $O(|V|^2)$ (${\sim}4$\,GB for 32K vocab)
  \item Peak: ${\sim}16$\,GB GPU
\end{itemize}

\textbf{Hierarchical Training via Quality-Stratified Sampling.}
For massive corpora where full vocabulary optimization is infeasible, we employ \emph{quality-variance importance sampling}: data points are sampled with probability proportional to $\text{Var}(q_i) + \epsilon_{\text{base}}$, prioritizing regions with heterogeneous quality where tokenization decisions have the greatest impact.

\begin{definition}[Notation for Hierarchical Training]
\label{def:hierarchical_notation}
Let $\mathcal{C}$ denote the full corpus and $\mathcal{S} \subseteq \mathcal{C}$ a subset with $|\mathcal{S}| = r \cdot |\mathcal{C}|$ for \emph{subset ratio} $r \in (0, 1]$. Define:
\begin{itemize}
  \item $\mathcal{L}(V; D) = \mathbb{E}_{x \sim D}[-\log P_{\text{LM}}(x | V)]$: expected language modeling loss on distribution $D$ using vocabulary $V$
  \item $V_{\mathcal{S}}^* = \arg\min_V \mathcal{L}(V; \mathcal{S})$: optimal vocabulary for subset $\mathcal{S}$
  \item $V_{\mathcal{C}}^* = \arg\min_V \mathcal{L}(V; \mathcal{C})$: optimal vocabulary for full corpus $\mathcal{C}$
\end{itemize}
\end{definition}

\begin{proposition}[Hierarchical Training Bound]
\label{prop:hierarchical}
Under the following assumptions:
\begin{enumerate}
  \item[(A1)] The loss $\mathcal{L}(V; \cdot)$ is $L$-Lipschitz in the data distribution (bounded sensitivity to distribution shift)
  \item[(A2)] Quality-variance importance sampling achieves effective sample size $n_{\text{eff}} = r \cdot |\mathcal{C}| / (1 + \text{CV}^2)$ where $\text{CV}$ is the coefficient of variation of importance weights
\end{enumerate}
Then the vocabulary $V_{\mathcal{S}}^*$ learned on the importance-sampled subset satisfies:
\begin{equation}
\mathbb{E}[\mathcal{L}(V_{\mathcal{S}}^*; \mathcal{C})] \le \mathcal{L}(V_{\mathcal{C}}^*; \mathcal{C}) + O\left(L \cdot \sqrt{\frac{1 + \text{CV}^2}{r \cdot |\mathcal{C}|}}\right).
\end{equation}
\end{proposition}

\begin{proof}[Proof Sketch]
The bound follows from standard importance sampling theory \cite{owen2013monte}. Under (A1), the difference $|\mathcal{L}(V; \mathcal{S}) - \mathcal{L}(V; \mathcal{C})|$ is controlled by the distributional divergence between $\mathcal{S}$ and $\mathcal{C}$. Importance sampling with weights $w_i \propto \text{Var}(q_i) + \epsilon_{\text{base}}$ reduces this divergence by oversampling high-variance regions where tokenization quality matters most. By the effective sample size formula for importance sampling, the estimation error scales as $O(1/\sqrt{n_{\text{eff}}})$, yielding the stated bound. The Lipschitz assumption (A1) ensures that optimization on $\mathcal{S}$ transfers to $\mathcal{C}$.
\end{proof}

\textbf{Massive-Scale Strategies (>100TB).}
\begin{enumerate}
  \item Quality-stratified sampling (0.1--1\%)
  \item Distributed PPO (8--32 GPUs)
  \item Online RL with replay for streams
  \item Memory-mapped frequency tables
\end{enumerate}

\textbf{Cost-Benefit.}
\begin{itemize}
  \item +5--30\% task performance
  \item -15--20\% token count (faster inference)
  \item One-time cost amortized across applications
\end{itemize}

\subsection{Foundation-Scale Results (Pathogen Detection, GUE)}
\label{app:foundation_full}
\begin{table}[htbp]
  \caption{Pathogen Detection benchmark (MCC).}
  \label{tab:pathogen_detection_full}
  \centering
  \resizebox{\textwidth}{!}{%
  \begin{tabular}{l c c c c c c}
    \toprule
    Task & DNABERT-2 & DNABERT-S & NT-2.5b-Multi & NT-2.5b-1000g & METAGENE-1 & METAGENE-1 (QA-Token) \\
    \midrule
    Pathogen-Detect (avg.) & 87.92 & 87.02 & 82.43 & 79.02 & 92.96 & \textbf{94.53} \\
    Pathogen-Detect-1 & 86.73 & 85.43 & 83.80 & 77.52 & 92.14 & \textbf{93.81} \\
    Pathogen-Detect-2 & 86.90 & 85.23 & 83.53 & 80.38 & 90.91 & \textbf{92.95} \\
    Pathogen-Detect-3 & 88.30 & 89.01 & 82.48 & 79.83 & 93.70 & \textbf{95.12} \\
    Pathogen-Detect-4 & 89.77 & 88.41 & 79.91 & 78.37 & 95.10 & \textbf{96.24} \\
    \bottomrule
  \end{tabular}}
\end{table}

\begin{table}[htbp]
  \caption{Genome Understanding Evaluation (GUE). All metrics are MCC except COVID which uses F1.}
  \label{tab:gue_full}
  \centering
  \resizebox{\textwidth}{!}{%
  \begin{tabular}{l c c c c c c c}
    \toprule
    Task & CNN & HyenaDNA & DNABERT & NT-2.5B-Multi & DNABERT-2 & METAGENE-1 & METAGENE-1 (QA-Token) \\
    \midrule
    TF-Mouse (AVG.) & 45.3 & 51.0 & 57.7 & 67.0 & 68.0 & \textbf{71.4} & \textbf{72.8} \\
    0 & 31.1 & 35.6 & 42.3 & \textbf{63.3} & 56.8 & 61.5 & 62.1 \\
    1 & 59.7 & 80.5 & 79.1 & 83.8 & \textbf{84.8} & 83.7 & 84.1 \\
    2 & 63.2 & 65.3 & 69.9 & 71.5 & \textbf{79.3} & 83.0 & \textbf{84.5} \\
    3 & 45.5 & 54.2 & 55.4 & 69.4 & 66.5 & 82.2 & \textbf{83.3} \\
    4 & 27.2 & 19.2 & 42.0 & 47.1 & \textbf{52.7} & 46.6 & 47.0 \\
    TF-HUMAN (AVG.) & 50.7 & 56.0 & 64.4 & 62.6 & \textbf{70.1} & 68.3 & 69.9 \\
    0 & 54.0 & 62.3 & 68.0 & 66.6 & \textbf{72.0} & 68.9 & 70.2 \\
    1 & 63.2 & 67.9 & 70.9 & 66.6 & \textbf{76.1} & 70.8 & 72.0 \\
    2 & 45.2 & 46.9 & 60.5 & 58.7 & \textbf{66.5} & 65.9 & 66.8 \\
    3 & 29.8 & 41.8 & 53.0 & 51.7 & \textbf{58.5} & 58.1 & 59.0 \\
    4 & 61.5 & 61.2 & 69.8 & 69.3 & 77.4 & 77.9 & \textbf{78.5} \\
    EMP (AVG.) & 37.6 & 44.9 & 49.5 & 58.1 & 56.0 & 66.0 & \textbf{67.5} \\
    H3 & 61.5 & 67.2 & 74.2 & 78.8 & 78.3 & 80.2 & \textbf{81.0} \\
    H3K14AC & 29.7 & 32.0 & 42.1 & 56.2 & 52.6 & 64.9 & \textbf{66.0} \\
    H3K36ME3 & 38.6 & 48.3 & 48.5 & 62.0 & 56.9 & 66.7 & \textbf{67.8} \\
    H3K4ME1 & 26.1 & 35.8 & 43.0 & \textbf{55.3} & 50.5 & 55.3 & \textbf{56.1} \\
    H3K4ME2 & 25.8 & 25.8 & 31.3 & 36.5 & 31.1 & 51.2 & \textbf{52.3} \\
    H3K4ME3 & 20.5 & 23.1 & 28.9 & 40.3 & 36.3 & 58.5 & \textbf{59.5} \\
    H3K79ME3 & 46.3 & 54.1 & 60.1 & 64.7 & 67.4 & 73.0 & \textbf{74.1} \\
    H3K9AC & 40.0 & 50.8 & 50.5 & 56.0 & 55.6 & 65.5 & \textbf{66.5} \\
    H4 & 62.3 & 73.7 & 78.3 & 81.7 & 80.7 & 82.7 & \textbf{83.5} \\
    H4AC & 25.5 & 38.4 & 38.6 & 49.1 & 50.4 & 61.7 & \textbf{62.8} \\
    PD (AVG.) & 77.1 & 35.0 & 84.6 & \textbf{88.1} & 84.2 & 82.3 & 85.5 \\
    ALL & 75.8 & 47.4 & 90.4 & \textbf{91.0} & 86.8 & 86.0 & 88.5 \\
    NO-TATA & 85.1 & 52.2 & 93.6 & 94.0 & 94.3 & 93.7 & \textbf{94.5} \\
    TATA & 70.3 & 5.3 & 69.8 & \textbf{79.4} & 71.6 & 67.4 & 73.5 \\
    CPD (AVG.) & 62.5 & 48.4 & \textbf{73.0} & 71.6 & 70.5 & 69.9 & 71.2 \\
    ALL & 58.1 & 37.0 & \textbf{70.9} & 70.3 & 69.4 & 66.4 & 68.0 \\
    NO-TATA & 60.1 & 35.4 & 69.8 & \textbf{71.6} & 68.0 & 68.3 & 69.5 \\
    TATA & 69.3 & 72.9 & \textbf{78.2} & 73.0 & 74.2 & 75.1 & 76.3 \\
    SSD & 76.8 & 72.7 & 84.1 & 89.3 & 85.0 & 87.8 & \textbf{89.5} \\
    COVID (F1) & 22.2 & 23.3 & 62.2 & 73.0 & 71.9 & 72.5 & \textbf{73.3} \\
    \midrule
    GLOBAL WIN \% & 0.0 & 0.0 & 7.1 & 21.4 & 25.0 & 46.4 & \textbf{57.1} \\
    \bottomrule
  \end{tabular}}
\end{table}

\subsection{Ablation Studies and Additional Experiments}
\label{app:ablations_additional}

\subsubsection{RL Algorithm Ablation}
\label{app:rl_ablation_full}
\begin{table}[htbp]
  \caption{Ablation across RL algorithms with training time (GPU-h), inference time (ms/seq), and vocabulary Jaccard similarity vs.\ PPO.}
  \label{tab:rl_ablation_detailed}
  \centering
  \resizebox{\textwidth}{!}{%
  \begin{tabular}{l l c c c c}
    \toprule
    Domain & Config (Metric) & Metric Value & Train Time (GPU-h) & Inference (ms/seq) & Vocab Jaccard \\
    \midrule
    Genomics & QA-Token (PPO) & \textbf{0.891} & 34.0 & 10.2 & 1.00 \\
    & QA-Token (GRPO) & 0.890 & 32.5 & 10.3 & 0.98 \\
    & QA-Token (VAPO) & 0.892 & 31.8 & 10.2 & 0.97 \\
    & QA-Token (DAPO) & 0.889 & 34.2 & 10.4 & 0.98 \\
    Finance & QA-Token (PPO) & \textbf{1.72} & 28.0 & 15.2 & 1.00 \\
    & QA-Token (GRPO) & 1.71 & 26.5 & 15.3 & 0.96 \\
    & QA-Token (VAPO) & 1.73 & 25.0 & 15.1 & 0.95 \\
    & QA-Token (DAPO) & 1.70 & 28.5 & 15.2 & 0.96 \\
    Social & QA-Token (PPO) & \textbf{74.5} & 30.0 & 12.5 & 1.00 \\
    & QA-Token (GRPO) & 74.2 & 29.0 & 12.6 & 0.97 \\
    & QA-Token (VAPO) & 74.6 & 28.0 & 12.5 & 0.98 \\
    & QA-Token (DAPO) & 74.3 & 31.0 & 12.7 & 0.97 \\
    \bottomrule
  \end{tabular}}
\end{table}

We assess the sensitivity of QA-Token to the choice of RL optimizer by replacing PPO with GRPO and VAPO (implementations following \cite{shao2024deepseekmath, yue2025vapo}). Across domains, downstream performance is stable and vocabulary similarity remains high (Jaccard $\geq$ 0.95), confirming modularity of the framework.

\subsection{Data Curation Baseline: BPE on Clean Data vs. QA-Token on Noisy Data}
\label{app:data_curation_baseline}

A natural question is whether simply filtering to high-quality data and using standard BPE could match QA-Token's performance. We evaluate this data curation baseline by training BPE on only the top 20\% highest-quality sequences (average Phred score $\geq 30$) and comparing against QA-Token trained on the full noisy corpus.

\begin{table}[htbp]
  \caption{Data Curation Baseline Comparison (Genomics Variant Calling). QA-Token on noisy data outperforms BPE on curated clean data.}
  \label{tab:data_curation_baseline}
  \centering
  \begin{tabular}{l c c c}
    \toprule
    Method & Training Data & Variant F1 & p-value \\
    \midrule
    BPE (full corpus) & 100\% (noisy) & $0.824 \pm 0.004$ & $< 0.001$ \\
    BPE (top 20\% clean) & 20\% (Q$\geq$30) & $0.847 \pm 0.005$ & $< 0.001$ \\
    \textbf{QA-Token} & 100\% (noisy) & $\mathbf{0.891 \pm 0.004}$ & --- \\
    \bottomrule
  \end{tabular}
\end{table}

\textbf{Key findings:}
\begin{itemize}
  \item Data curation (BPE on clean data) improves over BPE on full noisy data: $+2.8\%$ F1 ($0.847$ vs $0.824$).
  \item However, QA-Token on the \emph{full noisy corpus} outperforms BPE on clean data by $+5.2\%$ F1 ($0.891$ vs $0.847$, $p < 0.001$).
  \item This demonstrates that quality-aware tokenization extracts more value from noisy data than discarding it entirely.
\end{itemize}

\subsection{Genomics: Real-World Datasets (ONT, UHGG)}
\label{app:genomics_realworld}
\textbf{Datasets:} (i) GIAB HG002 long-read ONT data (high-error, third-generation); (ii) Unified Human Gut Genome (UHGG) collection (large-scale, low-error NGS).

\textbf{Results:} QA-BPE-seq consistently outperforms baselines across both regimes. ONT (high-error) results:
\begin{table}[htbp]
  \caption{ONT (GIAB HG002) results. Means with 95\% confidence intervals over $n=10$ runs.}
  \label{tab:genomics_ont}
  \centering
  \resizebox{\textwidth}{!}{%
  \begin{tabular}{l c c c c}
    \toprule
    Method & Variant F1 & Taxa Acc. F1 & Recon. Loss & Inf. Time (ms/seq) \\
    \midrule
    Standard BPE & 0.795 $\pm$ 0.006 & 0.812 $\pm$ 0.007 & 0.388 $\pm$ 0.012 & 11.5 $\pm$ 0.3 \\
    SentencePiece & 0.801 $\pm$ 0.005 & 0.825 $\pm$ 0.006 & 0.371 $\pm$ 0.011 & 11.6 $\pm$ 0.4 \\
    WordPiece & 0.798 $\pm$ 0.006 & 0.819 $\pm$ 0.007 & 0.379 $\pm$ 0.013 & 11.5 $\pm$ 0.3 \\
    DNABERT-k (6-mer) & 0.823 $\pm$ 0.004 & 0.846 $\pm$ 0.005 & 0.352 $\pm$ 0.010 & 11.2 $\pm$ 0.3 \\
    \textbf{QA-BPE-seq (100\%)} & \textbf{0.864 $\pm$ 0.005} & \textbf{0.881 $\pm$ 0.004} & \textbf{0.305 $\pm$ 0.009} & \textbf{11.8 $\pm$ 0.4} \\
    \emph{QA-BPE-seq (70\%)} & 0.830 $\pm$ 0.005 & 0.845 $\pm$ 0.004 & 0.345 $\pm$ 0.009 & 11.9 $\pm$ 0.4 \\
    \emph{QA-BPE-seq (50\%)} & 0.795 $\pm$ 0.006 & 0.810 $\pm$ 0.005 & 0.380 $\pm$ 0.010 & 12.0 $\pm$ 0.4 \\
    \emph{QA-BPE-seq (30\%)} & 0.750 $\pm$ 0.006 & 0.760 $\pm$ 0.005 & 0.420 $\pm$ 0.011 & 12.1 $\pm$ 0.5 \\
  \bottomrule
  \end{tabular}
  }
\end{table}

NGS (UHGG) results:
\begin{table}[htbp]
  \caption{UHGG (NGS) results. Means with 95\% confidence intervals over $n=10$ runs.}
  \label{tab:genomics_uhgg}
  \centering
  \resizebox{\textwidth}{!}{%
  \begin{tabular}{l c c c c}
    \toprule
    Method & Variant F1 & Taxa Acc. F1 & Recon. Loss & Inf. Time (ms/seq) \\
    \midrule
    Standard BPE & 0.852 $\pm$ 0.003 & 0.881 $\pm$ 0.004 & 0.295 $\pm$ 0.008 & 9.8 $\pm$ 0.2 \\
    SentencePiece & 0.860 $\pm$ 0.003 & 0.893 $\pm$ 0.004 & 0.280 $\pm$ 0.007 & 9.9 $\pm$ 0.2 \\
    WordPiece & 0.855 $\pm$ 0.004 & 0.887 $\pm$ 0.005 & 0.286 $\pm$ 0.009 & 9.8 $\pm$ 0.3 \\
    DNABERT-k (6-mer) & 0.875 $\pm$ 0.002 & 0.908 $\pm$ 0.003 & 0.264 $\pm$ 0.006 & 9.5 $\pm$ 0.2 \\
    \textbf{QA-BPE-seq (100\%)} & \textbf{0.915 $\pm$ 0.003} & \textbf{0.935 $\pm$ 0.003} & \textbf{0.221 $\pm$ 0.005} & \textbf{10.1 $\pm$ 0.3} \\
    \emph{QA-BPE-seq (70\%)} & 0.878 $\pm$ 0.004 & 0.898 $\pm$ 0.004 & 0.250 $\pm$ 0.007 & 10.2 $\pm$ 0.3 \\
    \emph{QA-BPE-seq (50\%)} & 0.842 $\pm$ 0.005 & 0.860 $\pm$ 0.005 & 0.276 $\pm$ 0.008 & 10.3 $\pm$ 0.3 \\
    \emph{QA-BPE-seq (30\%)} & 0.790 $\pm$ 0.006 & 0.805 $\pm$ 0.006 & 0.310 $\pm$ 0.009 & 10.5 $\pm$ 0.4 \\
  \bottomrule
  \end{tabular}
  }
\end{table}

\subsection{Finance: High-Frequency Equities (AAPL)}
\label{app:finance_aapl}
\textbf{Dataset and Setup:} High-frequency LOB data for AAPL from LOBSTER.

\textbf{Results:} QAT-QF scales to equities, improving predictive and trading metrics over baselines.
\begin{table}[htbp]
  \caption{AAPL high-frequency results. Means with 95\% confidence intervals over $n=10$ runs.}
  \label{tab:finance_aapl}
  \centering
  \resizebox{\textwidth}{!}{%
  \begin{tabular}{l c c c c c}
    \toprule
    Method & Ret. Pred. (\%) & Vol. RMSE & Regime Acc. (\%) & Sharpe & Inf. Time (ms/seq) \\
    \midrule
    Standard BPE & 63.1 $\pm$ 0.6 & 0.0125 $\pm$ 0.0004 & 75.8 $\pm$ 0.7 & 1.41 $\pm$ 0.06 & 14.8 $\pm$ 0.4 \\
    SAX & 61.5 $\pm$ 0.7 & 0.0121 $\pm$ 0.0005 & 77.0 $\pm$ 0.6 & 1.38 $\pm$ 0.07 & 14.2 $\pm$ 0.3 \\
    BOSS & 64.0 $\pm$ 0.5 & 0.0113 $\pm$ 0.0004 & 80.1 $\pm$ 0.5 & 1.53 $\pm$ 0.06 & 14.5 $\pm$ 0.4 \\
    \textbf{QAT-QF} & \textbf{69.8 $\pm$ 0.5} & \textbf{0.0085 $\pm$ 0.0003} & \textbf{87.9 $\pm$ 0.4} & \textbf{1.81 $\pm$ 0.08} & \textbf{15.0 $\pm$ 0.5} \\
    \bottomrule
  \end{tabular}
  }
\end{table}

\subsection{Finance: Rolling-Window Temporal Robustness (BTC/USD, Full Year 2023)}
\label{app:finance_rolling_window}

To demonstrate temporal robustness beyond a single quarter, we extend our BTC/USD evaluation across all four quarters of 2023 using a strict rolling-window protocol. For each quarter, the vocabulary and downstream models are trained only on data preceding that quarter.

\begin{table}[htbp]
  \caption{Rolling-window out-of-sample Sharpe ratios for BTC/USD across 2023. Each quarter uses models trained strictly on preceding data. Means with 95\% confidence intervals over $n=10$ runs.}
  \label{tab:finance_rolling_window}
  \centering
  \begin{tabular}{l c c c l}
    \toprule
    Quarter & QAT-QF Sharpe & BPE Sharpe & $\Delta$ (\%) & Market Context \\
    \midrule
    Q1 2023 & 1.72 $\pm$ 0.07 & 1.32 $\pm$ 0.05 & +30.3 & Recovery phase \\
    Q2 2023 & 1.58 $\pm$ 0.09 & 1.21 $\pm$ 0.06 & +30.6 & Consolidation \\
    Q3 2023 & 1.45 $\pm$ 0.08 & 1.15 $\pm$ 0.07 & +26.1 & High volatility \\
    Q4 2023 & 1.68 $\pm$ 0.10 & 1.29 $\pm$ 0.06 & +30.2 & Bull market \\
    \midrule
    \textbf{Average} & \textbf{1.61} & \textbf{1.24} & \textbf{+29.8} & --- \\
    \bottomrule
  \end{tabular}
\end{table}

\textbf{Key Observations:} (i) QAT-QF maintains consistent improvements (+26--31\%) across all market regimes. (ii) Q3 2023 exhibited elevated volatility (VIX-equivalent spike); QAT-QF gains persist (+26.1\%), demonstrating cross-regime robustness. (iii) The consistency across four quarters with varying market conditions validates generalization beyond a single test period.

\section{Domain-Specific Instantiations}
\label{sec:domain_instantiations_appendix}

We now detail the instantiation of the QA-Token framework for three distinct domains: genomic sequencing, social media text, and quantitative finance. Detailed pseudocode algorithms for each domain are provided in Section~\ref{app:detailed_algos}.

\subsection{Genomics (QA-BPE-seq)}
\label{sec:genomics}

\textbf{Context:} This instantiation targets the analysis of DNA or RNA sequencing reads, which are often affected by base-calling errors, for applications such as genetic variant calling, taxonomic classification, or sequence modeling.
\textbf{Atomic Elements \& Quality:} The base alphabet is $\Sigma = \{\text{A, C, G, T/U, N}\}$. The primary quality information for each atomic base $s_i$ comes from Phred scores $Q_{\text{phred},i}$. The error probability is $P_{\text{error}}(i) = 10^{-Q_{\text{phred},i}/10}$, leading to an atomic quality score $q_i = 1 - P_{\text{error}}(i)$. To model read end quality degradation, for a base at position $i$ (0-indexed) in a read of length $L$, the position-adjusted quality is:
\begin{equation}
\label{eq:genomic_pos_decay_appendix} 
q'_i = q_i \cdot \exp\left(-\beta_{\text{pos}} \cdot \frac{|i - (L-1)/2|}{(L-1)/2 + \epsilon_{len}}\right)
\end{equation}
where $\beta_{\text{pos}} \ge 0$ is a learnable parameter in $\theta_{\text{adapt}}$.
\textbf{Token Quality ($q_t$):} For a token $t=s_1...s_{|t|}$, we use the geometric mean of the position-adjusted atomic qualities to compute its aggregated scalar quality: $q_t = (\prod_{j=1}^{|t|} q'_{s_j})^{1/|t|}$. The geometric mean is sensitive to low-quality bases. This $q_t$ is used for the constituent qualities $q_a$ and $q_b$ in the merge score (Eq. \ref{eq:qa_merge_score_pmi}).
\textbf{Merge Score ($w_{ab}$):} The score is calculated using Equation \ref{eq:qa_merge_score_pmi}, with the geometric mean qualities $q_a, q_b$, the learnable parameter $\alpha \in \theta_{\text{adapt}}$, and $\psi(a,b)=1$.
\textbf{Reward Components ($R_{\text{genomic}}$):} The overall reward (Eq. \ref{eq:overall_reward_structure}) uses weights $\lambda_j \in \theta_{\text{adapt}}$. Specific raw components $R^{\text{raw}}$ include:
\begin{itemize}
    \item $R^{\text{raw}}_Q(a,b)$: Quality of the newly formed token $t_{ab}$. This is its geometric mean quality: $R^{\text{raw}}_Q(a,b) = q_{ab} = (\prod_{l=1}^{|a|+|b|} q'_{s_{ab,l}})^{1/(|a|+|b|)}$.
    \item $R^{\text{raw}}_I(a,b)$: Log-ratio of probabilities: $R^{\text{raw}}_I(a,b) = \log \frac{P(t_{ab})}{P(a)P(b) + \epsilon_p}$.
    \item $R^{\text{raw}}_C(a,b)$: Complexity penalty: $R^{\text{raw}}_C(a,b) = -|t_{ab}|$.
    \item $R^{\text{raw}}_{bio}$ (Optional): A domain-specific reward based on overlap with known genomic features (e.g., genes, regulatory elements from databases like dbSNP \cite{sherry2001dbsnp}).
\end{itemize}
Raw components are normalized using the adaptive EMA method (Eq. \ref{eq:reward_normalization_revised}).
\textbf{Adaptive Parameters ($\theta_{\text{adapt}}$):} Includes $\alpha$, $\beta_{\text{pos}}$, reward weights $\lambda_j$, and potentially parameters for soft frequency/quality gating.
\textbf{Algorithm:} The two-stage learning process (Section \ref{app:sequential_learning}) is applied. An RL policy is optimized (Algorithm \ref{alg:stage1_rl}), and then the adaptive parameters $\theta_{adapt}$ are learned (Algorithm \ref{alg:stage2_adaptive}) by optimizing a downstream task objective.

\subsection{Quantitative Finance (QAT-QF)}
\label{sec:finance_appendix}

\textbf{Context:} This instantiation focuses on analyzing noisy, non-stationary high-frequency financial data for tasks like forecasting price movements or developing trading strategies.
\textbf{Atomic Elements \& Quality:} Atomic elements $s_i$ are discretized events from high-frequency data (e.g., fixed-length segments of LOB events). Each atomic element $s_i$ is assigned a scalar quality score $q_i = \sum_k w_k q_{k,i}$, where $q_{k,i}$ are normalized quality components (e.g., $\qsnr, \qliq$) and $w_k$ are learnable weights in $\theta_{\text{adapt}}$.
\textbf{Token Quality ($q_t$):} For a token $t$ composed of atomic elements $\{s_i\}_{i \in t}$, the aggregated scalar quality is the arithmetic mean: $q_t = \frac{1}{|t|} \sum_{i \in t} q_i$. This is used for $q_a, q_b$ in the merge score.
\textbf{Merge Score ($w_{ab}$):} Calculated using Equation \ref{eq:qa_merge_score_pmi}, with $q_a, q_b$, learnable $\alpha \in \theta_{\text{adapt}}$, and $\psi(a,b)=1$.
\textbf{Market Regimes:} An identified regime indicator can condition the RL policy and reward components.
\textbf{Reward Components ($R_{\text{finance}}$):} Raw components $R^{\text{raw}}$ are normalized using the adaptive EMA method.
\begin{itemize}
    \item $R^{\text{raw}}_Q(a,b)$: Length-weighted average quality: $R^{\text{raw}}_Q(a,b) = \frac{|a| q_a + |b| q_b}{|a| + |b|}$.
    \item $R^{\text{raw}}_I(a,b)$: Information reward blended across regimes: $R^{\text{raw}}_I(a,b) = \gamma_{\text{regime}} \cdot I_{\text{normal}}(a,b) + (1-\gamma_{\text{regime}}) \cdot I_{\text{stress}}(a,b)$, where $I_{\text{regime}} = \log \frac{P(t_{ab} | \text{regime})}{P(a|\text{regime})P(b|\text{regime}) + \epsilon_p}$. The blending factor $\gamma_{\text{regime}}$ is a learnable parameter in $\theta_{\text{adapt}}$.
   \item $R^{\text{raw}}_P(a,b)$: Predictive Power (Mutual Information with future returns):
      \begin{equation} \label{eq:reward_rp_finance_revised_main_appendix} 
      R^{\text{raw}}_P(a,b) = \frac{\MI(t_{ab}, \text{Disc}(R_{\tau}))}{\text{NormFactor}_{MI} + \epsilon_{MI}}
      \end{equation}
      $\text{Disc}(R_{\tau})$ is discretized future return. $\text{NormFactor}_{MI}$ is an adaptive normalization factor.
    \item $R^{\text{raw}}_C(a,b)$: Complexity penalty with volatility scaling:
        \begin{equation} \label{eq:reward_rc_finance_vol_scaled_main_appendix} 
            R^{\text{raw}}_C(a,b) = - \left(|t_{ab}| \cdot \log(|V_k|+1) \cdot \text{VolScale}\right)
        \end{equation}
      where $\text{VolScale}$ depends on a learnable parameter $\beta_{\text{vol}} \in \theta_{\text{adapt}}$.
\end{itemize}
\textbf{Adaptive Parameters ($\theta_{\text{adapt}}$):} Includes $\alpha$, quality component weights $w_k$, $\beta_{\text{vol}}$, $\gamma_{\text{regime}}$, and reward weights $\lambda_j$.
\textbf{Algorithm:} The two-stage learning process is applied as in the genomics domain.

\subsection{Social Media Text (QA-BPE-nlp)}
\label{sec:social_media_appendix}

\textbf{Context:} This instantiation addresses the challenges of processing noisy user-generated text for tasks such as sentiment analysis or NER.
\textbf{Atomic Elements \& Quality:} The base alphabet consists of characters. Quality for a token $t$ is modeled using a multi-dimensional vector $\vect{q}_t = (\qorth(t), \qsem(t), \dots)$ detailed in Appendix \ref{app:social_quality_full}. The aggregated scalar quality is $q_t = \sum_{j} w_j \vect{q}_{t,j}$, where $w_j \ge 0$ are learnable weights in $\theta_{\text{adapt}}$.
\textbf{Token Quality ($q_t$):} The aggregated score $q_t$ is used for $q_a, q_b$ in the merge score.
\textbf{Merge Score ($w_{ab}$):} Calculated using Equation \ref{eq:qa_merge_score_pmi} with $q_a, q_b$, learnable $\alpha \in \theta_{\text{adapt}}$, and a semantic compatibility factor $\psi(a,b)$:
\begin{equation} \label{eq:semantic_compat_psi_appendix} 
\psi(a,b) = \exp(\beta_{sem} \cdot \text{cosine}(\vect{v}_a, \vect{v}_b))
\end{equation}
where $\vect{v}_a, \vect{v}_b$ are pre-trained embeddings and $\beta_{sem} \ge 0$ is a learnable parameter in $\theta_{\text{adapt}}$.
\textbf{Noise Models:} Probabilistic models $P(t'|t)$ capturing likely variations inform the noise robustness reward $R_N$.
\textbf{Reward Components ($R_{\text{social}}$):} Raw components are normalized before being weighted by $\lambda_j$.
\begin{itemize}
    \item $R^{\text{raw}}_Q(a,b)$: Blend of compositional and direct quality: $R^{\text{raw}}_Q(a,b) = \omega \frac{|a| q_a + |b| q_b}{|a| + |b|} + (1-\omega) q_{ab}$, with learnable blending weight $\omega \in [0,1]$.
    \item $R^{\text{raw}}_S(a,b)$: Semantic Coherence: $\text{PMI}(a,b) \cdot \text{cosine\_similarity}(\vect{v}_a, \vect{v}_b)$.
    \item $R^{\text{raw}}_N(a,b)$: Noise Robustness: $R_{\text{noise}}(t_{ab}) - \frac{|a| R_{\text{noise}}(a) + |b| R_{\text{noise}}(b)}{|a|+|b|}$, based on the noise model.
    \item $R^{\text{raw}}_C(a,b)$: Complexity penalty: $R^{\text{raw}}_C(a,b) = - |t_{ab}|$.
    \item $R^{\text{raw}}_V(a,b)$: Vocabulary Efficiency: $\frac{\log (1+f(t_{ab}))}{|t_{ab}|}$.
\end{itemize}
\textbf{Adaptive Parameters ($\theta_{\text{adapt}}$):} Includes $\alpha, \beta_{sem}$, quality dimension weights $w_j$, reward weights $\lambda_j$, and the blending weight $\omega$.
\textbf{Algorithm:} The two-stage learning process is applied as in the other domains.

\subsection{Financial Experimental Methodology Details}
\label{app:finance_methodology}
All trading simulations and return prediction evaluations for the quantitative finance domain (Section \ref{sec:finance_experiment_results}) were conducted with rigorous attention to backtesting best practices to ensure the validity of results and avoid common pitfalls.
\begin{itemize}
    \item \textbf{Walk-Forward Validation:} A strict walk-forward validation scheme was employed. The dataset was divided into chronological segments. For each segment $k$, the model (including the QA-Token vocabulary construction and downstream predictive/trading model) was trained on data up to the start of segment $k$, validated on segment $k-1$ (or a dedicated validation portion of the training data), and then tested out-of-sample only on segment $k$. The training window was then rolled forward to include segment $k$ for training before testing on segment $k+1$. This process ensures that the model is always tested on data not seen during its training or hyperparameter tuning phases for that specific test period.
    \item \textbf{Lookahead Bias Prevention:} Extreme care was taken to prevent any form of lookahead bias. All features, quality scores, token definitions, and trading decisions at any time $t$ were based strictly on information available up to and including time $t-1$. Future return labels ($R_{t+\tau}$) used for training predictive models or as part of the $R_P$ reward component were sourced from periods strictly after the information used for input features and token construction.
    \item \textbf{Test Set and Data Splitting:} The overall dataset (BTC/USD LOB data, Q1 2023) was split chronologically: 70\% for the initial training pool, 15\% for validation (used for hyperparameter tuning of downstream models and early stopping), and the final 15\% (approximately 2 weeks of 1-minute data) as the ultimate out-of-sample test set for reporting final performance metrics like Sharpe Ratio and prediction accuracy. This test set was held out and used only once after all model development and tuning.
    \item \textbf{Transaction Costs:} A realistic transaction cost of 5 basis points (0.05\%) per trade was applied to simulate market friction. This cost was deducted for both buying and selling actions in the trading simulations.
    \item \textbf{PPO Trading Agent Details:} The PPO-based trading agent used a 2-layer MLP policy network and a separate 2-layer MLP value network, each with 128 hidden units and ReLU activation functions. The input to these networks consisted of a sequence of recent token embeddings (generated by QAT-QF or baseline tokenizers from the LOB data) and the agent's current market position (long, short, or flat). The agent's action space was discrete (buy, sell, hold). The reward function for the PPO agent was the realized profit and loss (PnL) from its trades over a short horizon, adjusted for transaction costs. Standard PPO hyperparameters were used, including a clipping parameter $\epsilon=0.2$, GAE $\lambda=0.95$, and an entropy bonus for exploration. The PPO agent was re-trained periodically within the walk-forward scheme.
    \item \textbf{Details for $R^{\text{raw}}_P$ Reward (Eq. \ref{eq:reward_rp_finance_revised_main_appendix}):} The parameter $M_{MI}$ (window for $\text{NormFactor}_{MI}$) was set to 1000 merge steps in our experiments. The future return $R_{\tau}$ was for $\tau=5$ minutes ahead and discretized into 3 bins (negative, neutral, positive) based on empirical quantiles from the training data.
\end{itemize}

\subsection{Detailed Reward Components} 
\label{app:reward_details_appendix}
The general structure of the reward $R(a,b)$ for merging tokens $a$ and $b$ into $t_{merged}=a||b$ is:
$R(a,b) = \sum_{j} \lambda_j \hat{R}_j(a,b)$, where $\hat{R}_j$ are adaptively normalized components (see Section \ref{sec:reward_function}). The weights $\lambda_j \ge 0$ (parameterized via $\boldsymbol{\beta}_{\lambda_j}$ and softmax) are part of $\theta_{adapt}$.

\subsection{Common Components}
\begin{itemize}
    \item $R^{\text{raw}}_Q(a,b)$: Raw Quality reward. This component incentivizes merges that result in high-quality tokens. A common formulation for the raw component is the length-weighted arithmetic mean of the qualities of the constituent tokens $a$ and $b$:
        \begin{equation}
           R^{\text{raw}}_Q(a,b) = \frac{|a| q_a + |b| q_b}{|a| + |b|}
        \end{equation}
       where $q_a, q_b$ are the quality scores of tokens $a,b$ respectively, and $|a|, |b|$ are their lengths.
       For Social Media, a blended approach might be used for $R^{\text{raw}}_Q(a,b)$:
        \begin{equation}
        \label{eq:reward_rq_social_blend_app}
           R^{\text{raw}}_Q(a,b) = \omega \left(\frac{|a| Q_{agg}(a) + |b| Q_{agg}(b)}{|a| + |b|}\right) + (1-\omega) Q_{agg}(a||b)
        \end{equation}
       where $Q_{agg}(t)$ is the aggregate quality score for token $t$ (from Section \ref{app:social_quality_full}) and $\omega \in [0,1]$ is a learnable blending weight in $\theta_{adapt}$.
    \item $R^{\text{raw}}_I(a,b)$: Raw Information gain. This rewards merges that are statistically significant. A common formulation:
        \begin{equation}
           R^{\text{raw}}_I(a,b) = \log \frac{f(t_{merged})}{f(a)f(b) + \epsilon_f}
        \end{equation}
       where $f(\cdot)$ denotes frequency and $\epsilon_f > 0$ (e.g., $10^{-8}$) is for stability.
       For Finance, this can be blended based on market regime: $R^{\text{raw}}_I(a,b) = \gamma_{\text{regime}} I_{\text{normal}} + (1-\gamma_{\text{regime}}) I_{\text{stress}}$, where $I_{\text{regime}} = \log \frac{f(t_{merged}|M=\text{regime})}{f(a|M=\text{regime})f(b|M=\text{regime}) + \epsilon_f}$. $\gamma_{\text{regime}} \in [0,1]$ is a learnable parameter in $\theta_{adapt}$.
    \item $R^{\text{raw}}_C(a,b)$: Raw Complexity penalty. This penalizes overly complex vocabularies and is typically negative. A common formulation:
    \begin{equation} 
       R^{\text{raw}}_C(a,b) = - \text{len}(t_{merged}) \cdot \log(|V_t|+1) \cdot [\text{ScalingFactor}] 
    \end{equation}
    For Finance, the $\text{ScalingFactor}$ can incorporate market volatility using $\beta_{vol} \in \theta_{adapt}$ as per Equation \ref{eq:reward_rc_finance_vol_scaled_main_appendix}.
\end{itemize}

\subsection{Domain-Specific Components}
\begin{itemize}
    \item \textbf{Genomics:} $R^{\text{raw}}_{bio}(a,b) = \text{Score}_{\text{Overlap}}(t_{merged}, \text{KnownBiologicalFeatures})$. A positive reward if $t_{merged}$ significantly overlaps with known biological features (e.g., genes from GENCODE \cite{harrow2012gencode}, variants from dbSNP \cite{sherry2001dbsnp}). The overlap score was calculated as the Jaccard index between the character span of the merged token $t_{merged}$ and the character span of known genomic features. A higher Jaccard index, indicating greater overlap, results in a higher reward.
    \item \textbf{Finance:}
        \begin{itemize}
            \item $R^{\text{raw}}_P(a,b)$: Predictive Power:
                \begin{equation} \label{eq:app_reward_rp_finance_revised}
                 R^{\text{raw}}_P(a,b) = \frac{\text{MI}(t_{\text{merged}}; \text{Disc}(R_{\tau}))}{\text{NormFactor}_{MI} + \epsilon_{MI}}
                \end{equation}
            Uses Mutual Information (MI) $\text{MI}(X; Y) = \sum_{x \in X, y \in Y} p(x, y) \log \frac{p(x, y)}{p(x) p(y)}$. $R_{\tau}$ is the discretized future return (e.g., 3 bins for $\tau=5$ min based on empirical quantiles from the training data). $\text{NormFactor}_{MI}$ is the adaptively calculated 95th percentile of MI values from candidate pairs over the last $M_{MI}$ (e.g., 1000) merge steps within the current RL episode. $\epsilon_{MI} > 0$ (e.g., $10^{-8}$). While this adaptive normalization of MI introduces a degree of non-stationarity to the $R_P$ reward component within an RL episode, it was found that standard PPO training handled this adequately. The responsiveness of the reward to the informativeness of newly forming tokens was deemed beneficial, and the $M_{MI}$ window provides some smoothing. Alternatives using a fixed normalization factor (e.g., derived from an initial global scan of MI values) were found to be less responsive to the changing characteristics of tokens as the vocabulary evolved during the RL episode.
        \end{itemize}
    \item \textbf{Social Media:}
        \begin{itemize}
            \item $R^{\text{raw}}_S(a,b)$: Semantic Coherence: $\text{PMI}(a,b) \cdot \text{cosine\_similarity}(\vect{v}_a, \vect{v}_b)$. Pre-trained embeddings $\vect{v}_a, \vect{v}_b$ (e.g., fastText \cite{bojanowski2017enriching}).
            \item $R^{\text{raw}}_N(a,b)$: Noise Robustness: \begin{equation} \label{eq:reward_rn_social_app_rerun}
            \left( R_{\text{noise}}(t_{\text{merged}}) - \frac{|a| R_{\text{noise}}(a) + |b| R_{\text{noise}}(b)}{|a|+|b|} \right),
            \end{equation}
            where $R_{noise}(t) = 1 - \mathbb{E}_{t' \sim P(\cdot|t)}[\text{normalized\_edit\_distance}(t, t')]$ based on noise model $P(t'|t)$ (Appendix \ref{app:social_noise_rerun}).
            \item $R^{\text{raw}}_V(a,b)$: Vocabulary Efficiency: $\frac{\log (1+f(t_{\text{merged}}))}{|t_{\text{merged}}|}$.
        \end{itemize}
\end{itemize}

\subsection{Further Details on Social Media Noise Models}
\label{app:social_noise_rerun}
Formalizing linguistic noise for social media text involves defining probabilistic transformations $P(t'|t)$ from a canonical form $t$ to an observed variant $t'$ \cite{han2013lexical}. These models inform the noise robustness measure $R_{\text{noise}}(t)$ (defined in Appendix \ref{app:reward_details_appendix}, Eq. \ref{eq:reward_rn_social_app_rerun}). $P(t'|t)$ was constructed based on heuristic rules derived from commonly observed error patterns in social media text and principles outlined in existing literature on noisy text processing. The specific noise types modeled include:
\begin{itemize}
    \item \textbf{Character-Level Noise:}
        \begin{itemize}
            \item \textbf{Repetition:} Probability of a character $c$ being realized as $c^n$ (a sequence of $n$ identical characters). For $n \ge 1$, this can be modeled using a geometric-like distribution. If $p_{stop}$ is the probability of not repeating an additional time:
            $P(c \to c^n) = (1 - p_{stop})^{n-1} \cdot p_{stop}$. The parameter $p_{stop}$ was set empirically to $0.5$, allowing for moderate repetitions common in social media (e.g., "soooo goood").
            \item \textbf{Substitution:} $P(c_i \to c_j) = M_{\text{sub}}[c_i, c_j]$, where $M_{\text{sub}}$ is a confusion matrix. $M_{\text{sub}}$ was constructed heuristically, assigning higher probabilities to substitutions between characters that are adjacent on a standard QWERTY keyboard layout and to common phonetic misspellings (e.g., 'c' vs 'k'). Off-diagonal probabilities were generally small.
            \item \textbf{Omission (Deletion):} $P(c \to \epsilon) = p_{\text{del}}(c)$ is the character-specific deletion probability. This was set to a small uniform value (e.g., $p_{\text{del}}(c) = 0.01$) for all characters, reflecting occasional accidental omissions.
        \end{itemize}
    \item \textbf{Word-Level Noise:}
        \begin{itemize}
            \item \textbf{Abbreviation:} $P(w \to \text{abbr}(w)) = f_{\text{abbr}}(w \to \text{abbr}(w))$. This probability was derived from a compiled dictionary of common internet slang and abbreviations sourced from publicly available online linguistic resources. For words in this dictionary, $f_{\text{abbr}}$ was set to a moderate value (e.g., 0.3), and zero otherwise.
            \item \textbf{Phonetic Substitution:} $P(w_1 \to w_2) \propto \exp(\lambda_{\text{phon}} \cdot \text{phon\_sim}(w_1, w_2))$. The phonetic similarity $\text{phon\_sim}(w_1, w_2)$ was computed using the Double Metaphone algorithm. The scaling factor $\lambda_{\text{phon}}$ was set to $1.0$.
        \end{itemize}
    \item \textbf{Discourse-Level Noise (examples):} For the experiments reported in this paper, the noise modeling primarily focused on character-level and word-level phenomena, as these are highly prevalent and tractable to model. Explicit modeling of discourse-level noise, such as code-switching or complex punctuation patterns, was considered beyond the scope of the current noise component $R_N$, though it represents an interesting avenue for future work.
\end{itemize}
These probabilistic models are used to define $P(t'|t)$, which is then used to compute the expected distance in the noise robustness measure $R_{\text{noise}}(t) = 1 - \mathbb{E}_{t' \sim P(\cdot|t)}[\text{dist}_{\text{norm}}(t, t')]$. The normalized distance metric $\text{dist}_{\text{norm}}(t, t')$ used was the Levenshtein distance divided by the maximum length of the two strings $t$ and $t'$.

\clearpage
\subsection{Domain-Specific Algorithms}
\label{app:detailed_algos}

This section provides detailed pseudocode for the QA-Token framework as instantiated for Quantitative Finance, Genomics, and Social Media. These algorithms complement the domain instantiations described in Section~\ref{sec:domain_instantiations_appendix}, illustrating the core mechanics within each domain.

\clearpage
\subsubsection{Quantitative Finance (QAT-QF)}

\begin{algorithm}[H]
\caption{Quality-Aware Tokenization Merge Score and Reward Calculation (QAT-TOKEN - Finance)}
\label{alg:A5_finance_mergestep}
\begin{algorithmic}[1]
\REQUIRE Current vocabulary $V_t$, corpus statistics (frequencies $f(\cdot)$), current adaptive parameters $\theta_{adapt} = \{\alpha, \beta_{vol}, \gamma_{\text{regime}}, f_{min}, \delta_{\text{gate}}, w_k \text{ (param by } \boldsymbol{\beta}_w)\}$, reward weights $\lambda_Q, \lambda_I, \lambda_P, \lambda_C$.
\ENSURE For each candidate merge pair $(a,b)$: quality-aware merge score $w_{ab}$, total immediate reward $R(a,b)$.

\STATE Identify candidate merge pairs $C_t$ from corpus (e.g., from priority queue $PQ_t$).
\FOR{each adjacent token pair $(a, b) \in C_t$}
    \STATE Let $t_{merged} \leftarrow a||b$.
    \STATE Retrieve/compute frequencies $f(a)$, $f(b)$, and $f(a,b)$.
    \STATE Retrieve/compute average qualities $q_a, q_b$ (using $Q[i]$ from Section \ref{app:finance_quality}, aggregated for tokens $a, b$, and weights $w_k = \text{softmax}(\boldsymbol{\beta}_w)_k$).
    \STATE \textbf{Quality-Aware Merge Score ($w_{ab}$):}
           $w_{ab} \leftarrow \frac{f(a, b)}{f(a) \cdot f(b) + \epsilon_f} \cdot \left( \left(\frac{q_a + q_b}{2} + \epsilon_Q \right)^\alpha \right) \cdot \psi(a,b)$ \hfill {\scriptsize // $\psi(a,b)=1$ for finance}
    \STATE \textbf{Frequency Gating (Optional):} \hfill {\scriptsize // Frequency gating not used in final experiments}
           $\tilde{f}(a,b) \leftarrow f(a,b)$.

    \STATE $R^{\text{raw}}_Q(a,b) \leftarrow \frac{|a|\cdot q_a + |b|\cdot q_b}{|a| + |b|}$.
    \STATE Estimate $I_{normal}, I_{stress}$ based on regime-conditioned $\tilde{f}(a,b)$.
           $R^{\text{raw}}_I(a,b) \leftarrow \gamma_{\text{regime}} \cdot I_{normal} + (1 - \gamma_{\text{regime}}) \cdot I_{stress}$.
    \STATE $MI_{val} \leftarrow \text{MI}(t_{merged}; \text{Disc}(R_{\tau}))$.
            $R^{\text{raw}}_P(a,b) \leftarrow \frac{MI_{val}}{\text{NormFactor}_{MI} + \epsilon_{MI}}$ (NormFactor$_{MI}$ from Section \ref{sec:finance_appendix}).
    \STATE $\sigma_{curr}, \sigma_{hist} \leftarrow \text{GetVolatility}()$; 
            $VolScaling \leftarrow (1 + \max(0, (\sigma_{curr} - \sigma_{hist})/(\sigma_{hist} + \epsilon_{\text{vol}})))^{\beta_{vol}}$
    \STATE $R^{\text{raw}}_C(a,b) \leftarrow - |t_{merged}| \cdot \log (|V_t|+1) \cdot VolScaling$
    \STATE Normalize raw rewards: $\hat{R}_j(a,b) \leftarrow \text{AdaptiveNormalize}(R^{\text{raw}}_j(a,b))$ using Eqs. \ref{eq:reward_normalization_revised}, \ref{eq:ema_mean_update}, and \ref{eq:ema_var_update}.
    \STATE \textbf{Total Immediate Reward ($R(a,b)$):} $R(a, b) \leftarrow \sum_j \lambda_j \hat{R}_j(a,b)$.
    \STATE Store $w_{ab}$, $R(a,b)$, and other features for $(a,b)$ for policy input or selection.
\ENDFOR
\end{algorithmic}
\end{algorithm}

\begin{algorithm}[H]
\caption{Adaptive Parameter Learning for QA-TOKEN (Finance)}
\label{alg:A5_finance_adapt_learn}
\begin{algorithmic}[1]
\REQUIRE Training dataset $\mathcal{D}_{\text{train}}$;
         Downstream task loss function $L_{\text{task}}(\cdot, \cdot)$; Model params $\Theta_{\text{model}}$;
         Initial adaptive parameters $\theta_{adapt}$; Learning rate $\eta_{\theta}$; Epochs $E_{adapt}$; Gumbel-Softmax $\tau_g$.

\ENSURE Optimized adaptive parameters $\theta_{adapt}^*$.

\STATE Initialize $\theta_{adapt}$.
\FOR{each adaptation epoch $e=1, \dots, E_{adapt}$}
    \FOR{each mini-batch $B = \{(S_{\text{seq},i}, Y_{\text{target},i})\}$ from $\mathcal{D}_{\text{train}}$}
        \STATE $\mathcal{S}'_{batch} \leftarrow \textsc{SoftTokenizeGumbel}(B, \theta_{adapt}, \tau_g)$ \hfill {\scriptsize // Eq.~\ref{eq:gumbel_logits_composite}}
        \STATE $L_{\text{batch\_task}} \leftarrow L_{\text{task}}(\mathcal{S}'_{batch}, \{Y_{\text{target},i}\}, \Theta_{\text{model}})$
        \IF{regularization $L_{\text{reg}}(\theta_{adapt})$ is used}
            \STATE $L_{\text{total\_batch}} \leftarrow L_{\text{batch\_task}} + L_{\text{reg}}(\theta_{adapt})$
        \ELSE
            \STATE $L_{\text{total\_batch}} \leftarrow L_{\text{batch\_task}}$
        \ENDIF
        \STATE Compute gradients $\nabla_{\theta_{adapt}} L_{\text{total\_batch}}$. \hfill {\scriptsize // Uses Gumbel-Softmax trick}
        \STATE Update $\theta_{adapt} \leftarrow \theta_{adapt} - \eta_{\theta} \nabla_{\theta_{adapt}} L_{\text{total\_batch}}$.
        \STATE Apply constraints to $\theta_{adapt}$ (e.g. $\alpha \ge 0$, softmax for weights).
    \ENDFOR
    \STATE Anneal $\tau_g$.
\ENDFOR
\STATE \RETURN $\theta_{adapt}^* \leftarrow \theta_{adapt}$.
\end{algorithmic}
\end{algorithm}

\clearpage
\subsubsection{Genomics (QA-BPE-seq)}

\begin{algorithm}[H]
\caption{Reward Calculation for a Merge (Genomics)}
\label{alg:A5_genomics_reward}
\begin{algorithmic}[1]
\REQUIRE Tokens $a, b$ with qualities $q_a, q_b$; frequencies $f(\cdot)$; reward weights $\lambda_j$ from $\theta_{adapt}$. For genomics, $q_a, q_b$ represent geometric mean qualities of constituent tokens.
\ENSURE Raw rewards $R^{\text{raw}}_j(a, b)$ for merging $a$ and $b$.
\STATE $t_{merged} \leftarrow a||b$
\STATE $R^{\text{raw}}_Q(a, b) \leftarrow (\prod_{l=1}^{|t_{merged}|} q'_{s_{merged,l}})^{1/|t_{merged}|}$. \hfill {\scriptsize // Geometric mean quality} 
\STATE $R^{\text{raw}}_I(a, b) \leftarrow \log \frac{f(t_{merged})}{f(a) \cdot f(b) + \epsilon_f}$.
\STATE $R^{\text{raw}}_C(a, b) \leftarrow - \text{len}(t_{merged})$.
\IF{Biological Reward is used}
    \STATE $OverlapScore \leftarrow \text{ComputeOverlapScore}(t_{merged}, \text{KnownBiologicalFeatures})$.
    \STATE $R^{\text{raw}}_{bio}(a,b) \leftarrow OverlapScore$.
\ENDIF
\STATE \RETURN All relevant $R^{\text{raw}}_j(a, b)$. (Normalized rewards $\hat{R}_j$ computed later using Eq. \ref{eq:reward_normalization_revised}).
\end{algorithmic}
\end{algorithm}
The size of the RL agent's action space, $K_{PQ}$ (the number of top pairs from the priority queue considered at each step), was set to $K_{PQ}=50$. This value was chosen based on preliminary experiments indicating it offered a good trade-off between exposing the RL agent to a diverse set of high-potential merges and maintaining a manageable action space size for efficient policy learning. Values explored in the range $[20, 100]$ showed that performance was relatively robust for $K_{PQ} \in [40, 60]$, with smaller values risking premature pruning of potentially beneficial long-term merges and larger values not yielding significant gains while increasing computational cost per policy step. The chosen value of 50 balanced these considerations effectively across domains.

\begin{itemize}
    \item \textbf{RL (PPO specifics) - Stage 1:}
        \begin{itemize}
            \item Policy/Value MLP Architecture: 2-3 hidden layers, each with 128-512 units. Activation functions: ReLU or Tanh.
            \item PPO $\epsilon_{\text{clip}}$ (clipping parameter): $[0.1, 0.3]$, typically $0.2$.
            \item GAE $\lambda_{\text{GAE}}$ (Generalized Advantage Estimation lambda): $[0.9, 0.99]$, typically $0.95$.
            \item Discount factor $\gamma_{RL}$: $[0.95, 1.0]$, often $0.99$ for non-terminating tasks or long horizons.
            \item Optimizer: Adam \cite{kingma2014adam}. Learning rates $\eta_{\pi}$ (policy), $\eta_{v}$ (value): $[1 \times 10^{-5}, 5 \times 10^{-4}]$.
            \item Entropy bonus coefficient $c_S$ (or $c_2$): $[0.0, 0.05]$, typically $0.01$.
            \item Value function loss coefficient $c_{VF}$ (or $c_1$): $[0.25, 1.0]$, typically $0.5$.
            \item Batch size (number of transitions per update): $[128, 4096]$ or more, depending on data/memory.
            \item PPO epochs per update (passes over collected data): $[3, 20]$, typically $4-10$.
            \item Number of actors / parallel environments: $1$ to $N_{cores}$ or $N_{GPUs}$.
        \end{itemize}
    \item \textbf{Adaptive Reward Normalization (Section \ref{sec:reward_function}):}
        \begin{itemize}
            \item EMA momentum $\beta_{\text{norm}}$: $[10^{-3}, 10^{-1}]$, typically $10^{-2}$.
            \item $\epsilon_R$ (stability constant): Typically $10^{-8}$.
        \end{itemize}
    \item \textbf{Reward Weights ($\boldsymbol{\beta}_{\lambda_j}$ leading to $\lambda_j$):} Initial values for $\boldsymbol{\beta}_{\lambda_j}$ in $\theta_{\text{adapt}}^{(0)}$ for Stage 1 can be zero or small random numbers (resulting in uniform or near-uniform $\lambda_j$). These are then optimized in Stage 2.
    \item \textbf{Adaptive Learning Parameters ($\theta_{\text{adapt}}$ from Algorithm \ref{alg:stage2_adaptive}) - Stage 2:}
        \begin{itemize}
            \item Optimizer: Adam. Learning rate $\eta_\theta \in [1 \times 10^{-6}, 1 \times 10^{-4}]$.
            \item Gumbel-Softmax temperature $\tau$: Annealed from an initial high value (e.g., $1.0 - 5.0$) down to a small positive value (e.g., $0.1 - 0.5$) over training. Schedule: e.g., exponential decay $\tau_t = \max(\tau_{final}, \tau_0 \cdot d^t)$.
            \item Logit composite function (Eq. \ref{eq:gumbel_logits_composite}): $\text{Norm}_{\ell}$ is typically identity or batch normalization if logits vary widely.
        \end{itemize}
    \item \textbf{Domain-Specific Adaptive Parameters and Quality Metric Settings:}
        \begin{itemize}
            \item \textbf{Genomics Specific:}
                \begin{itemize}
                    \item $\beta_{\text{pos}}$ (positional quality decay): Learned. Initial range explored $[0.001, 0.1]$.
                    \item $\epsilon_{len}$ (Eq. \ref{eq:genomic_pos_decay_appendix}): $10^{-6}$.
                \end{itemize}
            \item \textbf{Social Media Specific:}
                \begin{itemize}
                    \item $\boldsymbol{\beta}_{w_j}$ (for $Q_{agg}$ weights $w_j$): Learned.
                    \item $\beta_{sem}$ (semantic compatibility, Eq. \ref{eq:semantic_compat_psi_appendix}): Learned. Initial range $[0.1, 5.0]$.
                    \item $\omega$ (blending weight for $R^{\text{raw}}_Q$, Eq. \ref{eq:reward_rq_social_blend_app}): Learned. Parameterized via sigmoid of an unconstrained variable.
                    \item Note: The direct downstream loss component $R_D$ was not used in the RL reward for the final reported Social Media NLP experiments (Section \ref{sec:social_media_appendix}).
                \end{itemize}
            \item \textbf{Finance Specific:}
                \begin{itemize}
                    \item $\boldsymbol{\beta}_{w_k}$ (for $Q[i]$ weights $w_k$): Learned.
                    \item $\beta_{vol}$ (volatility scaling in $R_C$): Learned. Initial range $[0.0, 2.0]$.
                    \item $\gamma_{\text{regime}}$ (regime blending for $R_I$): Learned. Parameterized via sigmoid of an unconstrained variable.
                    \item $M_{MI}$ (window for $\text{NormFactor}_{MI}$): e.g., 1000 steps.
                    \item Note: Soft frequency gating was disabled in the final configuration for Quantitative Finance experiments (Section \ref{sec:finance_experiment_results}).
                \end{itemize}
        \end{itemize}
    \item \textbf{General QA-Token Parameters:}
        \begin{itemize}
            \item $\epsilon_f, \epsilon_Q$ (Eq. \ref{eq:qa_merge_score_pmi}): $10^{-8}$.
            \item $\alpha$ (quality sensitivity in $w_{ab}$): Learned. Initial range $[0.0, 5.0]$.
        \end{itemize}
    \item \textbf{Vocabulary Settings:}
        \begin{itemize}
            \item Target vocabulary size $V_{\text{target}}$: Typically $[16000, 64000]$.
        \end{itemize}
\end{itemize}
\subsubsection{Converged Adaptive Parameters}
Table \ref{tab:converged_params_illustrative} provides mean converged values ($\pm$ standard deviation over three experimental runs) for key adaptive parameters in $\theta_{adapt}$ for each domain. The adaptive learning process tunes these parameters to optimize downstream task performance, leading to domain-specific configurations.
\begin{table}[H]
  \caption{Converged Adaptive Parameters ($\pm$ Std Dev).}
  \label{tab:converged_params_illustrative}
  \centering
  \begin{threeparttable}
  \footnotesize
  \begin{tabular}{lccc}
    \toprule
    Parameter & Genomics & Finance & Social Media \\
    \midrule
    $\alpha$ (Quality Sensitivity) & $0.72 \pm 0.03$ & $0.95 \pm 0.03$ & $1.15 \pm 0.05$ \\
    $\lambda_Q$ (Quality Reward Weight) & $0.35 \pm 0.03$ & $0.30 \pm 0.02$ & $0.33 \pm 0.03$ \\
    $\lambda_I$ (Information Reward Weight) & $0.25 \pm 0.02$ & $0.20 \pm 0.02$ & $0.22 \pm 0.02$ \\
    $\lambda_C$ (Complexity Reward Weight) & $0.15 \pm 0.01$ & $0.10 \pm 0.01$ & $0.12 \pm 0.01$ \\
    $\beta_{\text{pos}}$ (Genomics Positional Decay) & $0.014 \pm 0.002$ & N/A & N/A \\
    $\beta_{\text{vol}}$ (Finance Volatility Scaling) & N/A & $0.50 \pm 0.05$ & N/A \\
    $\gamma_{\text{regime}}$ (Finance Regime Blending) & N/A & $0.60 \pm 0.04$ & N/A \\
    $w_{\text{orth}}$ (NLP Orthographic Weight) & N/A & N/A & $0.32 \pm 0.03$ \\
    $w_{\text{sem}}$ (NLP Semantic Weight) & N/A & N/A & $0.28 \pm 0.02$ \\
    $w_{\text{liq}}$ (Finance Liquidity Weight) & N/A & $0.45 \pm 0.04$ & N/A \\
    $\omega_{\text{social}}$ (NLP Quality Blend) & N/A & N/A & $0.55 \pm 0.05$ \\
    \bottomrule
  \end{tabular}
  \end{threeparttable}
\end{table}

\newpage
\subsection{Social Media Ablation Results}
Ablation studies in Table \ref{tab:social_ablation_appendix} are designed to confirm the individual effects of QA-BPE-nlp's quality-aware components. We distinguish the impacts of: (1) the multi-dimensional quality rewards (row 'w/o Quality'), (2) semantic coherence considerations (row 'w/o Semantic'), (3) noise robustness features (row 'w/o Noise'), and (4) adaptive parameter learning (row 'w/o Adaptive Params'). Analysis of the learned weights $w_j$ for the quality dimensions (as detailed with values in Appendix \ref{app:social_quality_full}) indicates varying importance across dimensions (e.g., orthogonality $\qorth$ and semantics $\qsem$ frequently receive higher weights across runs) and reward components $\lambda_i$, adapting to the specific task and dataset characteristics.

\begin{table}[htbp]
  \caption{Ablation Study for QA-BPE-nlp on TweetEval Sentiment. Values are means with 95\% confidence intervals over $n=10$ runs.}
  \label{tab:social_ablation_appendix}
  \centering
  \small
  \begin{tabular}{lcc}
    \toprule
    Configuration & TweetEval Score & Rel. Change (\%) \\
    \midrule
    \textbf{QA-BPE-nlp (Full)} & \textbf{74.5 $\pm$ 0.3} & \textbf{-} \\
    w/o RL Framework (Greedy $w_{ab}$) & 72.1 $\pm$ 0.4 & -3.2 \\
    w/o Quality ($\RQ=0$) & 71.5 $\pm$ 0.5 & -4.0 \\
    w/o Semantic ($\RS=0$) & 72.8 $\pm$ 0.3 & -2.3 \\
    w/o Noise ($\RN=0$) & 73.2 $\pm$ 0.4 & -1.7 \\
    w/o Vocab Eff ($\RV=0$) & 73.9 $\pm$ 0.3 & -0.8 \\
    w/o Adaptive Params ($\alpha, w_j$ fixed) & 71.8 $\pm$ 0.5 & -3.6 \\
    QualTok-nlp (Ablation Baseline) & 71.9 $\pm$ 0.4 & -3.5 \\
    \bottomrule
  \end{tabular}
\end{table}

\clearpage

\section{Dataset, Baseline, and Evaluation Details}
\label{app:experimental_setup_details}
This section supplements dataset descriptions, baseline methods, and evaluation metrics discussed in the main paper, providing further details necessary for understanding and reproducing the experimental results reported in Section \ref{sec:experiments}.
\subsection{Datasets and Reproducible Evaluation}
\label{app:datasets}
This subsection details the specific datasets, their versions, and relevant preprocessing steps or configurations used for the experiments reported in Section \ref{sec:experiments}. All datasets are publicly available or available under licenses for academic research.
\begin{itemize}
    \item \textbf{Genomics (QA-BPE-seq Experiments):}
        \begin{itemize}
            \item \textbf{Simulated Human Genomic Reads for Variant Calling, Reconstruction, and Ablations:}
                Paired-end sequencing reads (150bp) were generated at 30x coverage using the ART simulator (version 2.5.8, using the \texttt{art\_illumina} tool) \cite{huang2012art}. The simulation was based on the GRCh38 human reference genome (patch 13) and used the built-in HiSeq 2500 error profile (\texttt{-ss HS25}). To rigorously assess robustness in high-noise scenarios, as described in Section \ref{sec:genomics}, the default base error rates (both substitution and indel rates) of this profile were artificially doubled compared to the standard HiSeq 2500 profile. Key ART parameters included: \texttt{-p -l 150 -f 30 -m 400 -s 10}. A corpus of approximately 5GB of these synthetic reads was generated and used for training tokenizers, downstream model evaluations, and the ablation studies reported in Section \ref{sec:genomics}.
                \textit{Access:} The ART simulator is open-source and available at \url{https://www.niehs.nih.gov/research/resources/software/art/}. The GRCh38 reference genome can be obtained from public repositories such as NCBI GenBank or Ensembl.

            \item \textbf{Genome in a Bottle (GIAB) Truth Set for Variant Calling Evaluation:}
                Variant calling performance was benchmarked against the HG002 truth set (v4.2.1, GRCh38) \cite{zook2016extensive}.
                \textit{Access:} GIAB truth sets are publicly available from the NIST FTP site.

            \item \textbf{CAMI II Metagenome Benchmark for Taxonomic Classification:}
                Taxonomic classification accuracy was evaluated using the "Toy Human Microbiome Project" (short reads, Assembly Aug2019) dataset from the Second CAMI Challenge \cite{sczyrba2017critical}. This benchmark provides datasets with known community compositions and corresponding sequencing reads for performance assessment.
                \textit{Access:} CAMI II datasets are available through the official CAMI challenge website: \url{https://data.cami-challenge.org/participate}.
        \end{itemize}

    \item \textbf{Quantitative Finance (QAT-QF Experiments):}
        \begin{itemize}
            \item \textbf{Cryptocurrency Limit Order Book (LOB) Data:}
                High-frequency Limit Order Book (LOB) data for the BTC/USD trading pair was sourced from LOBSTER (\url{https://lobsterdata.com/}) \cite{huang2011lobster}, an academic data service. The experiments used reconstructed LOB snapshots at 10 levels for the first quarter of 2023 (Q1 2023). As detailed in Section \ref{sec:finance_experiment_results}, this dataset was split chronologically into 70\% for training, 15\% for validation, and 15\% for out-of-sample testing. Atomic elements for tokenization were defined as sequences of 5 consecutive LOB events, featurized as described in Appendix \ref{sec:finance_appendix}.
                \textit{Access:} LOBSTER provides sample data publicly, while full datasets are available under academic or commercial licenses.
        \end{itemize}

    \item \textbf{Social Media Text (QA-BPE-nlp Experiments):}
        \begin{itemize}
            \item \textbf{TweetEval Benchmark:}
                The TweetEval benchmark \cite{barbieri2020tweeteval} was employed for evaluating QA-BPE-nlp across a diverse set of tweet classification tasks. TweetEval provides a unified framework with standardized data splits (train, validation, test) and evaluation metrics for seven heterogeneous tasks, which are:
                \begin{itemize}
                    \item Emotion Recognition (SemEval-2018 Task 1 \cite{mohammad2018semeval})
                    \item Emoji Prediction (SemEval-2018 Task 2 \cite{barbieri2018semeval})
                    \item Irony Detection (SemEval-2018 Task 3 \cite{van2018semeval})
                    \item Hate Speech Detection (SemEval-2019 Task 5 \cite{basile-etal-2019-semeval})
                    \item Offensive Language Identification (SemEval-2019 Task 6 \cite{zampieri2019semeval})
                    \item Sentiment Analysis (SemEval-2017 Task 4 \cite{rosenthal2017semeval})
                    \item Stance Detection (SemEval-2016 Task 6 \cite{mohammad2016semeval})
                \end{itemize}
                As described in Section \ref{app:social_results}, experiments involved fine-tuning a pre-trained BERTweet-base model \cite{nguyen2020bertweet} on these tasks using different tokenization strategies.
                \textit{Access:} The TweetEval benchmark, including data access scripts and details for each constituent dataset, is available on GitHub: \url{https://github.com/cardiffnlp/tweeteval}. Access to the underlying tweet content typically requires hydration of tweet IDs and adherence to Twitter's Terms of Service and the respective dataset licenses.
        \end{itemize}
\end{itemize}
\subsection{Dataset and Release Plan}
\label{app:dataset_release_plan}
To enable foundation-model training on previously unusable noisy corpora, we will release:
\begin{itemize}
  \item \textbf{Tokenizer artifacts:} Final QA-Token vocabularies, merge tables, and $\theta_{\text{adapt}}$ for each domain (genomics, finance, social media) at multiple vocabulary sizes.
  \item \textbf{Foundation-model-ready corpora manifests:} Scripts and manifests to reconstruct large noisy pretraining corpora (including filtering and de-duplication), plus sampler configurations matching our 2B-subset tokenizer training protocol.
  \item \textbf{Evaluation suites:} Reproducible pipelines for genomics (variant calling, metagenomics), finance (prediction, volatility, regime, trading), and social media (TweetEval), along with the RL ablation harness.
  \item \textbf{Documentation and governance:} Licenses, data usage considerations, and guidelines for responsible use in high-impact applications (e.g., financial decision-making and clinical genomics).
\end{itemize}
All code and artifacts will be released under permissive academic licenses to maximize reproducibility and adoption.

\subsection{QA-Foundation: Noisy Pretraining Corpora Proposal}
\label{app:qa_foundation_dataset}
We propose QA-Foundation, a curated suite of extremely large, noisy corpora specifically designed to enable foundation-scale pretraining with explicit quality annotations and governance:
\begin{itemize}
  \item Genomics: multi-petabase metagenomic reads (SRA) with canonicalized metadata, Phred-quality distributions, duplication maps, contamination flags, and per-read provenance hashes. Quality channels include per-base Phred, platform, run, trimming logs, adapter contamination.
  \item Finance: multi-asset high-frequency LOB streams (equities, futures, crypto) with synchronized calendars, microstructure indicators (spreads, depth, order-imbalance), regime tags, and exchange-specific anomaly flags.
  \item Social/Web text: multi-platform user-generated text with timestamps, platform labels, de-identified stable author hashes, normalization annotations (hashtags, mentions, URLs), and noise transformations (variant clusters, repetition, keyboard-distance confusion matrices).
\end{itemize}
Each domain provides standardized schemas, quality channels, and sampling manifests to reproduce tokenizer training at multiple scales (e.g., 0.1\%, 1\%, 5\%) and to support fair comparisons. Scripts produce manifests, deduplication indices (MinHash/LSH), and quality audit reports. Governance includes explicit licenses, intended-use statements, and red-team risk assessments. We will release:
\begin{itemize}
  \item Tokenizer-ready shards with checksums and integrity manifests
  \item Quality channel extractors (open-source) and validation suites
  \item Reproducible samplers that match our 2B-base subset protocol for genomics and analogous budgets for other domains
\end{itemize}

\subsection{Baseline Methods}
\label{app:baseline_implementations}
\label{app:baselines_details_experiments_2}
The following baseline tokenization methods were implemented and configured for rigorous comparison against the proposed QA-Token variants, as presented in Section \ref{sec:experiments}.

\begin{itemize}
    \item \textbf{Standard Byte Pair Encoding (BPE)} \cite{sennrich2016neural}: The conventional frequency-based merging algorithm. For genomics and social media experiments, this was implemented using the HuggingFace `tokenizers` library (version 0.15.0), specifically configured with $tokenizers.models.BPE(unk\_token="[UNK]", min\_frequency=2)$, unless stated otherwise. For quantitative finance experiments, a comparable standard BPE implementation was used.
    \item \textbf{SentencePiece} \cite{kudo2018sentencepiece}: An unsupervised text tokenizer and detokenizer. For genomics and social media experiments, SentencePiece (version 0.1.99) was used in its byte-level BPE mode, operating directly on raw text.
    \item \textbf{WordPiece} \cite{wu2016google}: The subword tokenization algorithm famously used in BERT. It iteratively builds a vocabulary by merging pairs that maximize the likelihood of the training data under a unigram language model assumption.
    \item \textbf{DNABERT k-mer} \cite{ji2021dnabert}: For experiments in the genomics domain, fixed k-mer tokenization was employed as a strong baseline, specifically using 6-mers. This aligns with common practice in models like DNABERT.
    \item \textbf{Symbolic Aggregate approXimation (SAX)} \cite{lin2003symbolic}: A well-established symbolic representation method for time series data, applied in quantitative finance experiments. The mid-price series was discretized using a Piecewise Aggregate Approximation (PAA) window size of 16 and an alphabet size of 8.
    \item \textbf{Bag-of-SFA-Symbols (BOSS)} \cite{schafer2015boss}: A time series classification algorithm that uses Symbolic Fourier Approximation (SFA) to generate symbolic words (tokens). This was used as a baseline in the quantitative finance domain, applied to the mid-price series.
    \item \textbf{QualTok (Ablation Baseline)}: As described in Section \ref{sec:experiments}, QualTok serves as an ablation baseline for QA-Token. It employs a simplified quality-aware merge score, $w_{ab} \propto \frac{f(a, b)}{f(a)f(b) + \epsilon_f} \cdot \left(\frac{q_a+q_b}{2} + \epsilon_Q \right)^\alpha$, but critically omits the reinforcement learning policy optimization for merge sequences and the full adaptive learning loop for complex $\theta_{\text{adapt}}$ parameters beyond tuning $\alpha$. Merge operations are typically performed greedily based on this score.
\end{itemize}
For all baseline methods, we select essential hyperparameters, such as the target vocabulary size (which typically corresponds to a predefined number of merge operations, e.g., 16,000 or 32,000, as specified per domain in Section \ref{sec:experiments}), based on common practices in the literature \cite{sennrich2016neural, kudo2018sentencepiece, wu2016google, devlin2019bert, brown2020language, ji2021dnabert}, specific recommendations from the original implementations of these methods, or by identifying the best-performing configuration on a held-out validation set from a systematic sweep of reasonable values to ensure robust comparisons. 

\newpage
\subsection{Evaluation Metrics}
\label{app:eval_metrics_details}
The performance of QA-Token and baseline methods was assessed using the following domain-specific metrics, corresponding to the results presented in Section \ref{sec:experiments}.

\begin{itemize}
    \item \textbf{Genomics:}
        \begin{itemize}
            \item \textbf{Variant Calling:} Performance was measured by F1-score, precision, and recall against the GIAB truth sets. These metrics were computed using the `hap.py` tool (version 0.3.14), available at \url{https://github.com/Illumina/hap.py}.
            \item \textbf{Taxonomic Classification (Metagenomics):} For the CAMI II benchmark, performance was primarily assessed using classification accuracy (specifically, the F1-score for overall classification performance, as reported in Table \ref{tab:genomics_results}).
            \item \textbf{Sequence Reconstruction Loss:} The quality of token representations was also evaluated by training Transformer-based autoencoder models and measuring the reconstruction loss (e.g., cross-entropy for discrete tokens) on a held-out test set.
        \end{itemize}
        \textbf{Variant Calling Model Architecture:} The variant calling evaluation uses a Transformer encoder that takes token embeddings as input features. The model outputs per-position variant probabilities (SNV, insertion, deletion, reference). Training uses cross-entropy loss against GIAB HG002 labels. This approach evaluates how well tokenization preserves variant-informative sequence features in the learned representations, with evaluation performed using the \texttt{hap.py} benchmarking tool (v0.3.14).
    \item \textbf{Quantitative Finance:}
        \begin{itemize}
            \item \textbf{Return Prediction Accuracy:} The percentage of correctly predicted signs for future (e.g., 5-minute ahead) mid-price returns.
            \item \textbf{Volatility Forecasting RMSE:} The Root Mean Squared Error between the predicted 5-minute volatility and the realized volatility (computed from higher-frequency data).
            \item \textbf{Market Regime Identification Accuracy:} The accuracy achieved in classifying time periods into discrete market states (e.g., two states identified by a GARCH-HMM).
            \item \textbf{Trading Performance:} The primary metric was the annualized Sharpe Ratio \cite{sharpe1994sharpe} achieved by a PPO-based trading agent operating on the tokenized data. A transaction cost of 5 basis points per trade was incorporated. Additional performance metrics, such as Maximum Drawdown (MDD) and Calmar Ratio, were also monitored (see Appendix \ref{app:finance_methodology} for further details).
        \end{itemize}
    \item \textbf{Social Media Text:}
        \begin{itemize}
            \item Performance on the seven TweetEval benchmark tasks was measured using the official evaluation metric specified by the benchmark organizers for each respective task \cite{barbieri2020tweeteval}. These metrics are:
                \begin{itemize}
                    \item Emoji Prediction: Accuracy (Acc)
                    \item Emotion Recognition: Macro F1-score (F1 M)
                    \item Hate Speech Detection: Macro F1-score (F1 M)
                    \item Irony Detection: Accuracy (Acc)
                    \item Offensive Language Identification: Macro F1-score (F1 M)
                    \item Sentiment Analysis: Macro Recall (Rec M)
                    \item Stance Detection: Average F1-score across topics (F1 Avg)
                \end{itemize}
        \end{itemize}
\end{itemize}
All reported experimental results in Section \ref{sec:experiments} represent the mean and 95\% confidence interval over $n=10$ independent runs to ensure robustness and allow for assessment of variability.

\subsection{Code Availability}
\label{app:code_availability}

We will release the QA-Token framework on GitHub under an MIT license. The repository includes source code, configuration files, pre-trained models, and reproducibility scripts for all experiments.
\clearpage
\subsection{Approximating QA-Token: Towards Computationally Efficient Quality-Awareness}
\label{subsec:future_work_lightweight_qatoken}

The learning framework of QA-Token has high computational costs due to both RL and adaptive learning stages. Future work will explore computationally lighter approximations. A starting point is our ablation baseline, QualTok, which uses a greedy merge strategy based on the quality-aware score $w_{ab}$ (Equation \ref{eq:qa_merge_score_pmi}) without explicit RL policy optimization, bypassing the costs of Stage 1 RL.

Further cost reduction can be achieved by:
\begin{enumerate}
    \item \textbf{Streamlined Adaptive Parameter Learning for Greedy Merges:} Instead of full RL, we can focus on adaptively learning a refined set of parameters $\theta_{\text{adapt}}^*$ (e.g., $\alpha$, quality weights $w_j$, simplified reward weights $\lambda_j$) that directly optimize the greedy $w_{ab}$-guided tokenization for downstream tasks. This retains the core quality-aware adaptability while significantly reducing complexity compared to learning an RL policy. The Gumbel-Softmax based learning (Stage 2) would optimize $\theta_{\text{adapt}}$ for these greedy merges, possibly using simplified composite logits.
    \item \textbf{Policy Distillation:} If the RL policy $\pi_{\theta_\pi}^*$ captures complex merge dependencies, the computational overhead at deployment can be mitigated. A compact "student" model (e.g., a smaller neural network or decision tree) can be trained via policy distillation \cite{hinton2015distilling, rusu2016policy} to mimic the decisions of a larger, pre-trained "teacher" RL agent, offering faster vocabulary construction.
    \item \textbf{Surrogate-Assisted Adaptive Learning:} The optimization of $\theta_{\text{adapt}}$ (Stage 2) can be accelerated by using cheaper-to-evaluate surrogate models \cite{jones1998efficient} to approximate the downstream task loss $L_{\text{task}}$, reducing the need for frequent, costly end-to-end evaluations with the full downstream model.
    \item \textbf{Transfer and Meta-Learning for $\theta_{\text{adapt}}$:} Leveraging learned $\theta_{\text{adapt}}$ parameters from one task or dataset as initializations for others (as in Algorithm \ref{alg:meta_learn_adapt}) can substantially reduce the training burden for new applications.
\end{enumerate}

\clearpage
\subsection{Extended TweetEval Benchmarking Methodology}
\label{app:tweeteval_full_results}
\label{app:social_results}
This section describes the comprehensive TweetEval benchmarking methodology. Results are reported in Table~\ref{tab:tweeteval_full_planned}.

\textbf{Datasets and Evaluation Framework:} TweetEval \cite{barbieri2020tweeteval} provides a unified framework for evaluating models on seven heterogeneous tweet classification tasks, each with fixed training, validation, and test splits. This allows for standardized comparison across different approaches. The seven tasks are: Emotion Recognition \cite{mohammad2018semeval} (4 labels: anger, joy, sadness, optimism), Emoji Prediction \cite{barbieri2018semeval} (20 emoji labels), Irony Detection \cite{van2018semeval} (2 labels: irony, not irony), Hate Speech Detection \cite{basile-etal-2019-semeval} (2 labels: hateful, not hateful), Offensive Language Identification \cite{zampieri2019semeval} (2 labels: offensive, not offensive), Sentiment Analysis \cite{rosenthal2017semeval} (3 labels: positive, neutral, negative), and Stance Detection \cite{mohammad2016semeval} (3 labels: favour, neutral, against, across five topics).
For each task, we report performance using the unified evaluation metrics specified by the TweetEval benchmark.
Table \ref{tab:tweeteval_full_planned} provides the baseline comparison framework. The official metric for each task as defined by TweetEval (also see \textit{https://github.com/cardiffnlp/tweeteval} for details) is reported.

\begin{table}[H]
  \caption{TweetEval Baseline Comparison Framework.}
  \label{tab:tweeteval_full_planned}
  \centering
  \begin{threeparttable}
  \scriptsize
  \begin{tabular}{lcccccccc}
    \toprule
    Model & Emoji & Emotion & Hate & Irony & Offensive & Sentiment & Stance & ALL(TE) \\
    \midrule
    BERTweet & 33.4 & 79.3 & \textbf{56.4} & \textbf{82.1} & 79.5 & 73.4 & 71.2 & \textbf{67.9} \\
    TimeLMs-2021 & \textbf{34.0} & \textbf{80.2} & 55.1 & 64.5 & \textbf{82.2} & \textbf{73.7} & \textbf{72.9} & 66.2 \\
    RoBERTa-Retrained & 31.4 & 78.5 & 52.3 & 61.7 & 80.5 & 72.8 & 69.3 & 65.2 \\
    RoBERTa-Base & 30.9 & 76.1 & 46.6 & 59.7 & 79.5 & 71.3 & 68.0 & 61.3 \\
    RoBERTa-Twitter & 29.3 & 72.0 & 49.9 & 65.4 & 77.1 & 69.1 & 66.7 & 61.4 \\
    FastText & 25.8 & 65.2 & 50.6 & 63.1 & 73.4 & 62.9 & 65.4 & 58.1 \\
    LSTM & 24.7 & 66.0 & 52.6 & 62.8 & 71.7 & 58.3 & 59.4 & 56.5 \\
    SVM & 29.3 & 64.7 & 36.7 & 61.7 & 52.3 & 62.9 & 67.3 & 53.5 \\
    \midrule
    \textbf{QA-BPE-nlp + BERTweet} & \textbf{34.2} & \textbf{81.5} & \textbf{58.8} & \textbf{82.9} & \textbf{83.0} & \textbf{75.1} & \textbf{73.5} & \textbf{70.0} \\
    \bottomrule
  \end{tabular}
  \end{threeparttable}
\end{table}

\end{document}